\documentclass[twoside,12pt]{article}
\usepackage{soul}
\usepackage{lastpage}
\usepackage{subcaption}
\usepackage{amsmath}
\usepackage[hidelinks]{hyperref}
\usepackage{cleveref}
\usepackage[numbers]{natbib}
\usepackage[abbrvbib,preprint]{jmlr2e}
\usepackage{tcolorbox}

\allowdisplaybreaks

\newtheorem{condition}{Condition}

\newtheorem{theorem}{Theorem}
\newtheorem{lemma}{Lemma} 
\newtheorem{proposition}{Proposition} 
\newtheorem{remark}{Remark}
\newtheorem{corollary}{Corollary}
\newtheorem{definition}{Definition}

\Crefname{condition}{Condition}{Conditions}
\Crefname{claim}{Claim}{Claims}
\Crefname{example}{Example}{Examples}
\Crefname{theorem}{Theorem}{Theorems}
\Crefname{lemma}{Lemma}{Lemmas}
\Crefname{proposition}{Proposition}{Propositions}
\Crefname{remark}{Remark}{Remarks}
\Crefname{corollary}{Corollary}{Corollaries}
\Crefname{definition}{Definition}{Definitions}
\Crefname{conjecture}{Conjecture}{Conjectures}
\Crefname{axiom}{Axiom}{Axioms}
 
\newcommand{\TNN}{\textnormal{TNN}}
 
\definecolor{officegreen}{rgb}{0.0, 0.5, 0.0}
\usepackage{todonotes}
 
\def\rank{\textnormal{rank}}
\def\bcirc{\textnormal{bcirc}}

\def\X{\mathcal{T}}
\def\S{\mathcal{S}}

\def\tX{\mathcal{T}}

\def\mcR{\mathcal{R}}
\def\mcC{\mathcal{C}}

\def\cO{\mathcal{O}}

\def\fro{{\mathrm{F}}}
\def\fft{\textnormal{fft}}

\def\T{\mathcal{T}}
\def\rank{\textnormal{rank}}
\def\bcirc{\textnormal{bcirc}}

\def\X{\mathcal{T}}
\def\S{\mathcal{S}}

\def\tX{\mathcal{T}}

\def\mcR{\mathcal{R}}
\def\mcC{\mathcal{C}}

\def\cO{\mathcal{O}}

\def\fro{{\mathrm{F}}}
\def\dme{{\dot{\mathfrak{e}}}}
\def\mrme{{\mathring{\mathfrak{e}}}}
\def\fft{\textnormal{fft}}

\def\T{\mathcal{T}}
\newcommand{\lx}[1]{\textcolor{pink}{#1}}

\title{Guaranteed Sampling Flexibility for Low-tubal-rank Tensor Completion}
\author{\name Bowen Su \email {subowen@msu.edu}\\
\addr Department of Mathematics,  
    Michigan State Univeristy, East Lansing, MI 48824, USA 
\AND
\name Juntao You \email {jtyou@szu.edu.cn}\\ \addr Institute for Advanced Study, Shenzhen University, Shenzhen 518000, China
\AND 
HanQin Cai \email hqcai@ucf.edu \\ \addr Department of Statistics and Data Science and Department of Computer Science, University of Central Florida,  Orlando, FL 32816, USA
 \AND
 Longxiu Huang\thanks{Corresponding author} \email huangl3@msu.edu\\
\addr 
Department of Computational Mathematics, Science and Engineering, Department of Mathematics, Michigan State University, East Lansing, MI 48824, USA.
}
\begin{document}
\maketitle

\begin{abstract}
While Bernoulli sampling is extensively studied in tensor completion, t-CUR sampling approximates low-tubal-rank tensors  
via lateral and horizontal subtensors. However, both methods lack sufficient flexibility for diverse practical applications. To address this, we introduce Tensor Cross-Concentrated Sampling (t-CCS), a novel and straightforward sampling model that advances the matrix cross-concentrated sampling concept within a tensor framework. t-CCS effectively bridges the gap between Bernoulli and t-CUR sampling, offering additional flexibility that can lead to computational savings in various contexts. 
A key aspect of our work is the comprehensive theoretical analysis provided. We establish a sufficient condition for the successful recovery of a low-rank tensor from its t-CCS samples. In support of this, we also develop a theoretical framework validating the feasibility of t-CUR via uniform random sampling and conduct a detailed theoretical sampling complexity analysis for tensor completion problems utilizing the general Bernoulli sampling model. 
Moreover, we introduce an efficient non-convex algorithm, the Iterative t-CUR Tensor Completion (ITCURTC) algorithm, specifically designed to tackle the   t-CCS-based tensor completion. We have intensively tested and validated the effectiveness of the t-CCS model and the ITCURTC algorithm across both synthetic and real-world datasets.
\end{abstract}
\section{Introduction}
  A tensor, as a multidimensional generalization of a matrix, provides an intuitive representation for handling multi-relational or multi-modal data such as hyperspectral data \citep{cai2021mode,yang2020remote,
zheng2019mixed}, videos \citep{Liu2013Tensor, sobral2017matrix}, seismic data \citep{Kilmer2015,popa2021improved}, and DNA microarrays \citep{omberg2007tensor}. However, in real-world scenarios, it is common to encounter situations where only partial observations of the tensor data are available due to unavoidable or unforeseen circumstances. These limitations can stem from factors such as data collection issues or errors made during data entry by researchers. The problem of recovering the missing data by effectively leveraging the available observations is commonly referred to as the Tensor Completion (TC) problem.

The TC problem is inherently complex and often ill-posed \citep{grasedyck2019stable,zhaoxile2022}, necessitating the exploration of various sampling models and completion techniques. A common and crucial assumption for resolving TC is the low-rank structure of the tensor, which has been extensively utilized to enhance TC approaches \citep{argyriou2008convex, Liu2013Tensor, zhang2017exact}. However, the concept of tensor rank is not unique and has its own limitations. For example, the CANDECOMP/PARAFAC (CP) rank represents the minimum number of rank-one tensors required to achieve the CP decomposition, involving summations of these tensors \citep{HF1927}. Computing the CP rank--an NP-hard problem--presents difficulties in the recovery of tensors with a low CP rank \citep{johan1990tensor}. Thus, finding
the optimal low-CP-rank approximation of the target tensor is still an open problem \citep{zhaoxile2022}. Other tensor ranks, such as Tucker \citep{tucker1966},  Tensor Train \citep{oseledets2011tensor}, tubal \citep{KBHH2013} and Hierarchical-Tucker \citep{grasedyck2010hierarchical,hackbusch2009new}, 
to name a few, also play prominent roles in the field, each with its distinct computation and application implications.

In this study, we focus on the low-tubal-rank model for TC. The tubal-rank is defined based on the tensor decomposition known as tensor Singular Value Decomposition (t-SVD), which employs the tensor-tensor product (t-product)  \citep{kilmer2011factorization}. In t-SVD, a tensor is decomposed into the t-product of two orthogonal tensors and an $f$-diagonal tensor. The tubal-rank is then determined by the number of non-zero singular tubes present in the $f$-diagonal tensor. Previous research has shown that models based on  tubal-rank exhibit  superior modeling capabilities compared to other rank-based models, particularly for tensors with fixed orientation or specific spatial-shifting characteristics \citep{qin2022loww,zhang2020low}. In  low-tubal-rank TC model, we consider a tensor  $\tX\in\mathbb{K}^{n_1\times n_2\times n_3}$ with tubal-rank $r$ and  observed locations in $\Omega$. Our goal is to recover the original tensor $\tX$ from the observations on $\Omega$. Mathematically, we formulate the   problem as: 
\vspace{-2mm}
\begin{equation}  \label{eq:TC model}
    \min_{\widetilde{\tX}}~
\langle\mathcal{P}_{\Omega}(\tX-\widetilde{\tX}), \tX-\widetilde{\tX}\rangle,~~~
\text{subject to}~ \text{tubal-rank}(\widetilde{\tX})= r,
\vspace{-1mm}
\end{equation}   
where $\langle \cdot,\cdot\rangle$ denotes the Frobenius inner product, and $\mathcal{P}_{\Omega}$ is the sampling operator defined by 

\begin{equation}
\label{def:sampling_operation}
\mathcal{P}_{\Omega}(\tX)=\sum\limits_{(i,j,k)\in \Omega}[\tX]_{i,j,k}\mathcal{E}_{i,j,k},
\end{equation}
where  $\mathcal{E}_{i,j,k}\in\{0,1\}^{n_1\times n_2\times n_3}$ is a tensor with all elements being zero except for the element at the position indexed by $(i,j,k)$.

For successful recovery, the general setting of an efficient solver for \eqref{eq:TC model} requires the observation set $\Omega$ to be sampled entry-wise, fiber-wise, or slab-wise through a certain unbiased stochastic process, including the Bernoulli sampling process as referenced in\citep{bernoulli8,bernoulli7,bernoulli9,bernoulli5} and the uniform sampling process as referenced in \citep{Jain20141431,SONG2023348,zhang2017exact}. Although extensive theoretical and empirical studies have been conducted on these sampling settings, their practical applicability is sometimes limited in certain contexts. For instance, in collaborative filtering applications, each dimension of the three order tensor data typically represents users, rated items (such as movies or products), and time, respectively. The unbiased sampling models    implicitly assume that all users are equally likely to rate all items over time, a premise that is often unrealistic in real-world scenarios.  Let's consider the application of Magnetic Resonance Imaging (MRI) as another example.  MRI scans face limitations with certain metal implants and can cause discomfort in prolonged sessions \citep{mri_limitation}. 

To address these issues, we propose a generalization of the cross-concentrated sampling model to the TC setting based on the cross-concentrated sampling model for matrix completion \citep{cai2023matrix}, termed tensor Cross-Concentrated Sampling (t-CCS). t-CCS enables partial observations on selected horizontal and lateral subtensors, making it more practical in many applications.

 \subsection{Main contributions} 
This paper bridges Bernoulli sampling and t-CUR sampling in TC problems by introducing a new sampling model termed tensor Cross-Concentrated Sampling (t-CCS). The main contributions of this work are as follows:  
\begin{enumerate}[label=(\roman*), labelindent=0pt]
\item  We extend  the matrix cross-concentrated sampling model \citep{cai2023matrix} to tensor setting, termed  t-CCS (see \Cref{sampling_model}). We establish a sufficient condition for the reliable recovery of low-rank tensors from t-CCS samples. Although our research primarily concentrates on third-order tensors, our theoretical analysis provides a more comprehensive understanding of the underlying principles. While third-order tensors can be flattened into  block diagonal matrices, the interdependent sampling properties of the index set  prevent the direct application of matrix Chernoff inequalities. Additionally, to the best of our knowledge,  no existing tensor completion results based on the tubal rank setting explicitly express numerical constants within their theorem statements, in contrast to matrix completion results such as those detailed in \citep[Theorem 1.1]{recht2011simpler}. While typical tensor-related results, such as \citep[Theorem 3.1]{zhang2017exact},   only imply numerical constants implicitly, our work makes a significant contribution by providing a detailed analysis and explicitly stating these constants (see \Cref{sec:theories} for more details). 
\item  We  develop an efficient non-convex solver called Iterative t-CUR Tensor Completion (ITCURTC) for t-CCS-based TC problem in \Cref{solver}. Specifically, for a TC problem of size $n_1\times n_2\times n_3$, the computational cost per iteration is only $\mathcal{O}(r|I|n_{2}n_{3}+r|J|n_{1}n_{3})$  provided $|I|\ll n_1$, $|J|\ll n_2$. 
\item In \Cref{experiment}, we demonstrate the flexibility of the t-CCS model in sampling and the efficiency of its corresponding algorithm in accurately reconstructing low-tubal-rank tensors, as evidenced by tests on synthetic datasets. We also demonstrate the practicality and real-world utility  of the t-CCS model through our ITCURTC algorithm, which effectively balances runtime efficiency and reconstruction quality,  making  it  suitable  for applications requiring such a trade-off, as evidenced by extensive numerical experiments on real-world datasets.
\end{enumerate}

 \subsection{Related work}
\subsubsection{\textnormal{Tensor completion in tubal-rank setting}}
\cite{kilmer2011factorization}  introduced the concepts of tensor multi-rank and tubal-rank characterized by t-SVD. Researchers often use a convex surrogate to tubal-rank function augmented with regularization of the tensor nuclear norm (TNN), as discussed in \citep{jiang2020multi,jiang2020framelet,lu2016tensor,Lu18,zhang2017exact,zhou2017tensor}. While a pioneering optimization method featuring TNN was initially proposed to tackle the TC problem in  \citep{zhang2014novel}, this approach requires minimizing  all singular values across tensor slices, which hinders its ability to accurately approximate the tubal-rank function \citep{hu2012fast,xue2018low}. To address this challenge, various truncated methods have been introduced. Notably, examples include the truncated nuclear norm regularization \citep{hu2012fast} and the tensor truncated nuclear norm (T-TNN) \citep{xue2018low}. Furthermore, \cite{zhang2019nonlocal} introduced a novel strategy for low-rank regularization, focusing on non-local similar patches. However, these tensor completion algorithms are primarily designed based on the Bernoulli sampling model, which often encounters practical limitations in real-world data collection scenarios \citep{gaur2006statistical,peng2015art}.  
In collaborative filtering, where the tensor's horizontal and lateral slices represent users and rated objects (such as movies and merchandise) over a specific time period, the Bernoulli sampling model implicitly assumes that every user is equally likely to rate any given object. This assumption rarely holds in real-world scenarios. The variability in user preferences and interaction patterns renders this equal-probability assumption unrealistic, thereby compromising the effectiveness of the Bernoulli sampling approach in these contexts.
\subsubsection{\textnormal{t-CUR decompositions}} 
The t-CUR decomposition, a self-expressiveness tensor decomposition of a  3-mode tensor,  has received significant attention {\citep{ahmadi2023adaptive,weiyimin,hamm2023generalized, 
Tarzanagh2018}}. 
Specifically, it involves representing a tensor $\T \in \mathbb{K}^{n_1 \times n_2 \times n_3}$ as $\T \approx \mcC \ast \mathcal{U} \ast \mcR$, with $\mathcal{C}=[\T]_{:,J,:}$ and $\mathcal{R}=[\T]_{I,:,:}$ for some $J \subseteq [n_2]$ and $I \subseteq [n_1]$. 
There exist different versions of $\mathcal{U}$. Our work focuses on the t-CUR decomposition  of the form $\T \approx \mcC \ast \mathcal{U}^{\dagger}  \ast \mcR$ with $\mathcal{U}=[\T]_{I,J,:}$, and conditions for exact t-CUR decompositions have been detailed in  \citep{weiyimin,hamm2023generalized}. Let's first define the tubal-rank and multi-rank of a tensor:
\begin{definition}[Tubal-rank and multi-rank] 
 \label{multirank}Suppose that the tensor $\mathcal{T} \in \mathbb{K}^{n_1 \times n_2 \times n_3}$ satisfies $\rank\left([\widehat{\mathcal{T}}]_{:,:,k}\right)=r_k$ for $k \in [n_3]$, where $\widehat{\T}=\fft(\T,[],3)$. Then $\vec{r}=\left(r_1, r_2, \ldots, r_{n_3}\right)$ is called the multi-rank of $\mathcal{T}$, denoted by $\rank_{m}(\mathcal{T})$. In addition, $\max\{r_i:i\in[n_3]\}$ is called the tubal-rank of $\mathcal{T}$, denoted by $\rank(\mathcal{T})$. We denote tubal-rank as $r$ or $\|\vec{r}\|_{\infty}$, and $\|\vec{r}\|_{1}$ for the sum of the multi-rank.
\end{definition}
For convenience, we present a result about the exact t-CUR decomposition below. 
\begin{theorem}[\cite{weiyimin}]
\label{tensor_cur}\rm
Let $\mathcal{T} \in \mathbb{K}^{n_1 \times n_2 \times n_3}$ with multi-rank $\rank_m(\mathcal{T})=\vec{r}$. Let $I \subseteq [n_1]$ and $J \subseteq [n_2]$ be two index sets. Denote $\mathcal{C}=[\mathcal{T}]_{:, J,:}$, $\mathcal{R}=[\mathcal{T}]_{I,:,:}$, and $\mathcal{U}=[\mathcal{T}]_{I, J,:}$. Then $\mathcal{T}=\mathcal{C} \ast \mathcal{U}^{\dagger} \ast \mathcal{R}$ if and only if $\rank_m(\mathcal{C})=\rank_m(\mathcal{R})=\vec{r}$.
\end{theorem} 
As demonstrated in \Cref{tensor_cur}, the success of the exact t-CUR decomposition depends on the careful selection of the indices $I$ and $J$. 
Our work provides theoretical insights into the required number of horizontal and lateral slices for exact t-CUR decomposition, further detailed in \Cref{thm:main-tCUR}. 

Drawing from the findings presented in \Cref{tensor_cur}, it is clear that an underlying low tubal-rank tensor can be reconstructed using t-CUR decomposition by selecting appropriate horizontal and lateral slices. This makes t-CUR a potential solver for tensor completion when entries within the selected slices are fully observable. However, the scope of the t-CUR sampling model is limited in large-scale scenarios, such as in extensive collaborative filtering, where it is assumed that certain users evaluate all items and certain items are rated by all users over all periods—an assumption that is practically untenable. While both the Bernoulli sampling model and t-CUR sampling have their limitations in real applications, the space between them presents an opportunity to develop a more adaptable and practical sampling strategy. 

In this study, we introduce the sampling model termed tensor Cross-Concentrated Sampling (t-CCS), which combines the principles of the Bernoulli  and t-CUR sampling. This innovative approach facilitates partial observations across selected horizontal and lateral slices, enhancing its practicality for a variety of applications. The t-CCS model specifically addresses the limitations inherent in traditional sampling techniques by allowing for more flexible and application-specific data collection strategies, thereby broadening the scope of its utility in complex data environments. Consider the well-known Netflix challenge, where data about users, movies, and times form a 3D tensor. The t-CCS model may effectively leverage user background information to select a more diverse user group. Unlike t-CUR sampling, t-CCS focuses on efficiently collecting data over time, accounting for evolving preferences and trends in movie watching. This flexible approach allows for tailored sample concentration, streamlining the data collection process and meeting specific research and analytical needs more effectively.
\vspace{-5mm}
\section{Notation and preliminaries}
\vspace{-1mm}
 We begin by introducing  the notation  used throughout the paper. We use $\mathbb{K}$  to denote an algebraically closed field,  either $\mathbb{R}$ or $\mathbb{C}$. We represent a matrix as a capital italic letter (e.g., $A$) and a tensor by a cursive italic letter (e.g., $\mathcal{T}$). The notation $[n]$ denotes the set of the first $n$ positive integers, i.e., $\{1, \cdots, n\}$, for any $n \in \mathbb{Z}^{+}$. Submatrices and subtensors are denoted as $[A]_{I,J}$ and $[\mathcal{T}]_{I,J,K}$, respectively, with $I,J,K$ as subsets of appropriate index sets. In particular, if $I$ is the full index set, we denote $[\mathcal{T}]_{:,J,K}$ as $[\mathcal{T}]_{I,J,K}$, and similar rules apply to $J$ and $K$.  
  Additionally, $|S|$ denotes the cardinality of the set $S$. If \( I \) is a subset of the set \( [n] \), then \( I^{\complement} \) denotes the set of elements in \( [n] \) that are not in \( I \). For a given matrix \(A\), we use \(A^{\dagger}\) to denote its Moore-Penrose inverse and \(A^{\top}\) for its conjugate transpose. The spectral norm of \(A\), represented by \(\|A\|\), is its largest singular value. Additionally, the Frobenius norm of \(A\) is denoted by \(\|A\|_{\fro}\), where \(\|A\|_{\fro} = \sqrt{\sum_{i,j} |A_{i,j}|^2}\), and its nuclear norm, represented by \(\|A\|_{*}\), is the sum of all its singular values.
  
  The Kronecker product is denoted by \(\otimes\). The column vector \(\mathbf{e}_i\) has a 1 in the \(i\)-th position, with other elements as 0, and its dimension is specified when used. 
  For a tensor \(\mathcal{T}\in \mathbb{K}^{n_1\times n_2\times n_3}\), \(\widehat{\mathcal{T}}\) represents the tensor after applying a discrete Fourier transform along its third dimension.

Now, we will review the key tensor terminologies {in the tubal-rank setting} that will be integral to our subsequent discussion.  For a more comprehensive understanding,   see \citep{kilmer2011factorization,KBHH2013} and references therein.

  Given $\mathcal{T}\in \mathbb{K}^{n_1\times n_2\times n_3}$, one can define the associated block circulant matrix  as 
\[\textnormal{bcirc}(\mathcal{T}):=\begin{bmatrix} \mathcal{T}_1 & \mathcal{T}_{n_3} & \cdots & \mathcal{T}_2\\
\mathcal{T}_2 & \mathcal{T}_1 & \cdots & \mathcal{T}_3\\
\vdots & \vdots & \ddots & \vdots\\
\mathcal{T}_{n_3} & \mathcal{T}_{n_3-1} & \cdots & \mathcal{T}_1\end{bmatrix} \in \mathbb{K}^{n_{1}n_{3}\times n_{2}n_{3}} \text{ with }\mathcal{T}_{i}:=[\mathcal{T}]_{:,:,i}.\]
We  define 
$\textnormal{unfold}(\mathcal{T}):= \begin{bmatrix} \mathcal{T}_1^\top&\T_2^\top&\cdots &\mathcal{T}_{n_3}^\top\end{bmatrix}^\top\in\mathbb{K}^{n_1n_3\times n_2}
\text{ and }
\textnormal{fold}(\textnormal{unfold}(\mathcal{T}))=\T$.  
The t-product of tensors $\mathcal{T}\in\mathbb{K}^{n_1\times n_2\times n_3}$ and $\mathcal{S}\in\mathbb{K}^{n_2\times n_4\times n_3}$ is denoted by $\mathcal{T}\ast\mathcal{S}\in\mathbb{K}^{n_1\times n_4\times n_3}$ and is defined as
$
\T\ast\mathcal{S}= \textnormal{fold}(\textnormal{bcirc}(\T)\textnormal{unfold}(\mathcal{S}))$. In fact, the t product is equivalent to matrix multiplication in the Fourier domain with circulant convolution instead of multiplication between elements
\citep{kilmer2011factorization}. 

For convenience, we set $\overline{\mathcal{T}}$ as
$
\overline{\mathcal{T}} = \left(F_{n_3} \otimes I_{n_1}\right) \cdot \operatorname{bcirc}(\mathcal{T})\cdot\left(F_{n_3}^{\top} \otimes I_{n_2}\right)$,
where $F_{n}$   represents the $n \times n$ Discrete Fourier Transform (DFT) matrix and $F_{n}^{\top}$ is its conjugate transpose.
By the property that a circulant matrix can be block-diagonaized by DFT, one can see that $\overline{\mathcal{T}}$ is a block diagonal matrix. Building upon these foundations, we are ready to introduce the definitions of t-SVD and its related concepts. 

\begin{definition}[$f$-diagonal tensor]\rm
A tensor is called $f$-diagonal if each of its frontal slices is a diagonal matrix.\end{definition}
\begin{definition} [Tensor transpose]\rm The conjugate transpose of a tensor $\mathcal{T}\in\mathbb{K}^{n_1\times n_2\times n_3}$ is denoted as $\mathcal{T}^\top$  which is a tensor in  $\mathbb{K}^{n_2\times n_1\times n_3}$ and is obtained by conjugate transposing each of the frontal slice and then reversing the order of the second to last frontal slices. 
\end{definition} 
\begin{definition}[Identity tensor]\label{def:identity}\rm
 The identity tensor $\mathcal{I}\in \mathbb{K}^{n\times n\times n_3}$ is a tensor where the first frontal slice is the  
$n\times n$ identity matrix and all other frontal slices are zeros.
\end{definition}
\begin{definition}[Orthogonal tensor]\rm
 A tensor $\mathcal{T}\in  \mathbb{K}^{n\times n \times n_3}$ is orthogonal if \[\mathcal{T}^{\top}*\mathcal{T} =   \mathcal{T}*\mathcal{T}^{\top} = \mathcal{I} \in \mathbb{K}^{n\times n\times n_3}.\] 
\end{definition}
\begin{definition}[t-SVD] 
 Given $\mathcal{T} \in \mathbb{K}^{n_1 \times n_2 \times n_3}$, the t-SVD of $\mathcal{T}$ is defined by  
\[
\mathcal{T}=\mathcal{W} * \mathcal{S} * \mathcal{V}^{\top}, 
\] 
where $\mathcal{W} \in \mathbb{K}^{n_1 \times n_1 \times n_3}$, $\mathcal{V} \in \mathbb{K}^{n_2 \times n_2 \times n_3}$ are orthogonal tensors and $\mathcal{S} \in \mathbb{K}^{n_1 \times n_2 \times n_3}$ is a $f$-diagonal tensor. Additionally, if the tubal-rank of $\mathcal{T}$ is $r$, the compact t-SVD of $\mathcal{T}$ is $\mathcal{T}=\mathcal{W}_r * \mathcal{S}_r * \mathcal{V}_r^{\top}$, 
where $\mathcal{W}_r=[\mathcal{W}]_{:,1:r,:}$, $\mathcal{V}_r=[\mathcal{V}]_{:,1:r,:}$ and $\mathcal{S}_r=[\mathcal{S}]_{1:r,1:r,:}$. For simplicity, we omit the subscript `$r$' in $\mathcal{W}_r,\mathcal{V}_r$ and $\mathcal{S}_r$  and write $\mathcal{T}=\mathcal{W}*\mathcal{S}*\mathcal{V}^\top$ as the compact t-SVD of $\mathcal{T}$. 
\end{definition}
Next, let us introduce several other key definitions relevant to our work.
\begin{definition}[Moore-Penrose inverse]
$\mathcal{T}^{\dagger} \in \mathbb{K}^{n_2 \times n_1 \times n_3}$ is  the Moore-Penrose inverse of $\mathcal{T} \in \mathbb{K}^{n_1 \times n_2 \times n_3}$, if $\mathcal{T}^{\dagger}$ satisfies 
\begin{align*}
\mathcal{T} \ast \mathcal{T}^{\dagger}\ast \mathcal{T} =\mathcal{T},~&~  \mathcal{T}^{\dagger} \ast  \mathcal{T} \ast \mathcal{T}^{\dagger}=\mathcal{T}^{\dagger},\\
 \left(\mathcal{T} \ast  \mathcal{T}^{\dagger}\right)^{\top}=\mathcal{T} \ast  \mathcal{T}^{\dagger},~&~
\left(\mathcal{T}^{\dagger} \ast  \mathcal{T}\right)^{\top}  =\mathcal{T}^{\dagger} \ast  \mathcal{T} .
\end{align*}
\end{definition}
\begin{definition}[Tensor Frobenius norm]
The   Frobenius norm $\|\mathcal{T}\|_{\fro}$ of   tensor $\mathcal{T}\in\mathbb{K}^{n_1\times n_2\times n_3}$ is defined as
$
\|\mathcal{T}\|_{\fro} := \sqrt{\sum_{i,j,k}|[\mathcal{T}]_{i,j,k}|^2}=
\frac{1}{\sqrt{n_{3}}} \|\textnormal{bcirc}(\mathcal{T})\|_{\fro}.$
\end{definition}

\begin{definition}[Standard tensor lateral basis] \rm
The lateral basis $\mathring{\mathfrak{e}}_{i}\in\mathbb{K}^{n_1 \times 1 \times n_3}$, is the tensor with only $[\mathring{\mathfrak{e}}_{i}]_{i,1,1}$ equal to 1 and the remaining being zero. 
\end{definition}

\begin{definition}[Tensor $\mu_0$-incoherence condition] \label{incoherence}
Let $\mathcal{T} \in \mathbb{K}^{n_1\times n_2\times n_3}$ 
 be a tubal-rank $r$ tensor with a compact t-SVD $\mathcal{T}=\mathcal{W}*\mathcal{S}*\mathcal{V}^{\top}$. Then  $\mathcal{T}$ has $\mu_0$-incoherence  if  for all $ k\in\{1\cdots,n_3\}$, the following  hold:
\begin{equation*}
\max _{i=1,2, ., n_1}\left\|[\widehat{\mathcal{W}}]_{:,:,k}^{\top} \cdot \mathbf{e}_i\right\|_{\fro} \leq \sqrt{\frac{\mu_0 r}{n_1 }} \quad 
 \text{ and } \quad 
\max _{j=1,2, \ldots, n_2}\left\|[\widehat{\mathcal{V}}]_{:,:,k}^{\top}\cdot \mathbf{e}_j\right\|_{\fro} \leq \sqrt{\frac{\mu_0 r}{n_2 }}.
\end{equation*}
In some cases, to highlight the incoherence parameter of a specific tensor $\mathcal{T}$, we will denote this parameter as $\mu_{\mathcal{T}}$.
\end{definition}
As our work focuses on the subtensors of an underlying tensor with low tubal-rank, we will introduce the concept of the sampling tensor for the completeness of this work. 
 \begin{definition}[Sampling tensor]\label{def:samp_tensor}\rm
Given a tensor $\mathcal{T}\in \mathbb{K}^{n_1\times n_2\times n_3}$ and $I\subseteq[n_1]$,  the horizontal subtensor $\mathcal{R}$ of $\mathcal{T}$ with indices $I$ can be obtained via 
\(
\mathcal{R}:=[\mathcal{T}]_{I,:,:}=[\mathcal{I}]_{I,:,:}\ast\mathcal{T},
\) 
where $\mathcal{I}$ is the identity tensor (see \Cref{def:identity}).  
For convenience, $[\mathcal{I}]_{I,:,:}$ will be called the horizontal sampling tensor and denoted as $\mathcal{S}_{I}$.

Similarly, the lateral sub-tensor $\mathcal{C}$ with indices $J\subseteq[n_2]$ can be obtained as $\mathcal{C}:=[\mathcal{T}]_{:,J,:}=\mathcal{T}\ast[\mathcal{I}]_{:,J,:}$.  $[\mathcal{I}]_{:,J,:}$ will be called the lateral sampling tensor and denoted by $\mathcal{S}_{J}$ with   index   $J$. 
The subtensor $\mathcal{U}$ of $\mathcal{T}$ with horizontal indices $I$ and lateral indices $J$ can be represented as $\mathcal{U}:=[\mathcal{T}]_{I,J,:}=\mathcal{S}_I\ast\mathcal{T}\ast\mathcal{S}_{J}$. 
\end{definition}
\begin{definition}[Tensor nuclear norm \citep{
Tarzanagh2018}]\rm
The tensor nuclear norm (\TNN) of $\mathcal{T}$ is defined as 
 $\left\|\mathcal{T}\right\|_{\TNN} 
 = \frac{1}{n_3}\left\|\operatorname{bcirc}{(\mathcal{T})}\right\|_{*}
 = \frac{1}{n_3}\left\|\overline{\mathcal{T}}\right\|_{*}$, 
 where $\left\|\cdot\right\|_{*}$ is the matrix nuclear norm.
 \end{definition}
\begin{definition}[Tensor spectral norm and condition number]\label{spectral_condition_number}\rm 
The tensor spectral norm $\|\mathcal{T}\|$ of a third-order tensor $\mathcal{T}$ is defined as $\|\mathcal{T}\|=\|\bcirc(\mathcal{T})\|$. The condition number of $\mathcal{T}$ is defined as:
$
\kappa(\mathcal{T}) = \|\mathcal{T}^{\dagger}\|\cdot \|\mathcal{T}\|$.
\end{definition}

\section{Proposed sampling model}\label{sampling_model}
We aim to develop a sampling strategy that is both efficient and effective for a range of real-world scenarios. Inspired by the cross-concentrated sampling model for matrix completion \citep{cai2023matrix} and t-CUR decomposition \citep{Tarzanagh2018}, we introduce a novel sampling model tailored for tensor data, named Tensor Cross-Concentrated Sampling (t-CCS). The t-CCS model extracts samples from both horizontal and lateral subtensors of the original tensor. Formally, let \(\mathcal{R} = [\mathcal{T}]_{I,:,:}\) and \(\mathcal{C} = [\mathcal{T}]_{:,J,:}\) represent the selected horizontal and lateral subtensors of \(\mathcal{T}\), determined by index sets \(I\) and \(J\) respectively. We then sample entries on \(\mathcal{R}\) and \(\mathcal{C}\) based on the Bernoulli sampling model. The t-CCS procedure is detailed in Procedure \ref{prod:t-CCS}. Moreover, we illustrate t-CCS against Bernoulli  and t-CUR sampling in \Cref{fig:1x4layout1}. Notably, t-CCS transitions to t-CUR sampling when the samples are dense enough to fully capture the subtensors, and it reverts to Bernoulli sampling when all horizontal and lateral slices are selected. 
The indices of the cross-concentrated samples are denoted by \(\Omega_{\mathcal{R}}\) and \(\Omega_{\mathcal{C}}\), corresponding to the notation used for the subtensors. Our task is to recover an underlying tensor \(\mathcal{T}\) with tubal-rank \(r\) from the observations on \(\Omega_{\mathcal{R}} \cup \Omega_{\mathcal{C}}\).  We approach this recovery through the following optimization problem:  
\begin{equation}  
    \min_{\widetilde{\tX}}~
\left\langle\mathcal{P}_{\Omega_{\mathcal{R}}\cup \Omega_{\mathcal{C}}}(\tX-\widetilde{\tX}), \tX-\widetilde{\tX}\right\rangle, ~~
\text{subject to}~\quad \text{tubal-rank}(\widetilde{\tX})= r,
\end{equation}   
where $\langle \cdot,\cdot\rangle$ is the Frobenius inner product and $\mathcal{P}_{\Omega_{\mathcal{R}}\cup \Omega_{\mathcal{C}}}$ is defined in \eqref{def:sampling_operation}.

\floatname{algorithm}{Procedure}
\begin{algorithm}[!th]
\caption{Tensor Cross-Concentrated Sampling (t-CCS)}\label{prod:t-CCS}
\begin{algorithmic}[1]
\State \textbf{Input:}  $\mathcal{T}\in \mathbb{K}^{n_1 \times n_2 \times n_3}$.
\State  Uniformly select the horizontal and lateral indices, denoted as \(I\) and \(J\), respectively.
\State Set $\mcR:=[\mathcal{T}]_{I,:,:}$ and $\mcC:=[\mathcal{T}]_{:,J,:}$.
\State Sample entries from \(\mathcal{R}\) and \(\mathcal{C}\) based on Bernoulli sampling models. Record the locations of these samples as \(\Omega_{\mathcal{R}}\) and \(\Omega_{\mathcal{C}}\) for \(\mathcal{R}\) and \(\mathcal{C}\), respectively.
\State\textbf{Output:} 
$[\mathcal{T}]_{\Omega_{\mcR}\cup \Omega_{\mcC}}$,
$\Omega_{\mcR}$, $\Omega_{\mcC}$, $I$, $J$.
\end{algorithmic}
\end{algorithm}

\begin{figure}[!th]
    \centering
    \begin{minipage}{0.22\linewidth}
        \includegraphics[width=\linewidth]{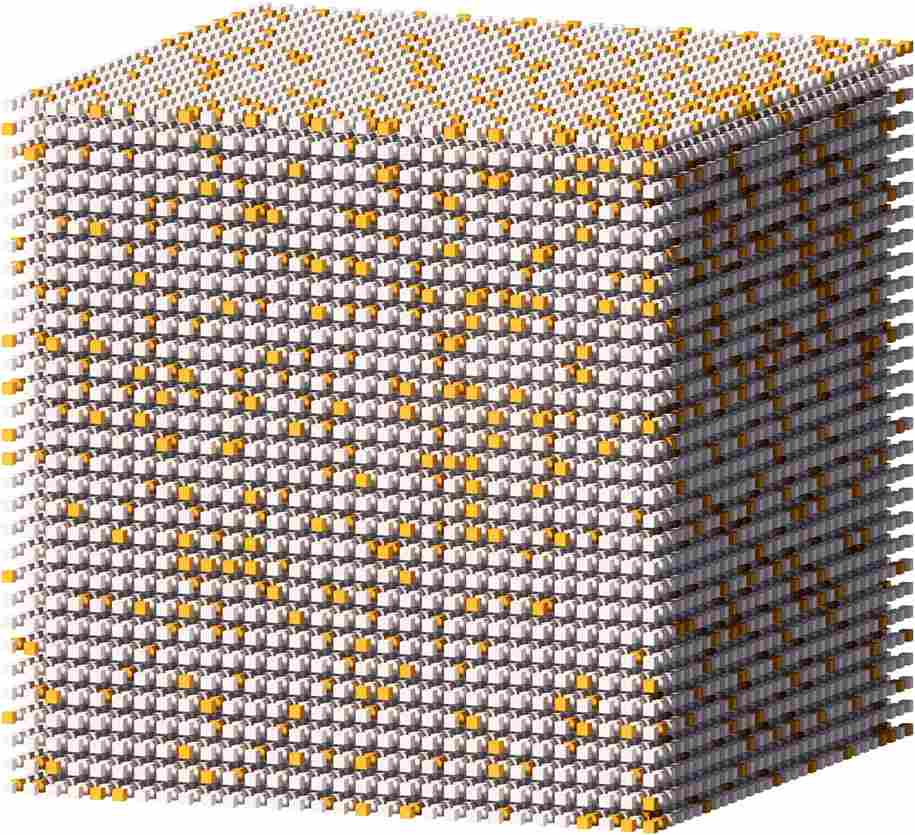}
        \subcaption*{\footnotesize{Bernoulli sampling}}
    \end{minipage}
    \hfill
    \begin{minipage}{0.22\linewidth}
        \includegraphics[width=\linewidth]{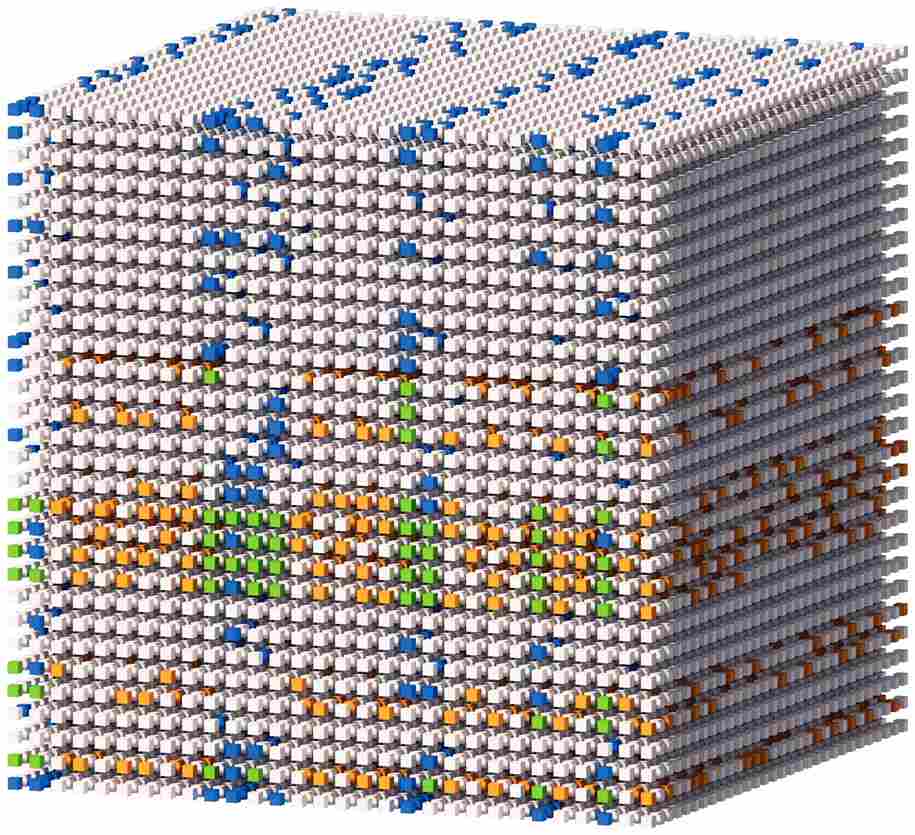}
        \subcaption*{\footnotesize{t-CCS less}}
    \end{minipage}
    \hfill
    \begin{minipage}{0.22\linewidth}
        \includegraphics[width=\linewidth]{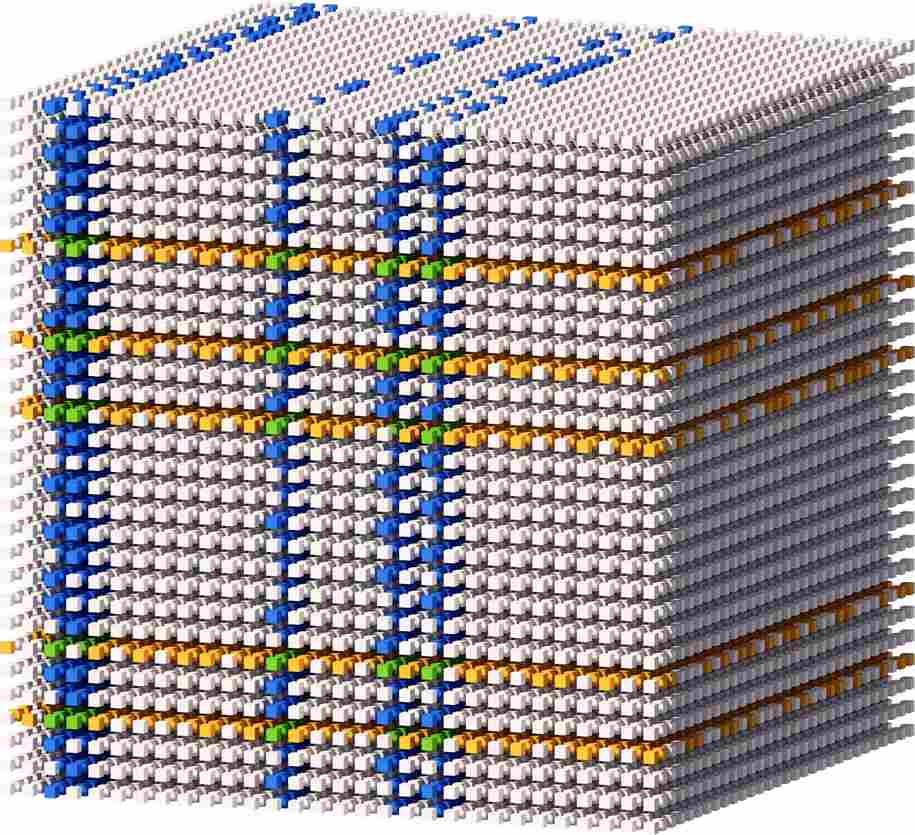}
        \subcaption*{\footnotesize{t-CCS more}}
    \end{minipage}
    \hfill
    \begin{minipage}{0.22\linewidth}
        \includegraphics[width=\linewidth]{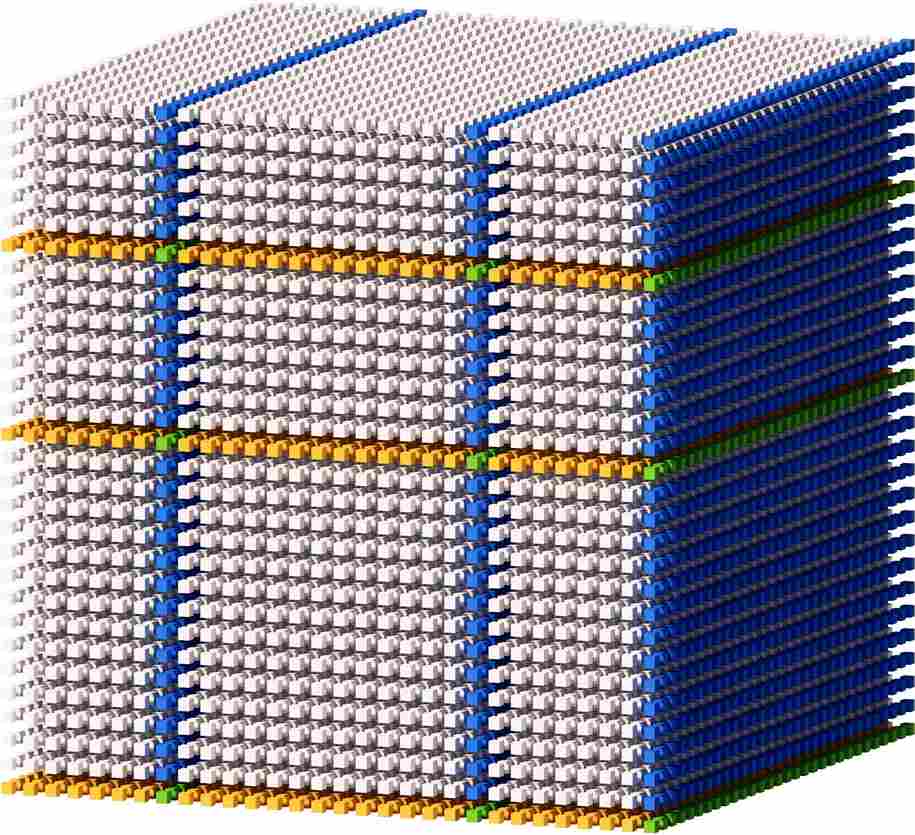}
        \subcaption*{\footnotesize{Tensor CUR sampling}}
    \end{minipage}
    \caption{\footnotesize{Visual comparison of four sampling strategies at the same total observation rate: Bernoulli, less concentrated t-CCS, more concentrated t-CCS, and tensor CUR. The blue, yellow and green grids mark the observed entries in the lateral, horizontal and intersected subtensors, respectively.}}\label{fig:1x4layout1}
\vspace{-6mm}.
\end{figure}


 \vspace{-3mm}
\subsection{Theoretical results}\label{sec:theories}
In this section, we present our main theoretical result, i.e., \Cref{thm:samplingcmp4ccs}, which supports the effectiveness of the proposed t-CCS model. 
\begin{theorem}\label{thm:samplingcmp4ccs} 
Let $\mathcal{T} \in \mathbb{K}^{n_1 \times n_2 \times n_3}$  satisfy the tensor $\mu_0$-incoherence condition  and have a multi-rank $\vec{r}$ with condition number $\kappa$. Let $I\subseteq [n_1]$ and $J \subseteq[n_2]$ be chosen uniformly with replacement to yield $\mathcal{R}=[\mathcal{T}]_{I,:,:}$ and $\mathcal{C}=[\mathcal{T}]_{:,J,:}$. And suppose that $\Omega_{\mcR}$ and $\Omega_{\mcC}$ are generated from $\mathcal{R}$ and $\mcC$  according to the Bernoulli distributions  with probability $p_{\mathcal{R}}$ and $p_{\mathcal{C}}$,  respectively.   If
\begin{align*}
|I|\geq  3200\beta\mu_0r\kappa^2\log^2(n_1n_3+n_2n_3),~&~
|J|\geq  3200\beta\mu_0r\kappa^2\log^2(n_1n_3+n_2n_3),\\
p_{\mathcal{R}} \geq \frac{1600(|I|+n_2)\mu_0r\kappa^2 \log^2((n_1+n_2)n_3) }{|I|n_2},~&~
p_{\mathcal{C}} \geq \frac{1600(|J|+n_1)\mu_0r\kappa^2 \log^2((n_1+n_2)n_3) }{|J|n_1}\,
\end{align*}
for some absolute constant $\beta>1$, 
then $\mathcal{T}$ can be uniquely determined from the entries on $\Omega_{\mathcal{R}}\cup \Omega_{\mathcal{C}}$ with probability at least
\begin{align*}
   1&-\frac{1}{(n_1n_3+n_2n_3)^{800\beta\kappa^2\log(n_2)}}-\frac{1}{(n_1n_3+n_2n_3)^{800\beta\kappa^2\log(n_1)}}\\
   &-\frac{3\log(n_1n_3+|J|n_3)}{(n_1n_3+|J|n_3)^{4\beta-2}}-\frac{3\log(n_2n_3+|I|n_3)}{(n_2n_3+|I|n_3)^{4\beta-2}}.
\end{align*}
\end{theorem}
\begin{remark}
   \begin{enumerate}[(i)]
       \item 
 When $n_1=n_2=n$, the results in the above theorem can be simplified to that 
 $\mathcal{T}$ can be uniquely determined from the entries on $\Omega_{\mathcal{R}}\cup \Omega_{\mathcal{C}}$ with probability at least $1- \frac{6\log(2nn_3)}{(nn_3)^{4\beta-2}}$.
  \item 
Supposed $\mathcal{T}$ with multi-rank $\vec{r}$ of low tubal-rank $r$ is the underlying tensor we aim to recover. 
Notice that such \(\mathcal{T}\) is one of feasible solutions to the optimization problem (\ref{eq:TC model}) since $\text{tubal-rank}(\mathcal{T}) = r$. Additionally, it is evident that for any  $\widetilde{\mathcal{T}}$ with tubal-rank $r$, $\langle\mathcal{P}_{\Omega}(\widetilde{\mathcal{T}}-\mathcal{T}), \widetilde{\mathcal{T}}-\mathcal{T})\rangle \geq 0~ \text{ and }~ \langle\mathcal{P}_{\Omega}(\mathcal{T}-\mathcal{T}), \mathcal{T}-\mathcal{T}\rangle = 0$. Thus, 
\(\mathcal{T}\) is a  global minimizer to the optimization problem (\ref{eq:TC model}). According to \Cref{thm:samplingcmp4ccs}, $\mathcal{T}$ with low tubal-rank $r$ can be reliably recovered using the t-CCS model with high probability. Consequently, we can obtain a minimizer for   (\ref{eq:TC model}) through samples that are partially observed from the t-CCS model.
 \end{enumerate}
\end{remark}
 \Cref{thm:samplingcmp4ccs} elucidates that a sampling complexity of $$\cO(r\kappa^2\max\{n_1,n_2\}n_3\log^2(n_{1}n_{3}+n_{2}n_{3}))$$ is sufficient for TC on t-CCS model. This complexity is  a $\kappa^2$ factor worse than that of the benchmark provided by the state-of-the-art Bernoulli-sampling-based TC methods, such as TNN method detailed in \citep{zhang2017exact}, which demands 
\[
\mathcal{O}(r\max\{n_1,n_2\}n_3\log^2(n_{1}n_{3}+n_{2}n_{3}))
\]
samples.  This observation suggests the potential for identifying a more optimal lower bound, which will leave as  a future direction. 
 
To support the proof of \Cref{thm:samplingcmp4ccs}, we rely on two critical theorems. The first theorem, \Cref{thm:main-tCUR}, establishes the necessary lower bounds for the number of lateral and horizontal slices required when uniformly sampling these slices to ensure an exact t-CUR decomposition. \Cref{thm:main-tCUR} is an adaptation of \citep[Corollary 3.10]{Tarzanagh2018}, which introduces a unique proof method tailored for uniform sampling and exact t-CUR, offering a more detailed analysis for this specific context.
 
Before presenting \Cref{thm:main-tCUR}, let's briefly review the sampling schemes for matrix CUR decomposition \citep{HammHuang2020}. Various sampling schemes have been designed to ensure the selected rows and columns validate the CUR decomposition. For example, deterministic methods are explored in works such as \citep{AB2013,AltschulerGreedyCSSP,LiDeterministicCSSP}.  
Randomized sampling algorithms for CUR decompositions and the column subset selection problem have been extensively studied, as seen in \citep{chiu2013sublinear,
DMM08,DMPNAS,tropp2009column,WZ_2013}.  For a comprehensive overview of both approaches, refer to \citep{hamm2020stability}.  Hybrid methods that combine both approaches are discussed in \citep{
BoutsidisOptimalCUR,cai2021robust}. 

Specifically, for a rank $r$ matrix in  $\mathbb{K}^{n_1\times n_2}$ with $\mu$-incoherence,  sampling $\cO(\mu r\log(n_1))$ rows and $\cO(\mu r\log(n_2))$ columns is sufficient to ensure the exact matrix CUR decomposition \citep{cai2019accelerated
}. In this work, we extend the uniform sampling results from the matrix setting to the tensor setting. 
\begin{theorem}\label{thm:main-tCUR}\rm
Consider a tensor $\mathcal{T} \in \mathbb{K}^{n_1 \times n_2 \times n_3}$ that  satisfies the tensor $\mu_0$-incoherence condition and has a multi-rank $\vec{r}$.   The indices $I$ and $J$ are selected uniformly randomly without replacement from $[n_1]$ and $[n_2]$,  respectively.  
Set $\mcC=[\X]_{:,J,:}$, $\mcR=[\X]_{I,:,:}$ and $\mathcal{U}=[\X]_{I,J,:}$.  Then $\mathcal{T}=\mcC*\mathcal{U}^\dagger*\mcR$ holds
 with   probability  at least $1-\frac{1}{n_{1}^{\beta}}-\frac{1}{n_{2}^{\beta}}$, provided that $|I|\geq 2\beta \mu_0 \|\vec{r}\|_{\infty} \log \left(n_1\|\vec{r}\|_{1}\right)$ and  $|J|\geq 2\beta \mu_0 \|\vec{r}\|_{\infty} \log \left(n_2\|\vec{r}\|_{1}\right)$. 
\end{theorem}
Another important supporting theorem is  \Cref{thm:sr4generalTCB}, which adapts \citep[Theorem~3.1]{zhang2017exact} for tensor recovery with tubal-rank \( r \) under Bernoulli sampling, essential for  \Cref{thm:samplingcmp4ccs}. Our contribution refines the theorem by explicitly detailing the numerical constants in the original sampling probability. The proof of \Cref{thm:sr4generalTCB} is under the same framework as in \citep{
zhang2017exact}.
\begin{theorem}\label{thm:sr4generalTCB}\rm
Let $\mathcal{Z}\in\mathbb{K}^{n_1 \times n_2 \times n_3}$ of tubal-rank $r$ satisfy the tensor $\mu_0$-incoherence condition. And its compact t-SVD is  $\mathcal{Z}=\mathcal{W} * \mathcal{S} * \mathcal{V}^{\top}$ where $\mathcal{W} \in \mathbb{K}^{n_1 \times r \times n_3}, \mathcal{S} \in \mathbb{K}^{r \times r \times n_3}$ and $\mathcal{V} \in \mathbb{K}^{n_2 \times r \times n_3}$. Suppose the entries in $\Omega$ are sampled according to the Bernoulli model with probability $p$.  If
\begin{equation}\label{ppp}
p \geq \frac{256\beta(n_1+n_2)\mu_0r\log^2(n_1n_3+n_2n_3)}{n_1n_2} \text{~with~}\beta\geq 1,
\end{equation}
then $\mathcal{Z}$ is the unique minimizer to 

\begin{equation*}
\begin{gathered}
\min _{\mathcal{X}}\|\mathcal{X}\|_{\TNN}, ~
\textnormal{ subject to } \mathcal{P}_{\Omega}(\mathcal{X})=\mathcal{P}_{\Omega}(\mathcal{Z}),
\end{gathered}
\end{equation*}
 with probability at least $1- \frac{3\log(n_1n_3+n_2n_3)}{(n_1n_3+n_2n_3)^{4\beta-2}}$.
\end{theorem} 


\section{An efficient solver for t-CCS }\label{solver}  
In this section, we explore how to solve the t-CCS-based TC problem efficiently. Initially,  we  apply existing TC algorithms, including BCPF \citep{zhao2015bayesian}, TMac \citep{xu2013parallel}, TNN \citep{zhang2017exact}, and F-TNN \citep{jiang2020framelet}, to a t-CCS-based image recovery problem. BCPF is CP-based, TMac is Tucker-based, while TNN and F-TNN are tubal-rank-based. However, these methods prove unsuitable for the t-CCS model, as shown in \Cref{failure_pics}, where they fail to produce reliable visualization. This highlights the need for new algorithm(s) tailored to the t-CCS model.
\begin{figure}[!th]
\begin{minipage}{0.16\linewidth}
        \centering \footnotesize{Ground truth} 
    \end{minipage}
    \begin{minipage}{0.16\linewidth}
        \centering \footnotesize{Observed} 
    \end{minipage}
    \begin{minipage}{0.16\linewidth}
        \centering \footnotesize{BCPF}
    \end{minipage}
    \begin{minipage}{0.16\linewidth}
        \centering \footnotesize{TMac} 
    \end{minipage}
    \begin{minipage}{0.16\linewidth}
        \centering \footnotesize{TNN} 
    \end{minipage}
       \begin{minipage}{0.16\linewidth}
        \centering \footnotesize{F-TNN} 
    \end{minipage}\\
\includegraphics[width=0.16\linewidth]{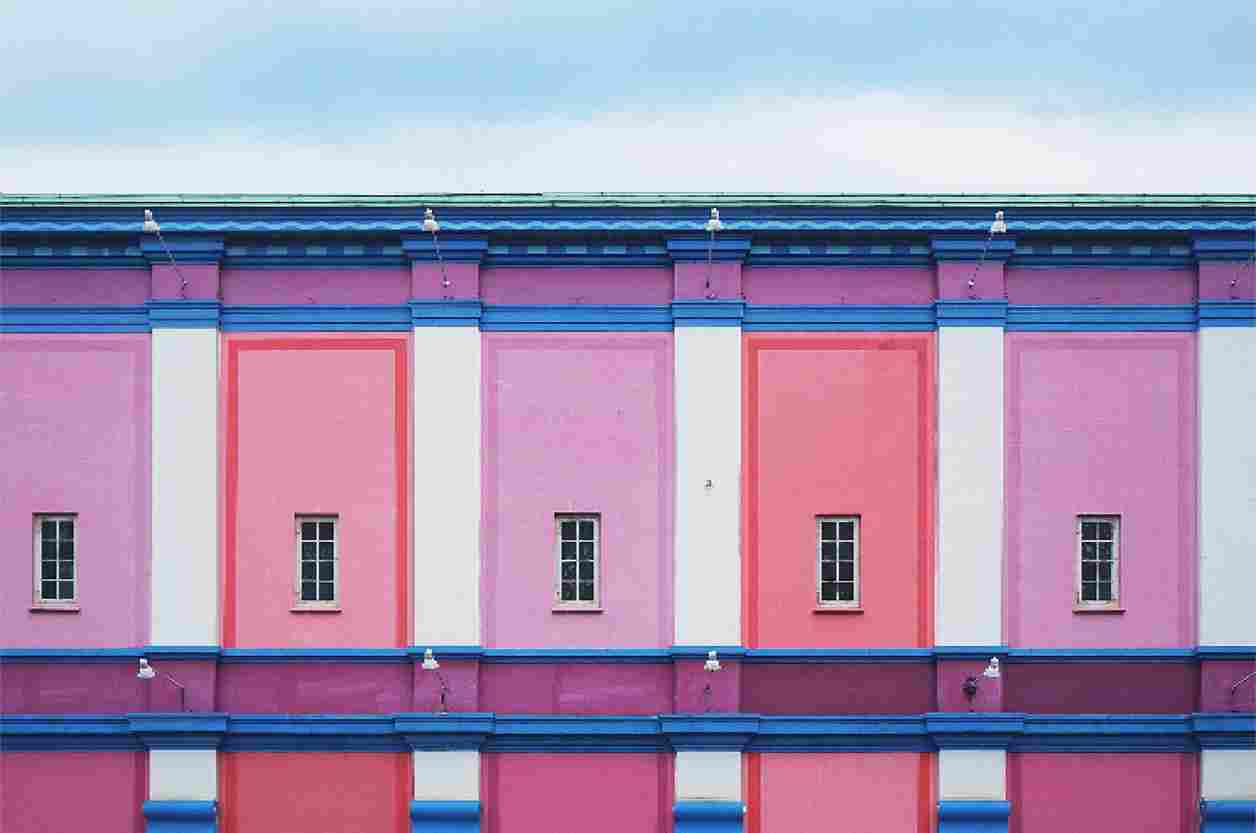}
\includegraphics[width=0.16\linewidth]{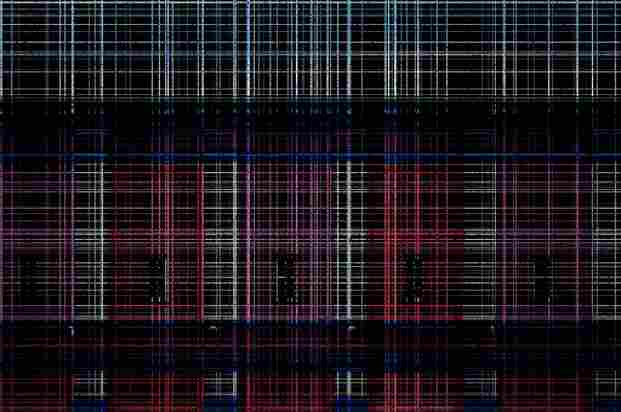} \includegraphics[width=0.16\linewidth]{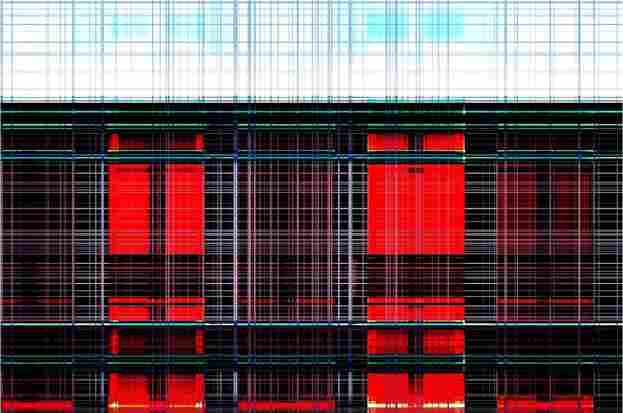} 
\includegraphics[width=0.16\linewidth]{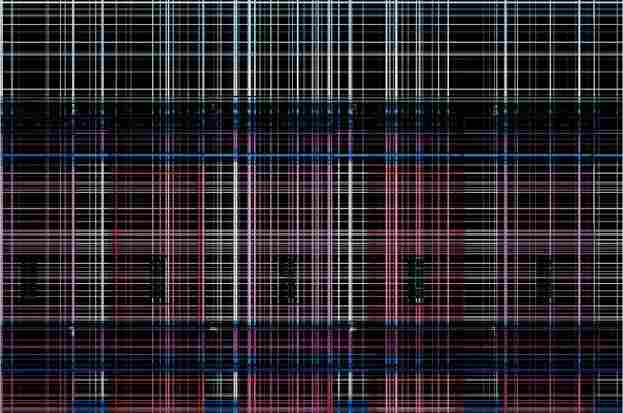} 
\includegraphics[width=0.16\linewidth]{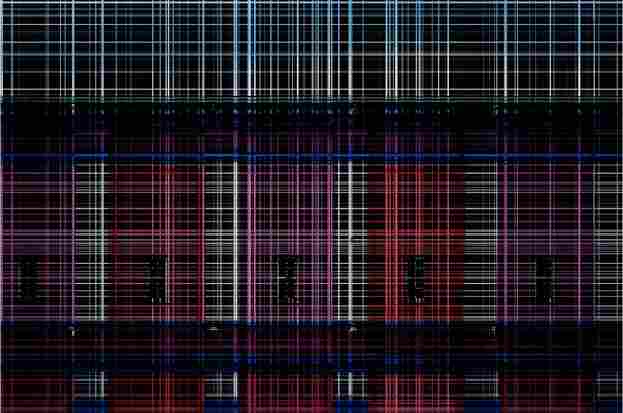} 
\includegraphics[width=0.16\linewidth]{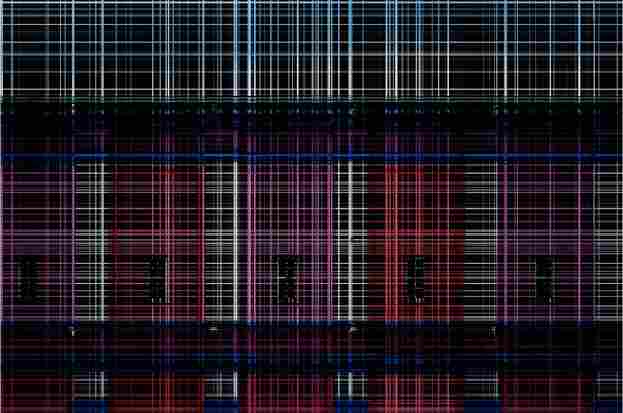}
  \caption{\footnotesize{Visual results of color image inpainting using t-CCS samples at an overall sampling rate of $20\%$ with BCPF, TMac, TNN, and F-TNN algorithms.}}\label{failure_pics} 
  \vspace{-12mm}
\end{figure} 
\subsection{Iterative t-CUR tensor completion algorithm}
To efficiently use the t-CCS structure, we develop the Iterative t-CUR Tensor Completion (ITCURTC) algorithm, a non-convex method inspired by projected gradient descent.  ITCURTC updates $\mathcal{R}$, $\mathcal{C}$, and $\mathcal{U}$ at each iteration to preserve the tubal-rank $r$ of $\mathcal{T}$. The update formulas are:
\begin{align}
&\left[\mathcal{R}_{k+1}\right]_{:, J^{\complement},:}  :=\left[\mathcal{T}_k\right]_{I, J^{\complement},:}+\eta_R\left[\mathcal{P}_{\Omega_\mathcal{R}}\left(\mathcal{T}-\mathcal{T}_k\right)\right]_{I, J^{\complement},:}, \label{eqn:Rcupdate}\\
&\left[\mathcal{C}_{k+1}\right]_{I^{\complement},:,: }  :=\left[\mathcal{T}_k\right]_{I^{\complement}, J}+\eta_C\left[\mathcal{P}_{\Omega_C}\left(\mathcal{T}-\mathcal{T}_k\right)\right]_{I^{\complement}, J,:}, \label{eqn:Ccupdate}\\
&\mathcal{U}_{k+1} :=\mathcal{H}_r\left(\left[\mathcal{T}_k\right]_{I, J,:}+\eta_{U}\left[\mathcal{P}_{\Omega_R \cup \Omega_C}\left(\mathcal{T}-\mathcal{T}_k\right)\right]_{I, J,:}\right), \label{eqn:Uupdate}
\end{align}
where $\eta_R, \eta_C, \eta_U$ are step sizes, and $\mathcal{H}_r$ is the truncated t-SVD operator. $\left[\mathcal{R}_{k+1}\right]_{:, J,:}$ and $\left[\mathcal{C}_{k+1}\right]_{I,:,:}$ are updated to $\mathcal{U}_{k+1}$.  The algorithm, starting from $\mathcal{T}_0=\mathbf{0}$, iterates until 
$e_k \leq \varepsilon$,
 where   $\varepsilon$ is a preset tolerance and 
 \setlength\abovedisplayskip{1pt}
\setlength\belowdisplayskip{1pt}
 \begin{equation}\label{eqn:relative error}
e_k=\frac{\left\langle\mathcal{P}_{\Omega_\mathcal{R} \cup \Omega_\mathcal{C}}\left(\mathcal{T}-\mathcal{T}_k\right),\mathcal{T}-\mathcal{T}_k\right\rangle}{\left\langle\mathcal{P}_{\Omega_{\mathcal{R}} \cup \Omega_{\mathcal{C}}}(\mathcal{T}),\mathcal{T}\right\rangle}.  
 \end{equation}
 The algorithm is summarized in \Cref{ticur}. 
\floatname{algorithm}{Algorithm}
\begin{algorithm}[!th]
\caption{Iterative t-CUR Tensor Completion (ITCURTC) for t-CCS }\label{ticur}
\begin{algorithmic}[1]
\State \textbf{Input:} 
$[\mathcal{T}]_{\Omega_{\mcR}\cup \Omega_{\mcC}}:$ observed data;
$\Omega_\mathcal{R}$, $\Omega_\mathcal{C}:$ observed locations; $I$, $J:$ horizontal and lateral indexes that define $\mathcal{R}$ and $\mathcal{C}$ respectively; $\eta_{R}, \eta_{C}, \eta_{U}:$ step sizes; $r:$ target tubal-rank; $\varepsilon:$ tolerance level.
\State Set 
$\mathcal{T}_0 = 0 \in \mathbb{K}^{n_1\times n_2\times n_3}$
\While{$e_k>\varepsilon$} \algorithmiccomment{\textcolor{officegreen}{$e_k$ is defined in \eqref{eqn:relative error}}}
\State $[\mathcal{R}_{k+1}]_{:,J^{\complement},:} = [\mathcal{T}_k]_{I,J^{\complement},:} + \eta_{R} [\mathcal{P}_{\Omega_\mathcal{R}}([\mathcal{T}]_{\Omega_{\mcR}\cup \Omega_{\mcC}} -\mathcal{T}_k)]_{I,J^{\complement},:}$
\State $[\mathcal{C}_{k+1}]_{I^{\complement},:,:} = [\mathcal{T}_k]_{I^{\complement},J,:} + \eta_{C} [\mathcal{P}_{\Omega_\mathcal{C}}([\mathcal{T}]_{\Omega_{\mcR}\cup \Omega_{\mcC}} -\mathcal{T}_k)]_{I^{\complement}, J, :}$
\State $\mathcal{U}_{k+1} =  \mathcal{H}_{r}([\mathcal{T}_k]_{I,J,:} + \eta_{U}[\mathcal{P}_{\Omega_{\mathcal{R}}\cup \Omega_{\mathcal{C}}}([\mathcal{T}]_{\Omega_{\mcR}\cup \Omega_{\mcC}}-\mathcal{T}_k)]_{I,J,:})$
\State $[\mathcal{R}_{k+1}]_{:,J,:} = \mathcal{U}_{k+1}$ 
\quad and \quad  $[\mathcal{C}_{k+1}]_{I,:,:} = \mathcal{U}_{k+1}$
\State  Update $\mathcal{T}_{k+1}$
\Comment{\textcolor{officegreen}{More details see \eqref{IcJ}, \eqref{IJc}, \eqref{IJ}}}
\State $k = k + 1$
\EndWhile
\State \textbf{Output:} $\mathcal{C}_{k+1}$, $\mathcal{U}_{k+1}$ and $\mathcal{R}_{k+1}$
\end{algorithmic}
\end{algorithm}
\subsection{Computational complexity analysis} \label{computational_cost}
This section outlines the implementation and computational costs of \Cref{ticur}. Updating $[\mathcal{R}_{k+1}]_{:,J^{\complement},:}$ and $[\mathcal{C}_{k+1}]_{I^{\complement},:,:}$ incurs $\mathcal{O}\left(|\Omega_{\mcR}|+|\Omega_{\mcC}|-|\Omega_{\mathcal{U}}|\right)$ flops, focusing only on observed locations (refer to \eqref{eqn:Rcupdate} and \eqref{eqn:Ccupdate}). The update of $\mathcal{U}_{k+1}\in\mathbb{K}^{|I|\times|J|\times n_3}$ involves (i) computing $\widetilde{\mathcal{U}}_{k+1}:=[\mathcal{T}_k]_{I,J,:}+ \eta_{U}[\mathcal{P}_{\Omega_{\mathcal{R}}\cup \Omega_{\mathcal{C}}}([\mathcal{T}]_{\Omega_{\mcR}\cup \Omega_{\mcC}}-\mathcal{T}_k)]_{I,J,:}$ and (ii) finding its tubal-rank $r$ approximation via t-SVD. The cost for (i) is $\cO(|\Omega_{\mathcal{U}}|)$, while (ii) requires $\max\{\cO(|I||J|rn_3),\cO(|I||J|n_3\log(n_3))\}$, making the total update cost for $\mathcal{U}_k$ to be $\max\{\cO(|I||J|rn_3),\cO(|I||J|n_3\log(n_3))\}$.

Considering the cost of updating $\mathcal{T}_{k+1}$ in \Cref{ticur}, we focus on $[\mathcal{T}_{k+1}]_{I^{\complement},J,:}$, $[\mathcal{T}_{k+1}]_{I,J^{\complement},:}$, and $[\mathcal{T}_{k+1}]_{I,J,:}$ each iteration. The update for $[\mathcal{T}_{k+1}]_{I^{\complement}, J,:}$ is:
\begin{equation}\label{IcJ}
    [\mathcal{T}_{k+1}]_{I^{\complement}, J,:} = [\mathcal{C}_k]_{I^{\complement},:,:}*\mathcal{U}_k^{\dagger}*\mathcal{U}_k  = [\mathcal{C}_k]_{I^{\complement},:,:}*[\mathcal{V}_{k}]_{:,1:r,:}*[\mathcal{V}_{k}]_{:,1:r,:}^{\top}.
\end{equation}
where $\mathcal{U}_k=\mathcal{W}_k*\S_k*\mathcal{V}_{k}^\top$ is $\mathcal{U}_k$'s t-SVD. 
Given $[\mathcal{V}_{k}]_{:,1:r,:}$'s size as $|J|\times r \times n_3$, the computational cost is $\mathcal{O}(n_{1}|J|rn_{3})$ flops for \eqref{IcJ}, making the total complexity for updating $[\mathcal{T}_{k+1}]_{I^{\complement}, J,:}$ also $\mathcal{O}(n_{1}|J|rn_{3})$ flops.  {We update $[\mathcal{T}_{k+1}]_{I, J^{\complement},:}$ by \begin{equation}\label{IJc}
[\mathcal{T}_{k+1}]_{I, J^{\complement},:}:=[\mathcal{W}_{k}]_{:,1:r,:}*[\mathcal{W}_{k}]_{:,1:r,:}^{\top}*[\mathcal{R}_k]_{:, J^{\complement},:}.
\end{equation}} 
Similar analysis for   updating $[\mathcal{T}_{k+1}]_{I,J^{\complement},:}$, the computational complexity  of updating $[\mathcal{T}_{k+1}]_{I,J^{\complement},:}$ is $\mathcal{O}(n_{2}|I|rn_{3})$. And we update 
$[\mathcal{T}_{k+1}]_{I, J,:}$ by setting 
\begin{equation}\label{IJ}
[\mathcal{T}_{k+1}]_{I, J,:}:= \mathcal{U}_k.
\end{equation}  Thus, the total computational complexity of updating $\mathcal{T}_{k+1}$ is $\mathcal{O}(|I|rn_{2}n_{3}+|J|rn_{1}n_{3})$. 

Computing the stopping criterion $e_k$ costs $\mathcal{O}\left(|\Omega_{R}|+|\Omega_{C}|-|\Omega_{U}|\right)$ flops
as we only make computations on the observed locations. 

The computational costs per iteration are summarized in \Cref{table:complexity}, showing a complexity of $\mathcal{O}(r|I|n_{2}n_{3}+r|J|n_{1}n_{3})$ when $|I|\ll n_1$ and $|J|\ll n_2$.
\begin{table}[!th]
\centering
\caption{\footnotesize{Computational costs per iteration for ITCURTC.}}\label{table:complexity}
\resizebox{\linewidth}{!}{ \begin{tabular}{cc|ccc|ccc}
\toprule
  &\textsc{Step}        & \textsc{Computational Complexity}        \\ \midrule
   &\small Line 3: Computing the stopping criterion $e_k$ & \small  $\mathcal{O}\left(|\Omega_{R}|+|\Omega_{C}|-|\Omega_{U}|\right)$ \\
 &\small{Line 4:  $[\mathcal{R}_{k+1}]_{:,J^{\complement},:} = [\mathcal{T}_k]_{I,J^{\complement},:} + \eta_{R} [\mathcal{P}_{\Omega_\mathcal{R}}([\mathcal{T}]_{\Omega_{\mcR}\cup \Omega_{\mcC}} -\mathcal{T}_k)]_{I,J^{\complement},:}$}      &\small $\mathcal{O}\left(|\Omega_{R}|-|\Omega_{U}|\right)$   \\ 
 &\small {Line 5:  $[\mathcal{C}_{k+1}]_{I^{\complement},:,:}= [\mathcal{T}_k]_{I^{\complement},J,:} + \eta_{C} [P_{\Omega_\mathcal{C}}([\mathcal{T}]_{\Omega_{\mcR}\cup \Omega_{\mcC}} -\mathcal{T}_k)]_{I^{\complement}, J, :}$}&\small $\mathcal{O}\left(|\Omega_{C}|-|\Omega_{U}|\right)$ \\ 
 &\small Line 6: $\mathcal{U}_{k+1}=\mathcal{H}_{r}([\mathcal{T}_k]_{I,J,:}+ \eta_{U}[\mathcal{P}_{\Omega_{\mathcal{R}}\cup \Omega_{\mathcal{C}}}([\mathcal{T}]_{\Omega_{\mcR}\cup \Omega_{\mcC}}-\mathcal{T}_k)]_{I,J,:})$ & \small $\cO( \max\{|I||J|rn_3,|J||I|n_3\log(n_3)\})$ \\ 
  &\small Line 8: Updating $\mathcal{T}_{k+1}$  & \small $\mathcal{O}(r|I|n_{2}n_{3}+r|J|n_{1}n_{3})$ \\ 
\bottomrule
\end{tabular}
}
\vspace{-6mm}
\end{table}
\section{Numerical experiments}\label{experiment}
This section presents the performance of our t-CCS based ITCURTC  through numerical experiments on both synthetic and real-world data. The computations are performed on one of the shared nodes of the Computing Cluster with a 64-bit Linux system (GLNXA64), featuring an Intel(R) Xeon(R) Gold 6148 CPU (2.40 GHz). All experiments are carried out using MATLAB 2022a.
\subsection{Synthetic  data examples}
This section evaluates ITCURTC for t-CCS-based TC, exploring the needed sample sizes and the impact of Bernoulli sampling probability and fiber sampling rates on low-tubal-rank tensor recovery. 
\begin{figure}[!th]
\centering
\begin{minipage}[b]{0.32\linewidth}
\centering
\footnotesize{$r = 2$} \\
\vspace{0.5cm}
\includegraphics[width=1\linewidth]{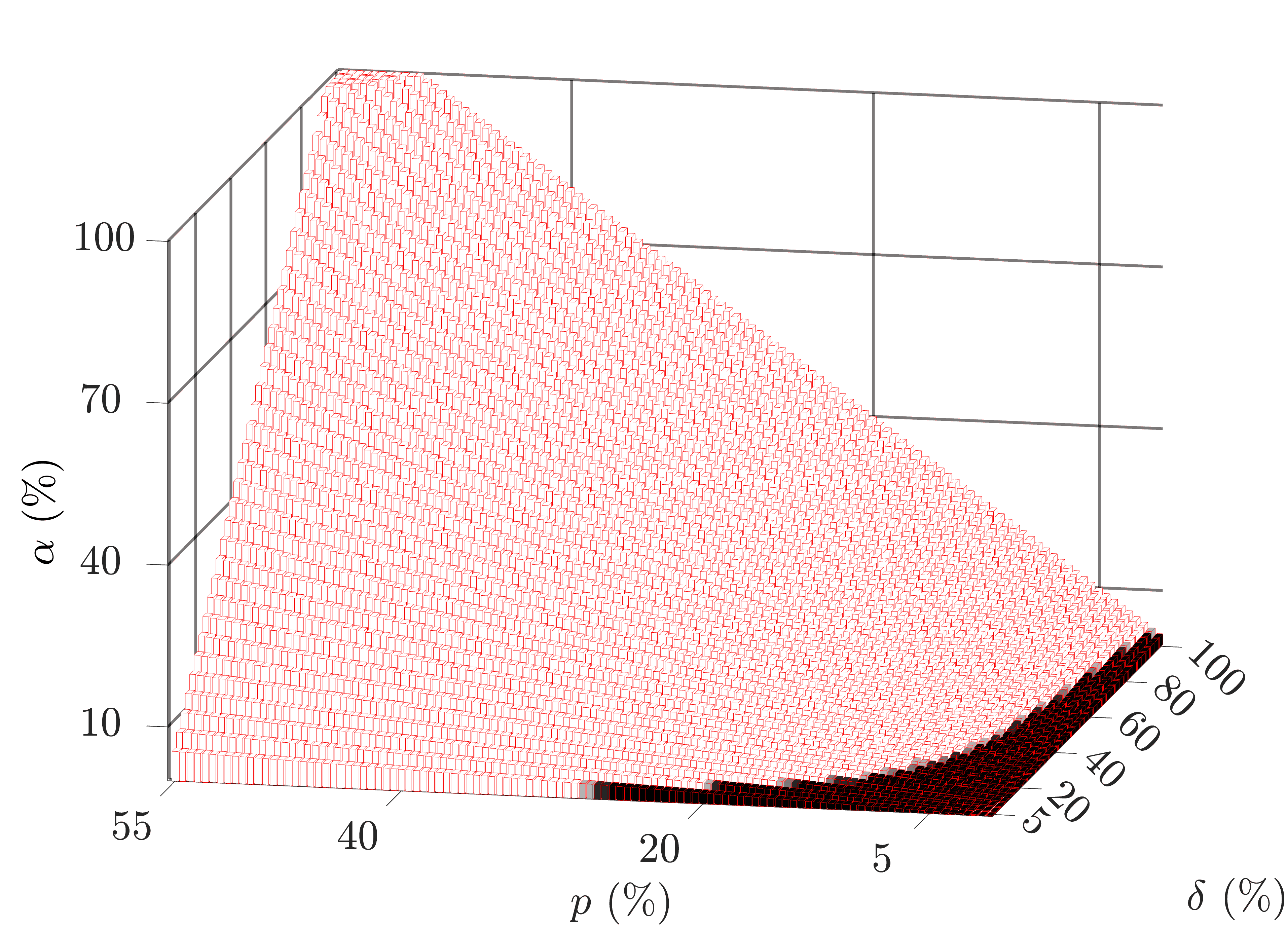} \\
\vspace{0.5cm}
\hspace*{-0.5cm}
\includegraphics[width=0.9\linewidth]{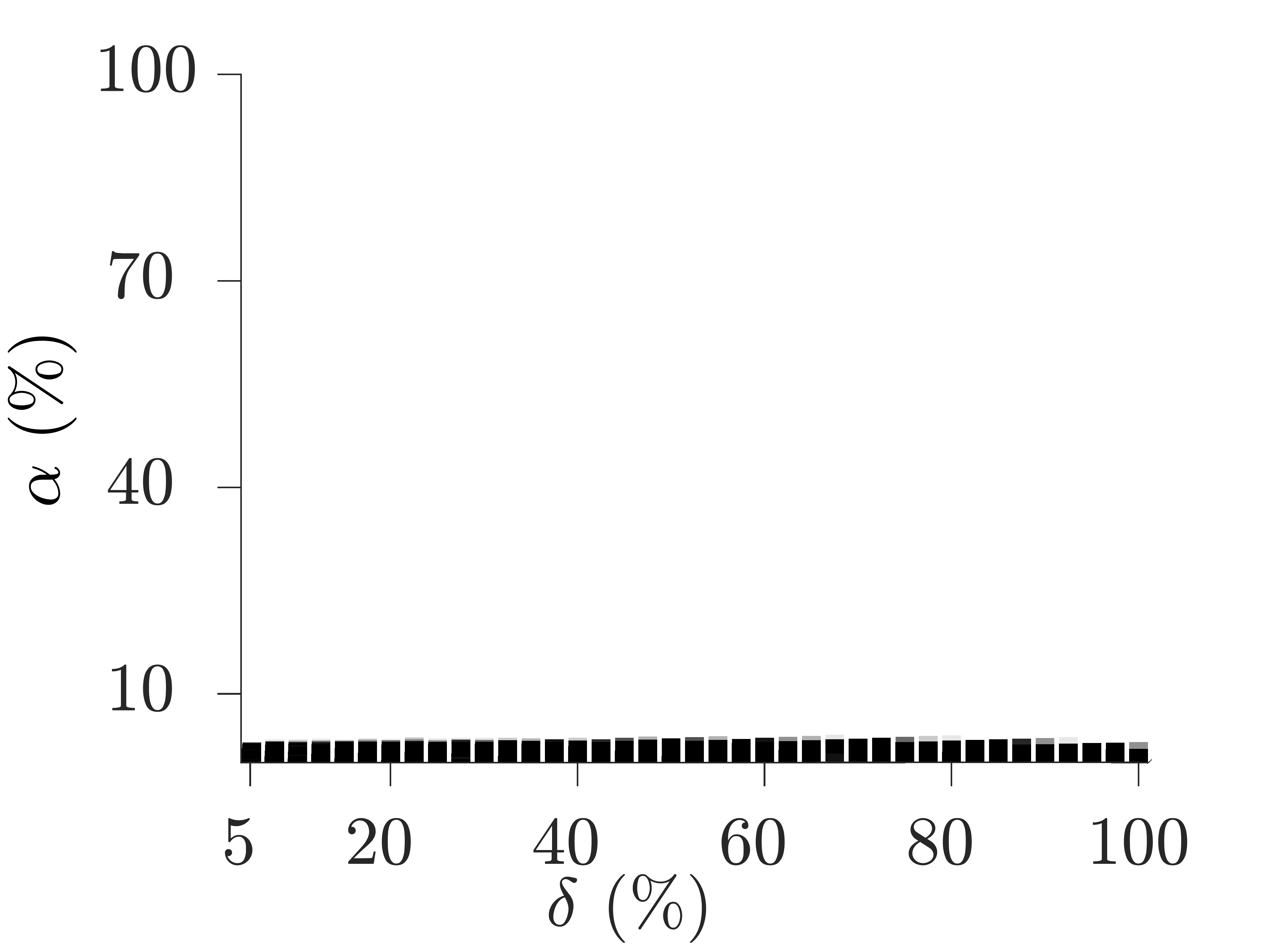} \\
\includegraphics[width=1\linewidth]{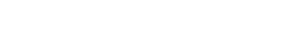}
\end{minipage}
\hfill
\begin{minipage}[b]{0.32\linewidth}
\centering
\footnotesize{$r = 5$} \\
\vspace{0.5cm}
\includegraphics[width=1\linewidth]{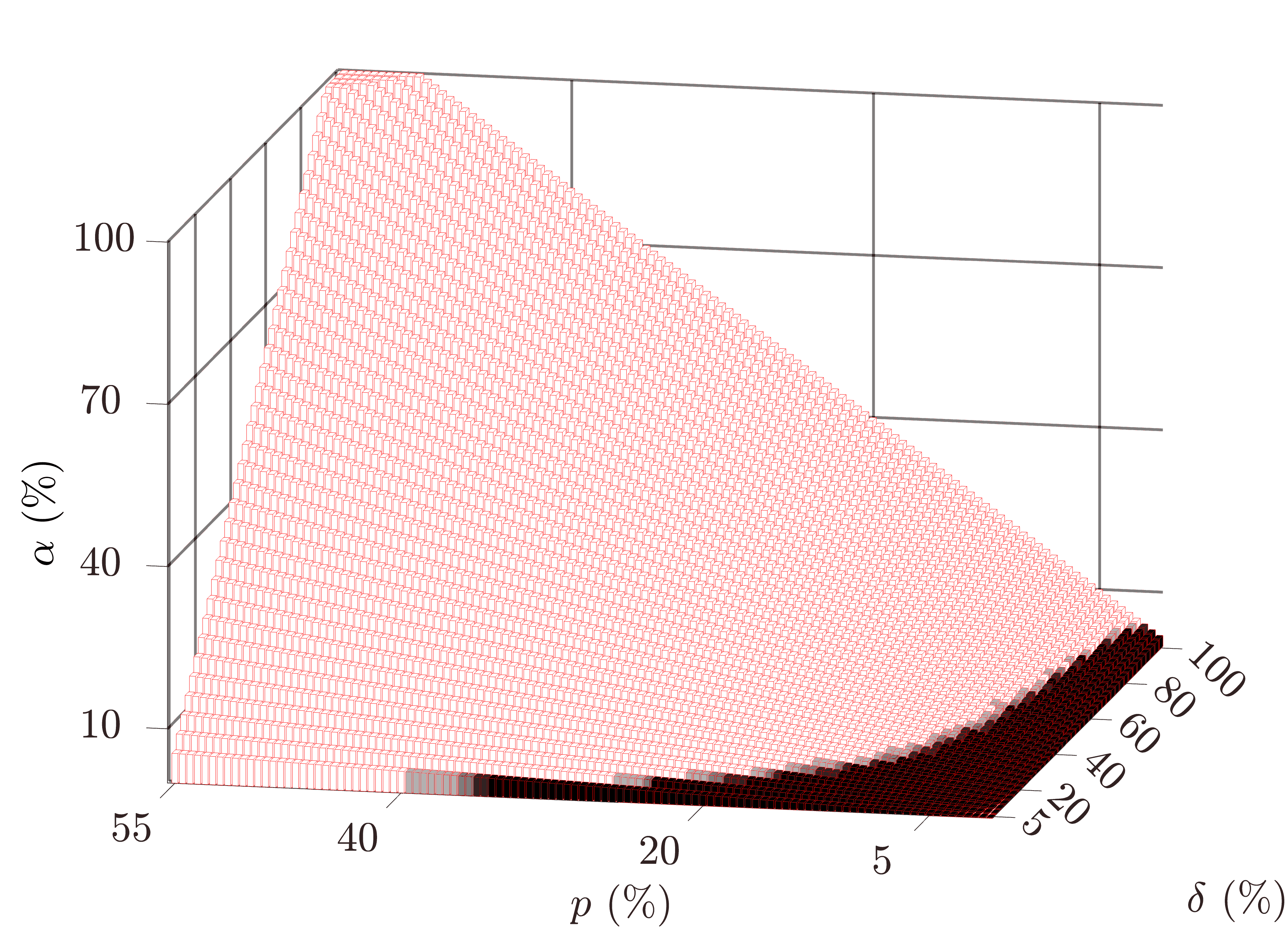} \\
\vspace{0.5cm}
\hspace*{-0.5cm}
\includegraphics[width=0.9\linewidth]{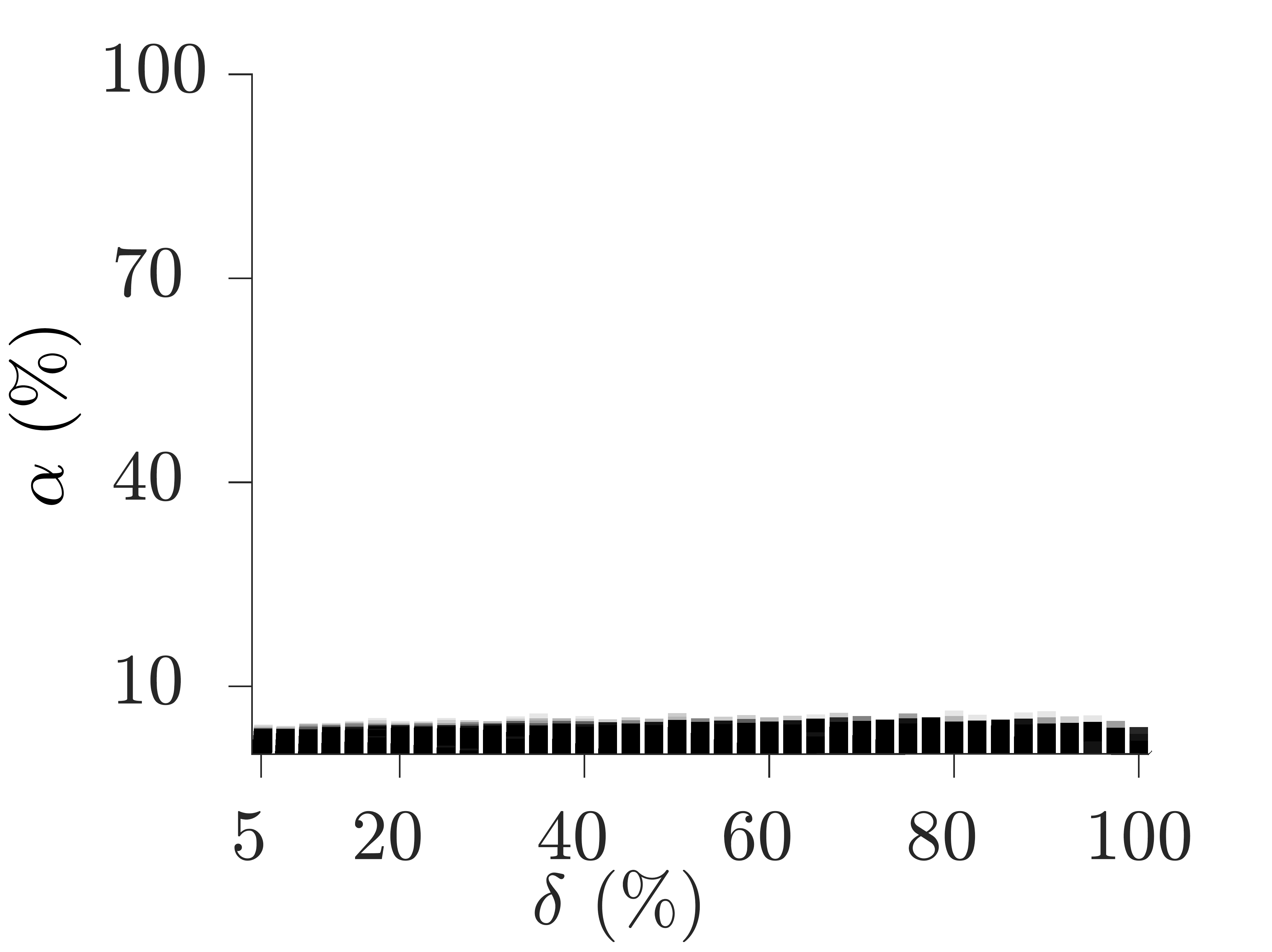} \\
\includegraphics[width=1\linewidth]{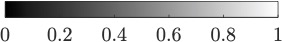}
\end{minipage}
\hfill
\begin{minipage}[b]{0.32\linewidth}
\centering
\footnotesize{$r = 7$} \\
\vspace{0.5cm}
\includegraphics[width=1\linewidth]{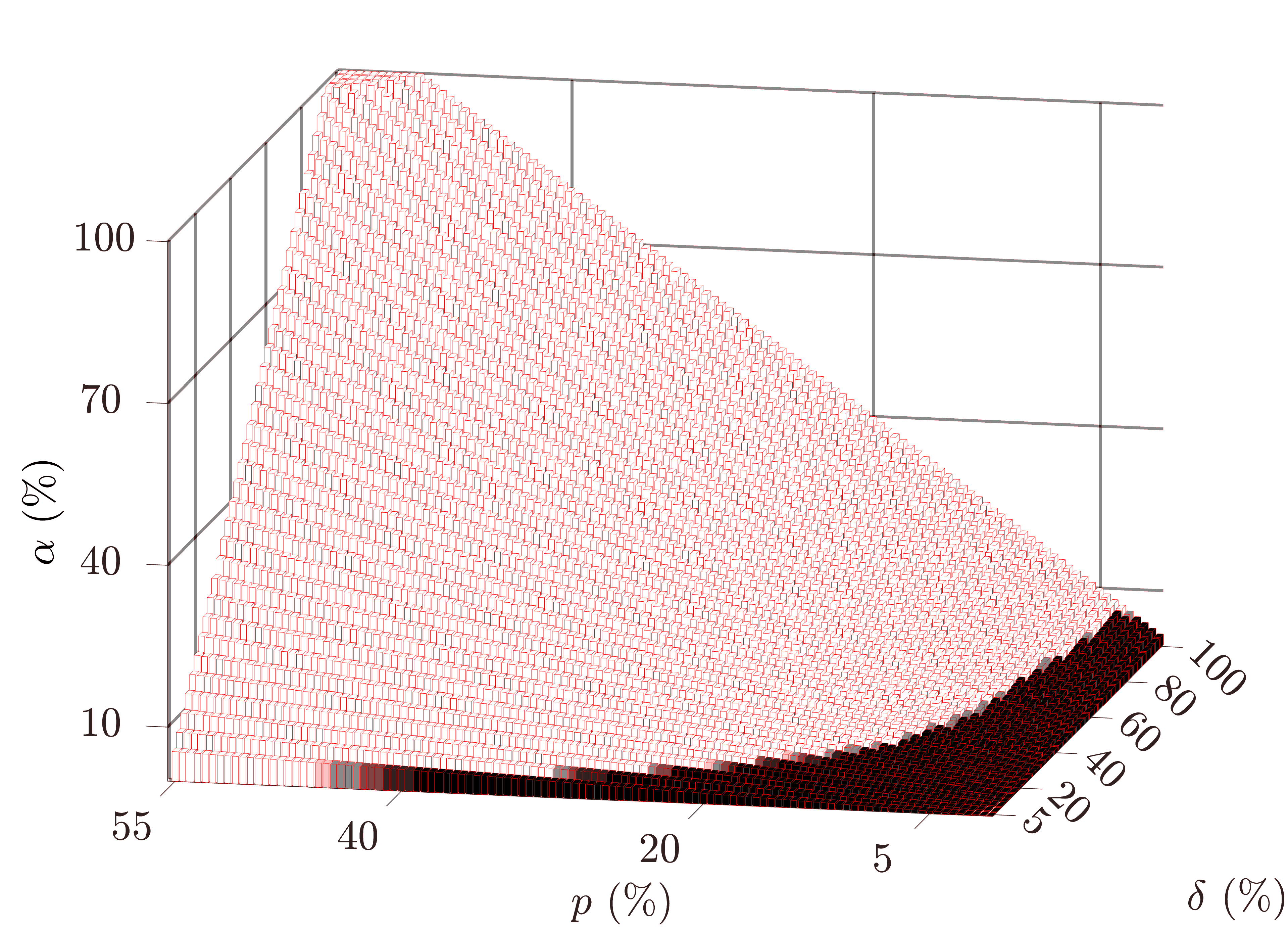} \\
\vspace{0.5cm}
\hspace*{-0.5cm}
\includegraphics[width=0.9\linewidth]{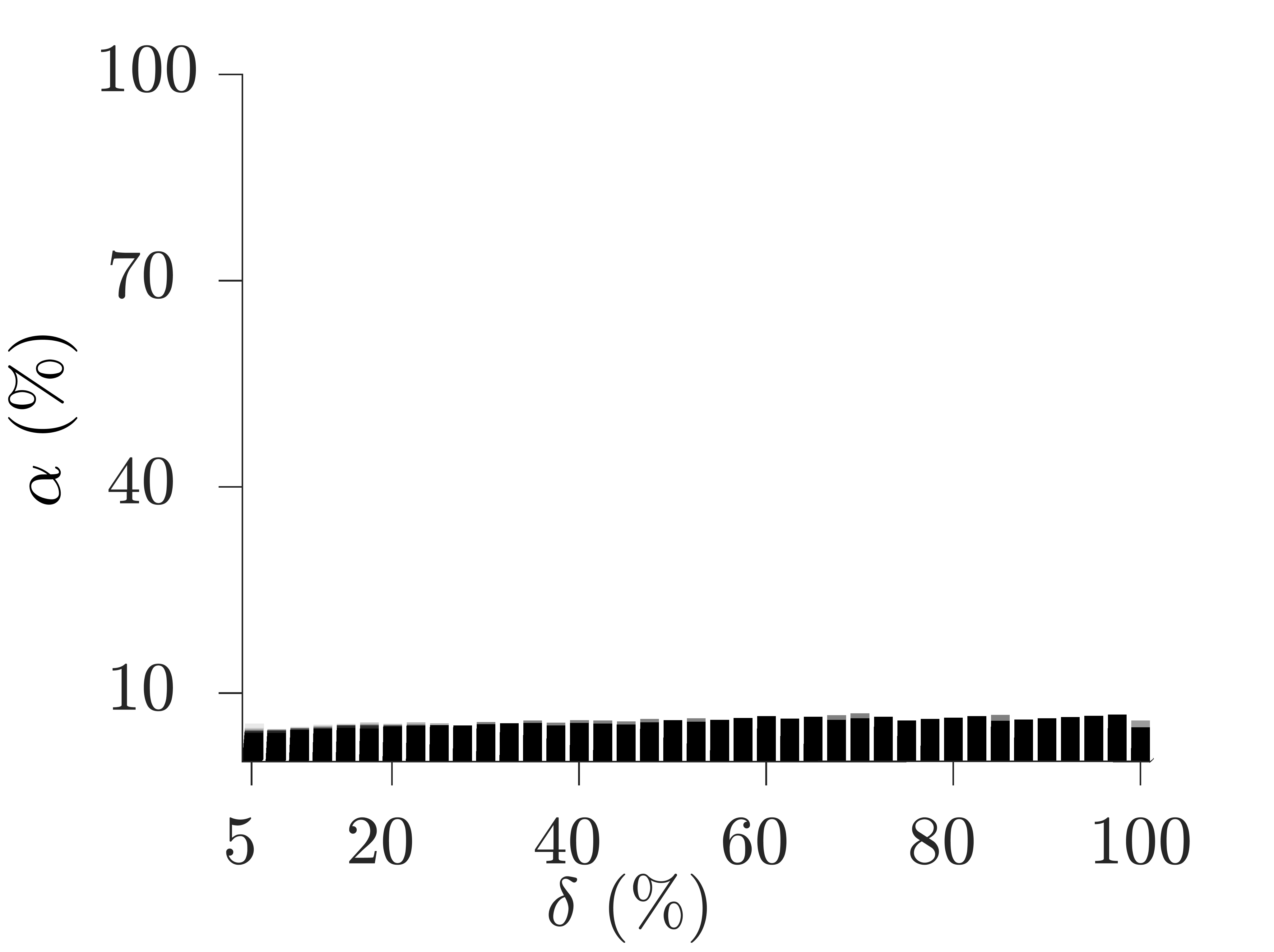} \\
\includegraphics[width=1\linewidth]{figure/phase/white_image.jpg}
\end{minipage}
\caption{\footnotesize{(\textbf{Row 1}) 3D and (\textbf{Row 2}) 2D views illustrate ITCURTC's empirical phase transition for the t-CCS model. $\delta=|I|/768=|J|/768$ shows sampled indices ratios, $p$ is the Bernoulli sampling probability over subtensors, and $\alpha$ is the overall tensor sampling rate. White and black in the $768\times 768\times 256$ tensor results represent success and failure, respectively, across 25 tests for tubal ranks 2, 5, and 7 (Columns 1-3). The $\alpha$ needed for success remains consistent across different combinations $\delta$ and $p$.}
}\label{emp}
\vspace{-7mm}
\end{figure}
We assess ITCURTC's tensor recovery capability under different combinations of horizontal and lateral slice numbers \(|I| = \delta n_{1}, |J| = \delta n_{2}\) and Bernoulli sampling rates \(p\) on selected subtensors. The study uses tensors of size \(768 \times 768\times 256\) with tubal-ranks \(r \in\{2,5, 7\}\). We conduct $25$ tests for each \((\delta,p,r)\) tuple, a test is considered successful if  
$
\varepsilon_k:={\|\mathcal{T}-\mathcal{C}_k* \mathcal{W}_k^{\dagger}* \mathcal{R}_k\|_{\fro}}/{\|\mathcal{T}\|_{\fro}}\leq 10^{-3}$.   
 
Our empirical phase transition results are presented in \Cref{emp}, with the first row showing a 3D view of the phase transition results and the second row the corresponding 2D view. White and black pixels in these visuals represent the success and failure of all the 25 tests, respectively. The results highlight that higher overall sampling rates are needed for successful completion with larger tubal-ranks \(r\). Importantly, tensor completion is achievable with sufficiently large overall sampling rates, regardless of the specific combinations of the horizontal, lateral slice sizes, and subtensor sampling rates, as demonstrated by the results in the 2D view.  This demonstrates ITCURTC's flexibility in sampling low-tubal-rank tensors for successful reconstruction. Additionally, we include our numerical results for the convergence of ITCURTC in Appendix~\ref{convergence_study}.
\subsection{Real-world applications}\label{more_experiments}
This section compares the t-CCS model with the Bernoulli sampling model in TC tasks across various data types to evaluate the t-CCS model's practical feasibility and applicability. We compare ITCURTC, based on the t-CCS model, with established TC methods based on the Bernoulli sampling model, including BCPF, TMac, TNN, and F-TNN.  The performance is measured by reconstruction quality and execution time. The quality is measured by the Peak Signal-to-Noise Ratio (PSNR) and the Structural Similarity Index (SSIM), where  
\begin{equation}\label{eqn:PSNR}
\text{PSNR} = 10 \log_{10}\left(\frac{n_{1}n_{2}n_{3}\|\mathcal{T}\|_{\infty}^2}{\|\mathcal{T}-\widetilde{\mathcal{T}}\|_{\fro}^2}\right).
\end{equation}
and SSIM evaluates the structural similarity between two images, as detailed in \citep{wang2004image}. Higher PSNR and SSIM scores indicate better reconstruction quality, and SSIM values are reported as the average across all frontal slices.

Our experimental process is as follows. We first generate random observations via the t-CCS model: uniformly randomly selecting concentrated horizontal (\(\mathcal{R}\)) and lateral (\(\mathcal{C}\)) subtensors, defined as \(\mathcal{R} = [\mathcal{T}]_{I,:,:}\) and \(\mathcal{C} = [\mathcal{T}]_{:,J,:}\), with  \( |I|=\delta n_{1} \text{ and }|J|=\delta n_2\); entries in \(\mathcal{R}\) and \(\mathcal{C}\) are sampled based on the Bernoulli sampling model with the locations of the  observed entries denoted by \( \Omega_{\mathcal{R}} \) and \( \Omega_{\mathcal{C}} \). The procedure of the t-CCS model results in a tensor that is only partially observed in  \(\mathcal{R}\) and \(\mathcal{C}\). ITCURTC is then applied to estimate the missing entries and thus  recover the original tensor. For comparison, we also generate observations of the entire original tensor \( \mathcal{T} \) using the Bernoulli sampling model  with a   probability  \( p_{\mathcal{T}} := \frac{|\Omega_{\mathcal{R}}\cup \Omega_{\mathcal{C}}|}{n_{1}n_{2}n_{3}} \).  Additionally, we estimate the missing data using several TC methods:  BCPF\footnote{\url{https://github.com/qbzhao/BCPF}}, which is based on the CP decomposition framework; TMac\footnote{\url{https://xu-yangyang.github.io/TMac/}}, which utilizes the Tucker decomposition framework; and TNN\footnote{\url{https://github.com/jamiezeminzhang/}} and F-TNN\footnote{\url{https://github.com/TaiXiangJiang/Framelet-TNN}}, which are both grounded in the t-SVD framework.  To ensure reliable results, we repeat this entire procedure 30 times, averaging the PSNR and SSIM scores and the runtime to minimize the effects of randomness. 
\subsubsection{\textnormal{Color image completion}}
Color images, viewed as 3D tensors with dimensions for height, width, and color channels, are effectively modeled as low-tubal-rank tensors \citep{Liu2013Tensor, Lu2020Tensor}. In our tests, we focus on two   images: `Building'\footnote{\url{https://pxhere.com/en/photo/57707}} (of size $2579\times 3887\times 3$) and `Window'\footnote{\url{https://pxhere.com/en/photo/1421981}} (of size $3009\times 4513\times 3$). We present averaged test results over various overall observation rates $(\alpha)$ in \Cref{tab:colorimage}, and visual comparisons at $\alpha=20\%$ in \Cref{visual}. 
\begin{figure}[!th]
\vspace{-2mm}
 \centering
 \begin{minipage}{0.16\linewidth} \centering \footnotesize{Ground truth} \end{minipage}
 \begin{minipage}{0.16\linewidth} \centering \footnotesize{BCPF} \end{minipage}
\begin{minipage}{0.16\linewidth} \centering \footnotesize{TMac} \end{minipage}
\begin{minipage}{0.16\linewidth} \centering \footnotesize{TNN}\end{minipage}
\begin{minipage}{0.16\linewidth} \centering \footnotesize{F-TNN}\end{minipage}
\begin{minipage}{0.16\linewidth} \centering \footnotesize{ITCURTC}\end{minipage}
\\
\vspace{0.7mm}
\includegraphics[width=0.16\linewidth]{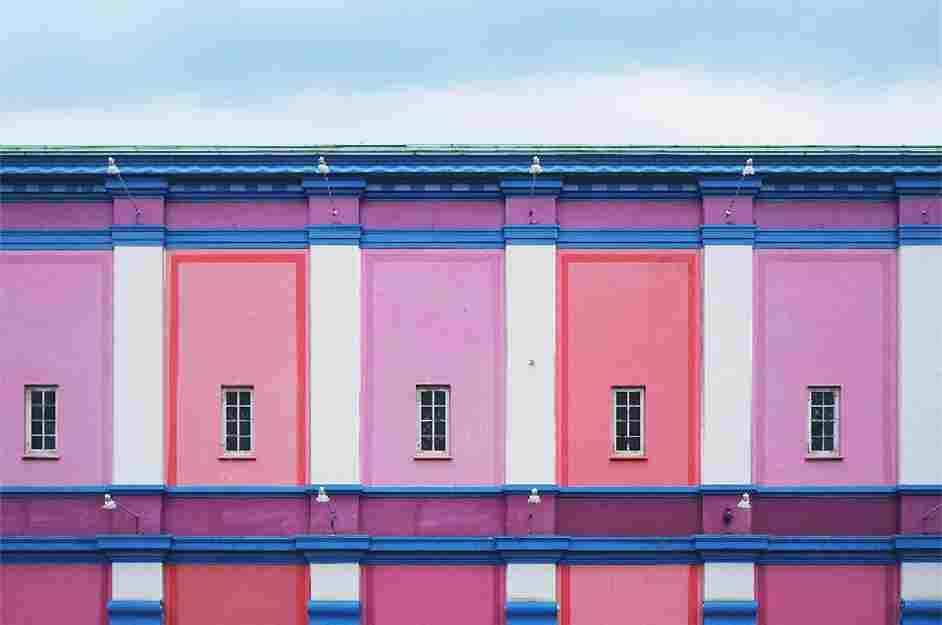}
\includegraphics[width=0.16\linewidth]{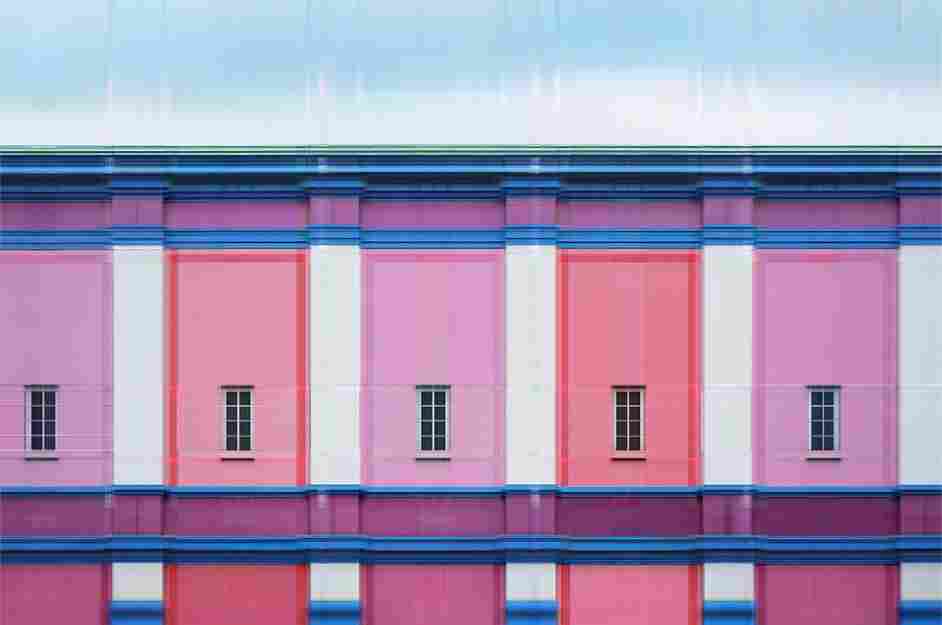}
\includegraphics[width=0.16\linewidth]{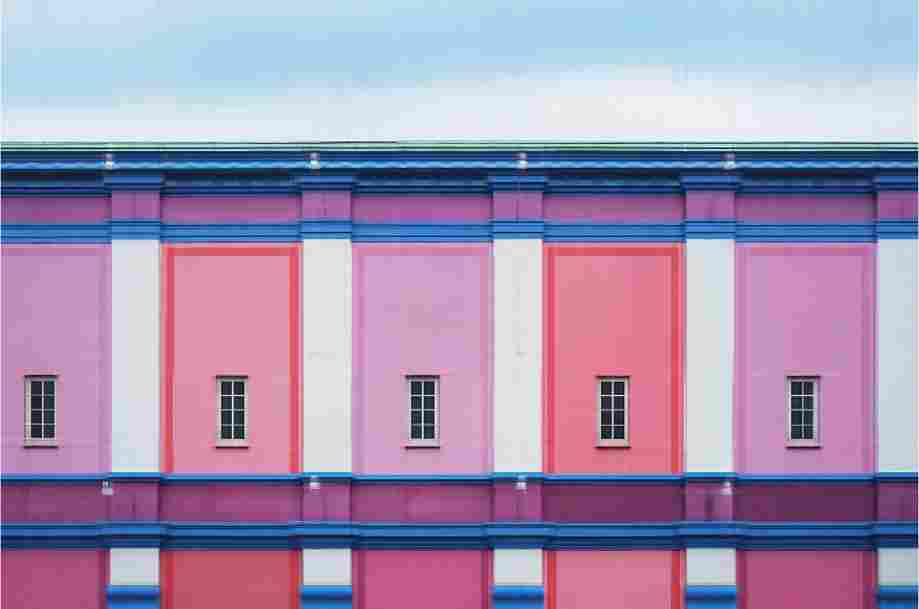}
\includegraphics[width=0.16\linewidth]{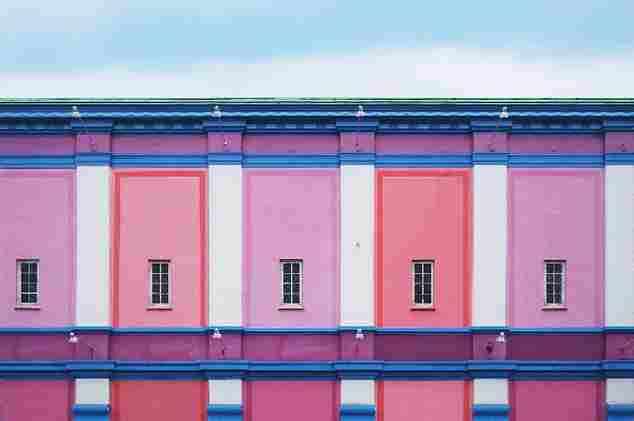}
\includegraphics[width=0.16\linewidth]{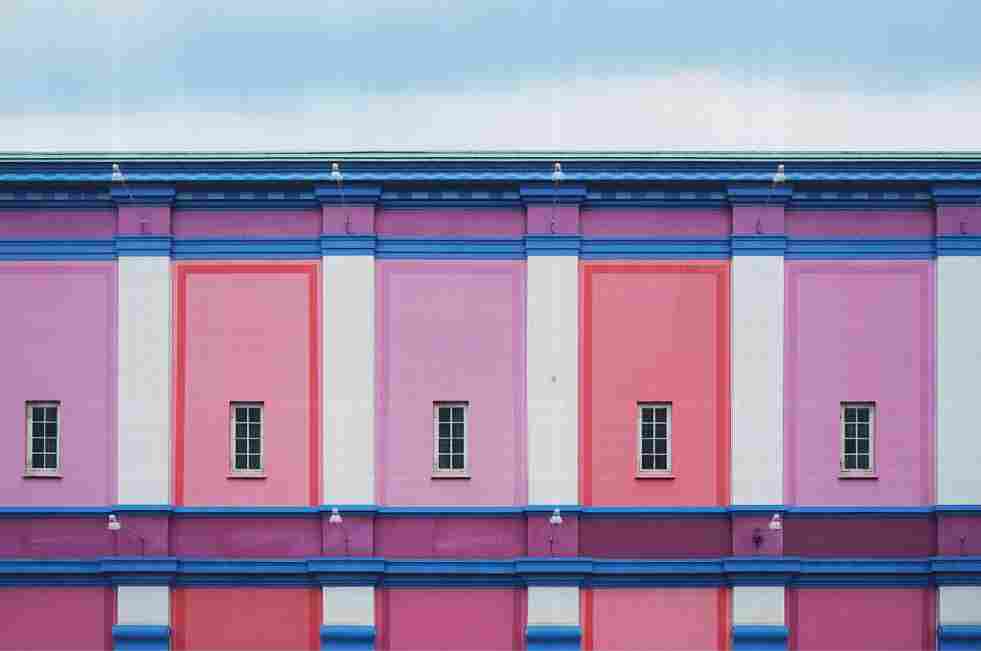}
\includegraphics[width=0.16\linewidth]{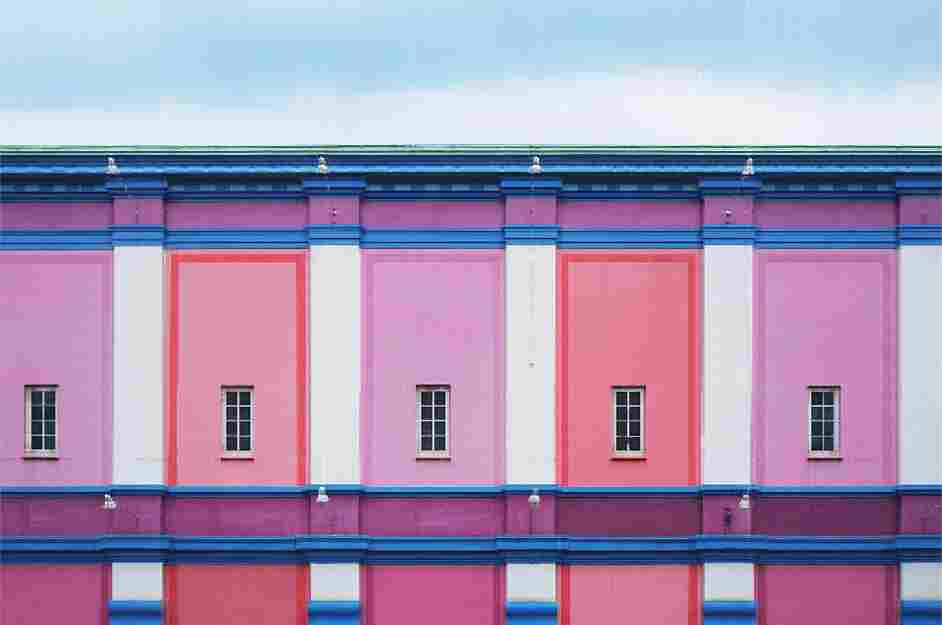}  
\\
\includegraphics[width=0.16\linewidth]{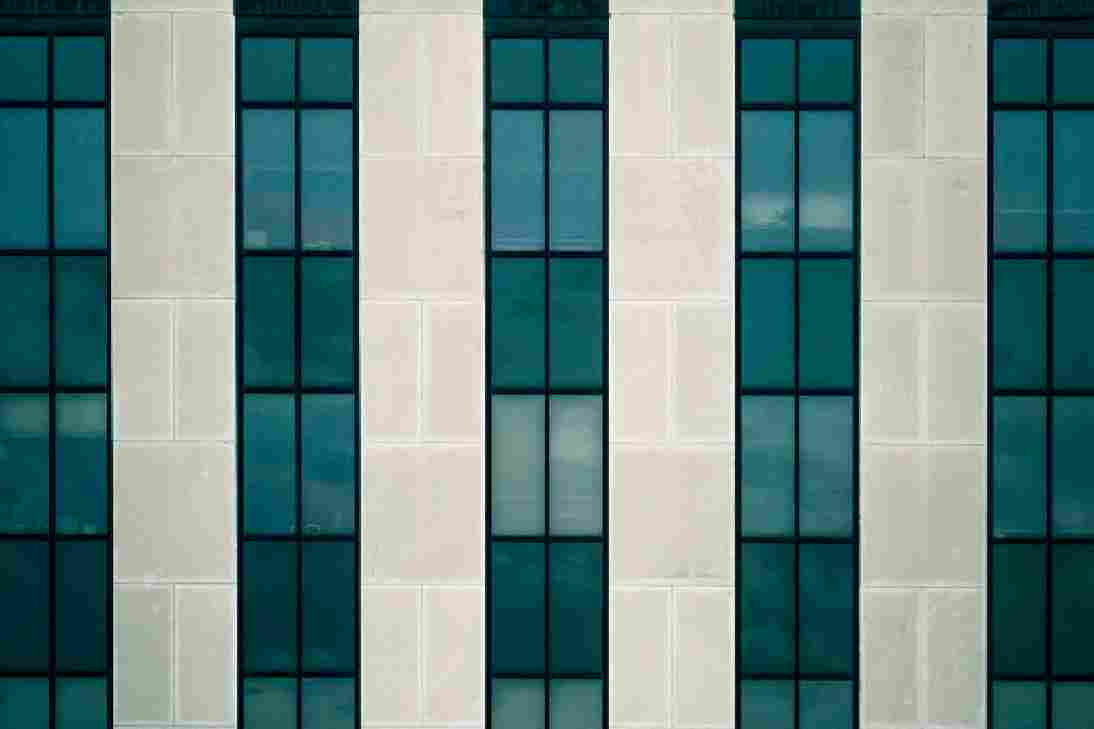}
\includegraphics[width=0.16\linewidth]{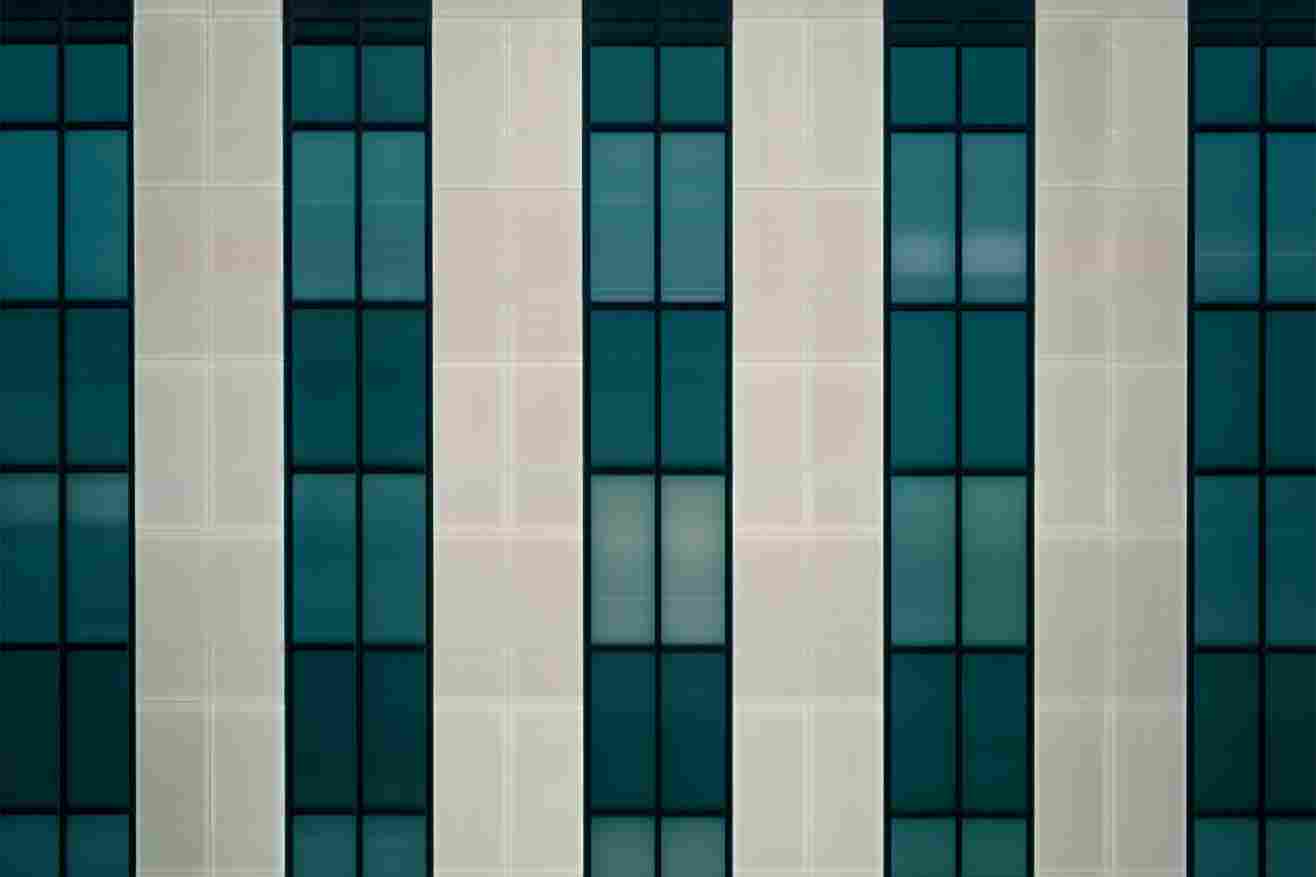}
\includegraphics[width=0.16\linewidth]{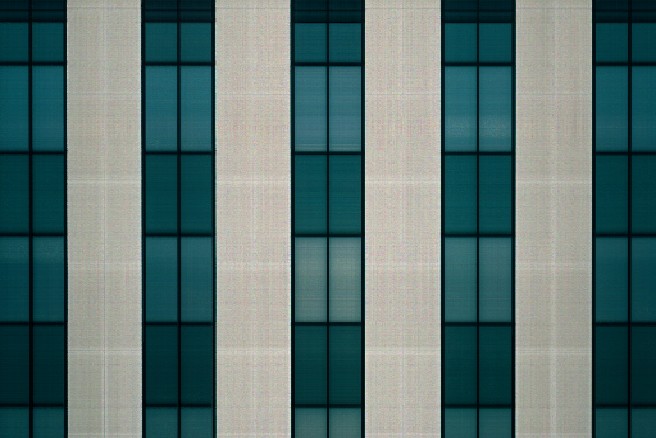}
\includegraphics[width=0.16\linewidth]{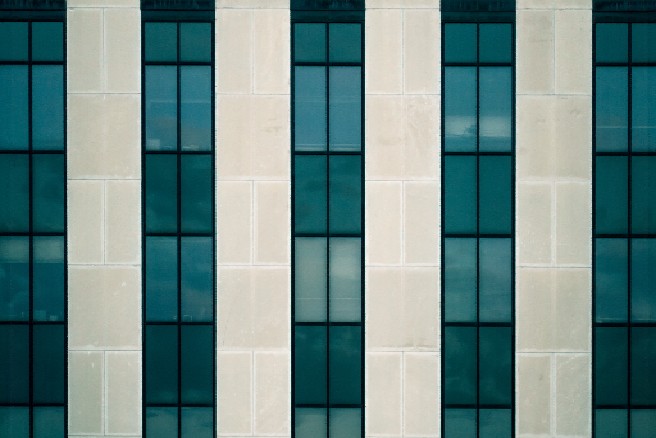}
\includegraphics[width=0.16\linewidth]{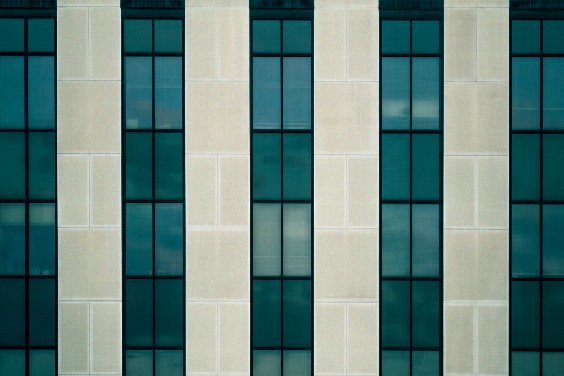}
\includegraphics[width=0.16\linewidth]{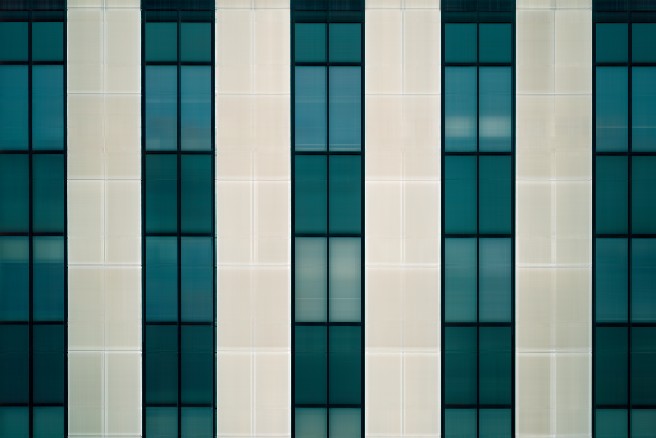}  
 \caption{\footnotesize{The visualization of color image inpainting for Building and Window datasets by setting tubal-rank $r= 35$ with the percentage of selected horizontal and lateral slices $\delta = 13\%$ and overall sampling rate $20\%$ for ITCURTC.  Other algorithms are applied based on Bernoulli sampling model with the same overall sampling rate $20\%$. Additionally, t-CCS samples on  the Building  for ITCURTC are the same as those in \Cref{failure_pics}}.  
 }\label{visual}
 \vspace{-5mm}
\end{figure}

\Cref{visual} presents a clear visual comparison of results from different methods at a $20\%$ overall sampling rate, where the algorithms BCPF, TMac, TNN, and F-TNN are applied on the Bernoulli sampling model and ITCURTC are applied on t-CCS model. The ground truth is the original image of the building and the window. BCPF underperforms in visual effects compared to other methods.  TNN shows slight variations from the ground truth, maintaining colors and details with minor discrepancies. TMac  reveals some notable differences. F-TNN improves reflection fidelity and color saturation, closely resembling the ground truth. ITCURTC also achieves the high similarity to the ground truth, accurately reproducing colors and details. Moreover, ITCURTC significantly outperforms BCPF, TMac, TNN, and F-TNN in the t-CCS-based color image inpainting task, as evidenced by the unsatisfactory results of BCPF, TMac, TNN, and F-TNN under the t-CCS model (see \Cref{failure_pics}).
\begin{table}[!th]
\centering
\caption{\footnotesize{Image inpainting results on the Building and Window datasets. The \textbf{best results} are emphasized in bold, while the \underline{second-best results} are underlined. ITCURTC-$\delta$ refers to the ITCURTC method with the percentages of selected horizontal and lateral slices set at a fixed rate of $\delta\%$. The t-CCS-based algorithm ITCURTC-$\delta\%$s are performed on t-CCS scheme while other Bernoulli-based algorithms are performed on Bernoulli sampling scheme.}}
\label{tab:colorimage}
\resizebox{0.9\linewidth}{!}{
\begin{tabular}{|cc|ccc|ccc|}
\toprule
\multicolumn{2}{|c|}{\textsc{Dataset}} & \multicolumn{3}{c|}{Building} & \multicolumn{3}{c|}{Window} \\ \midrule
\multicolumn{2}{|c|}{\textsc{Overall Observation Rate}} & 12\% & 16\% & 20\% & 12\% & 16\% & 20\% \\ \midrule
\multicolumn{1}{|c|}{\multirow{7}{*}{$\mathrm{PSNR}$}} & ITCURTC-11 & \textbf{28.9249} &         \textbf{31.0050}&          32.1645 & \underline{35.2830}&         36.1611&         37.0236\\
\multicolumn{1}{|c|}{} & ITCURTC-12 & \underline{28.5518}&          \underline{30.8055}&        31.9489&    35.1195&          36.1145&          37.0174\\
\multicolumn{1}{|c|}{} & ITCURTC-13 & 28.1893&          30.7260&        31.6825& 35.0196&          36.1215&          36.8885\\
\multicolumn{1}{|c|}{} & BCPF & 26.7939&          28.2949&          29.4298& 30.1611&        33.9990&         35.4780\\
\multicolumn{1}{|c|}{} & TMac &27.0425&          30.1755&          \underline{32.3632}& 33.2673 & \underline{36.6370} &  \textbf{37.5877}\\
\multicolumn{1}{|c|}{} & TNN &  26.3466&      30.3844 & 31.7512 & 31.8747&           34.6443&           36.7893\\

\multicolumn{1}{|c|}{} & F-TNN &  28.2529&        30.1521&\textbf{33.1660}& \textbf{35.6747}&        \textbf{36.9233}&         \underline{37.2618} \\
\midrule
\multicolumn{1}{|c|}{\multirow{7}{*}{$\mathrm{SSIM}$}} & ITCURTC-11 & 0.8310&         \textbf{0.8880}&         \textbf{0.9118} & \underline{0.8571}&0.8738&         0.8848\\
\multicolumn{1}{|c|}{} & ITCURTC-12 & 0.8172&        \underline{0.8818}&        0.9033 & 0.8535&        0.8733&         0.8850\\ 
\multicolumn{1}{|c|}{} & ITCURTC-13 & 0.8016&       0.8774&         0.8954 &0.8504&         0.8731&0.8837 \\
\multicolumn{1}{|c|}{} & BCPF & \textbf{0.8639}&0.8761&         0.8873&0.8269&         0.8554&        0.8727 \\
\multicolumn{1}{|c|}{} & TMac &  \underline{0.8402}   & 0.8586  &  \underline{0.9111} & 0.8200  &  \textbf{0.8928} &   \underline{0.9035} \\
\multicolumn{1}{|c|}{} & TNN & 0.6458&        0.8257&        0.8382 & 0.8333&         0.8564&        0.8804\\
\multicolumn{1}{|c|}{} & F-TNN &  0.7583&        0.8354&      0.8626& \textbf{0.8745}   &     \underline{0.8899} & \textbf{0.9066} \\
\midrule
\multicolumn{1}{|c|}{\multirow{7}{*}{\textsc{Runtime} (sec)}} & ITCURTC-11 & \underline{10.9354}&        \textbf{17.3187}&         \textbf{18.1098}& \textbf{23.8990}&          \textbf{24.1286}&                \textbf{25.1853} \\
\multicolumn{1}{|c|}{} & ITCURTC-12 & \textbf{10.7715}&          \underline{19.3517}&               \underline{19.7731} & \underline{25.5856}&          \underline{26.3275}&               \underline{28.0392} \\
\multicolumn{1}{|c|}{} & ITCURTC-13 & 12.2208&          21.2458&          22.0287 & 28.8653&                29.4986&          30.8361\\
\multicolumn{1}{|c|}{} & BCPF & 213.6800&          360.2903&              613.3072&  345.3425&          500.3060&         1629.8061
  \\ 
\multicolumn{1}{|c|}{} & TMac &92.9568&          104.8518&          108.6827& 233.8853 & 242.7499&  259.6068\\
\multicolumn{1}{|c|}{} & TNN & 3651.4556&          3289.5535&          3004.6557 & 5801.1631&            6572.9697&         6690.7945\\
\multicolumn{1}{|c|}{} & F-TNN & 2642.9409&          2692.6197&              2267.5622& 4739.2703      &        4134.5206   &           4105.0327\\ 
\bottomrule
\end{tabular}}
\vspace{-5mm}
\end{table}

\Cref{tab:colorimage} shows that ITCURTC for t-CCS typically offers quality comparable to that of Bernoulli sampling-based TC algorithms. In runtime efficiency, ITCURTC leveraging the t-CCS model significantly surpasses BCPF, TMac, TNN, and F-TNN, all of which are based on the Bernoulli sampling model. This efficiency enhancement highlights the t-CCS model's superior performance in practical applications. Additionally, ITCURTC's consistent performance in delivering similar quality results across different $\delta$, provided the same overall sampling rates. These highlight the flexibility and feasibility of the t-CCS model. 

\subsubsection{\textnormal{MRI reconstruction}}
In this study, we evaluate our model on an MRI heart dataset\footnote{\url{http://medicaldecathlon.com/dataaws}} (of size \(320\times 320 \times 110\)), where compact t-SVD with tubal-rank \(35\) yields less than \(10\%\) relative error, suggesting a low-tubal-rank property of the dataset. The visualization results of the reconstruction of the MRI  using different methods at a $30\%$ overall sampling rate are presented in \Cref{mridata}, and reconstruction quality and runtime are detailed in \Cref{psnr}. 

\begin{figure}[!th]
\vspace{-2mm}
\centering
\setlength{\tabcolsep}{0.5pt}
\small
\begin{tabular}{ccccccc}
\footnotesize{Ground truth}
&\footnotesize{BCPF} 
&\footnotesize{TMac} 
&\footnotesize{TNN} 
&\footnotesize{F-TNN}
&\footnotesize{ITCURTC}
\\
\includegraphics[width=0.16\linewidth]{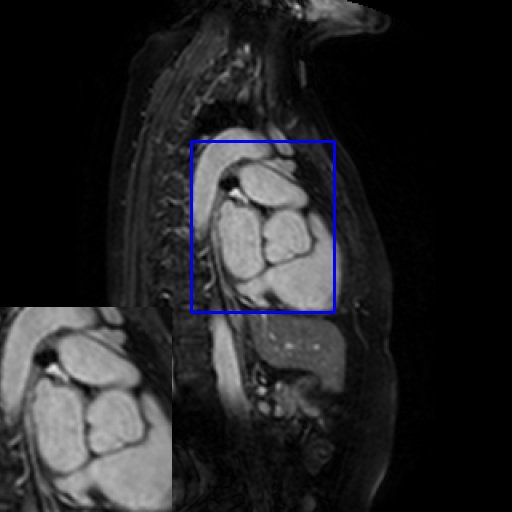}&
\includegraphics[width=0.16\linewidth]{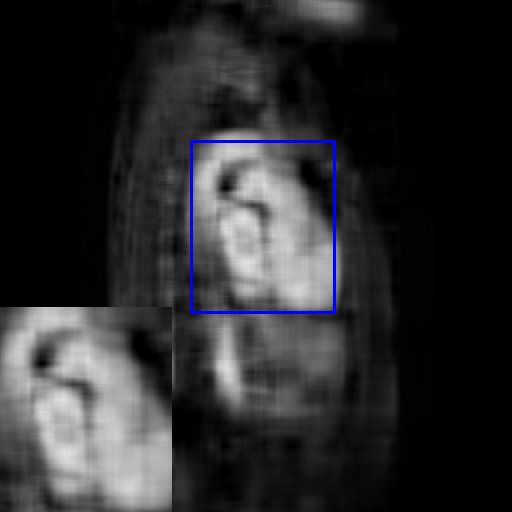}&
\includegraphics[width=0.16\linewidth]{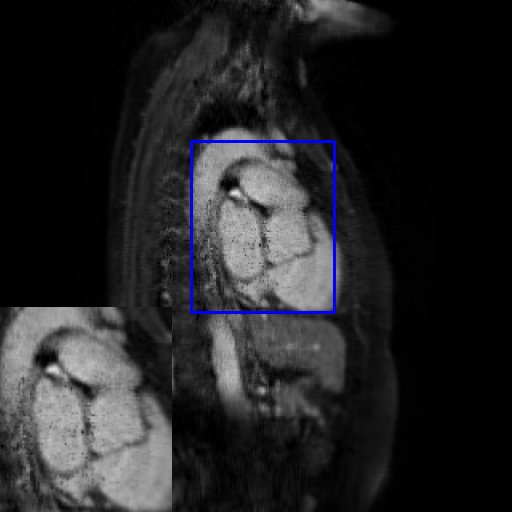}&
\includegraphics[width=0.16\linewidth]{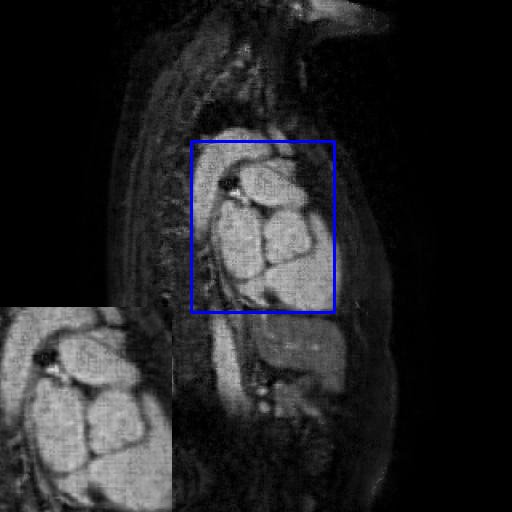}&
\includegraphics[width=0.16\linewidth]{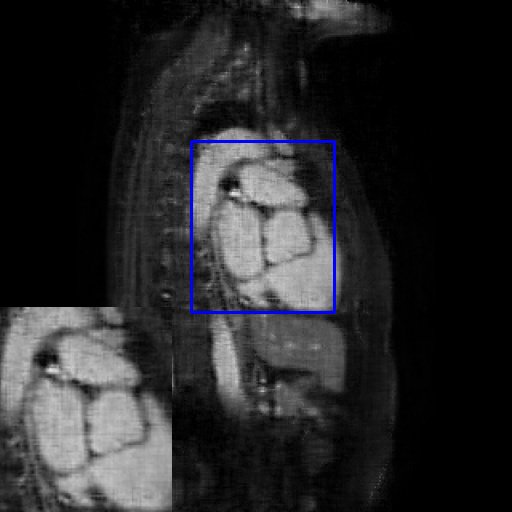}&
\includegraphics[width=0.16\linewidth]{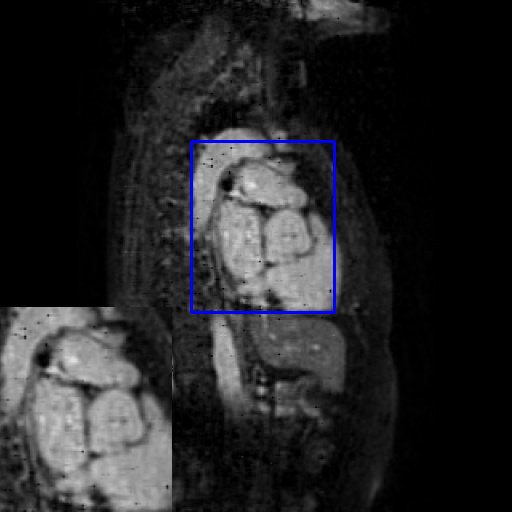}\\
\includegraphics[width=0.16\linewidth]{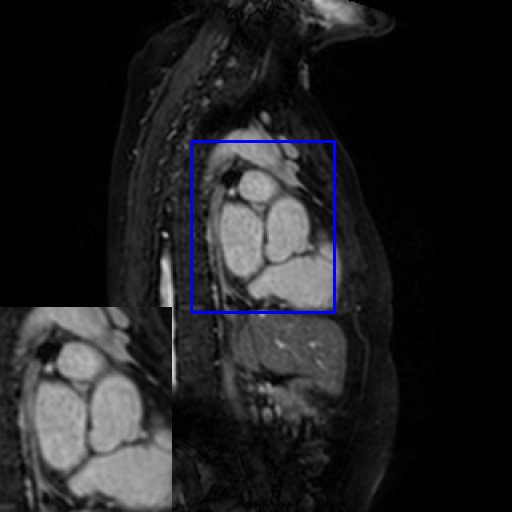}&
\includegraphics[width=0.16\linewidth]{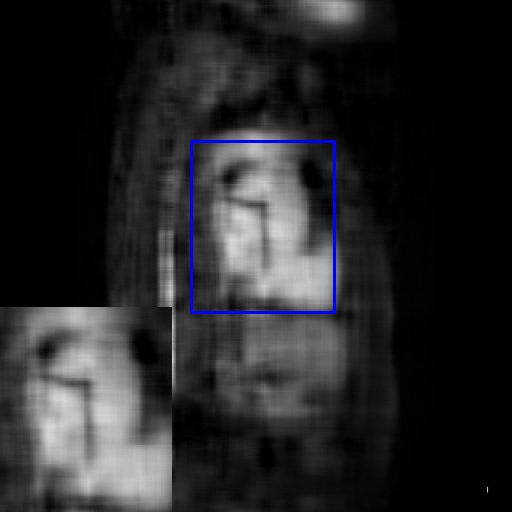}&
\includegraphics[width=0.16\linewidth]{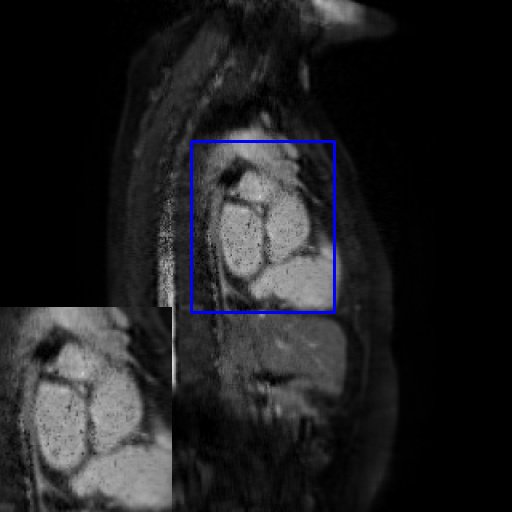}&
\includegraphics[width=0.16\linewidth]{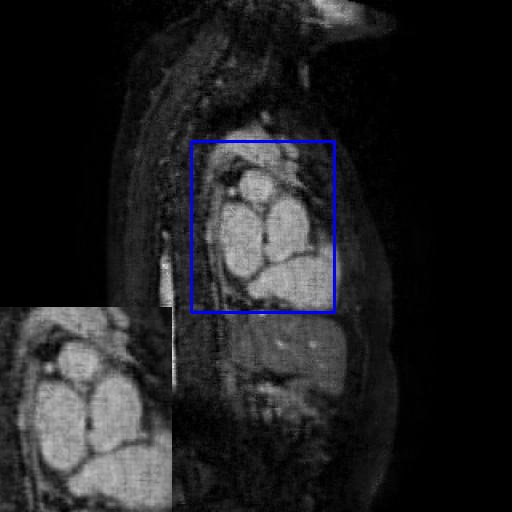}&
\includegraphics[width=0.16\linewidth]{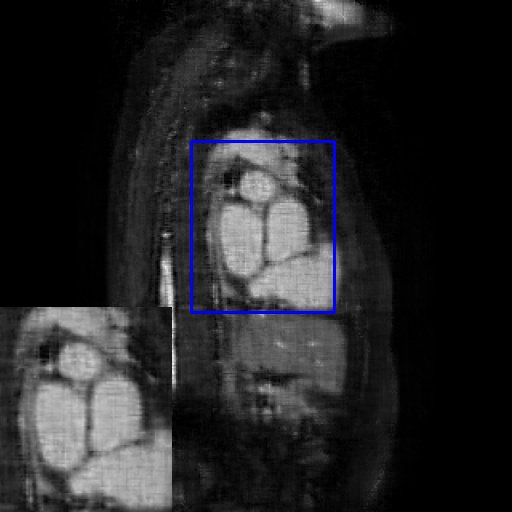}&
\includegraphics[width=0.16\linewidth]{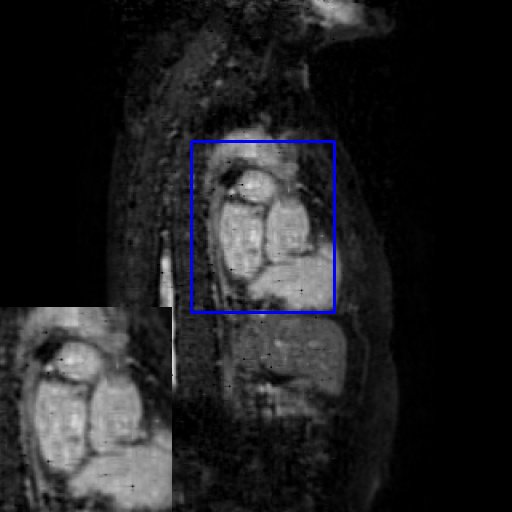}
\\   
\includegraphics[width=0.16\linewidth]{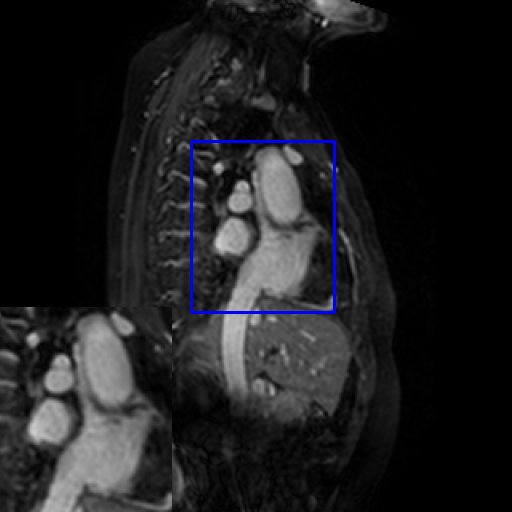}&
\includegraphics[width=0.16\linewidth]{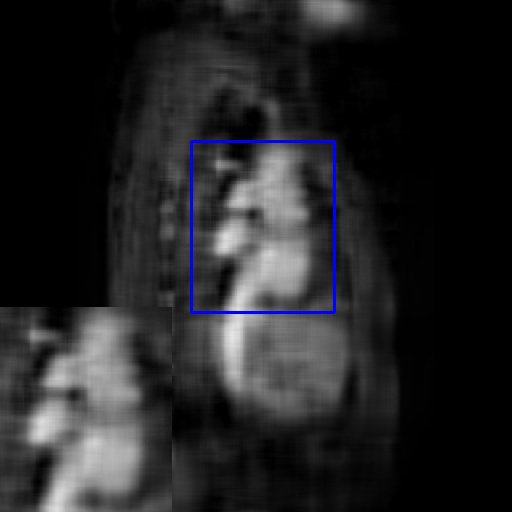}&
\includegraphics[width=0.16\linewidth]{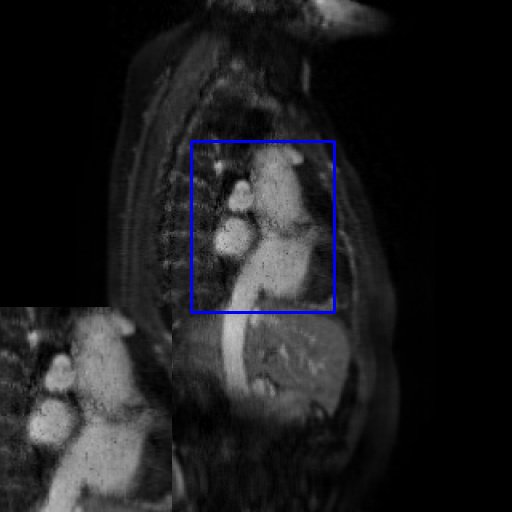}&
\includegraphics[width=0.16\linewidth]{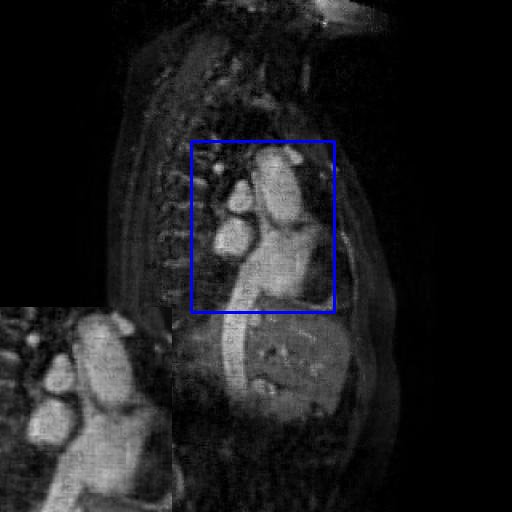}&
\includegraphics[width=0.16\linewidth]{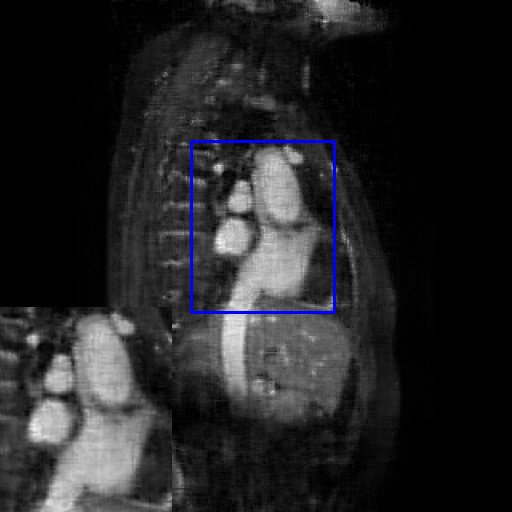}&
\includegraphics[width=0.16\linewidth]{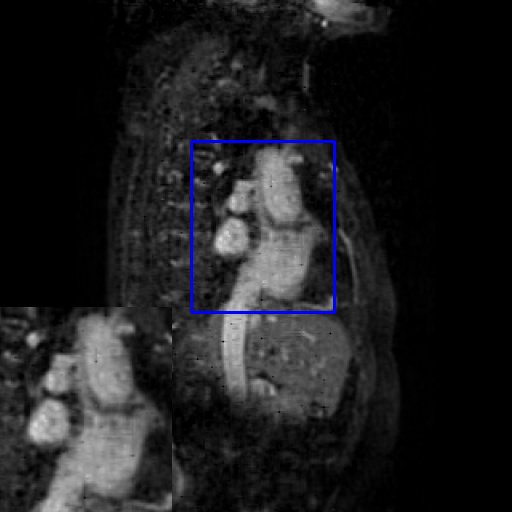}
\\  
\includegraphics[width=0.16\linewidth]{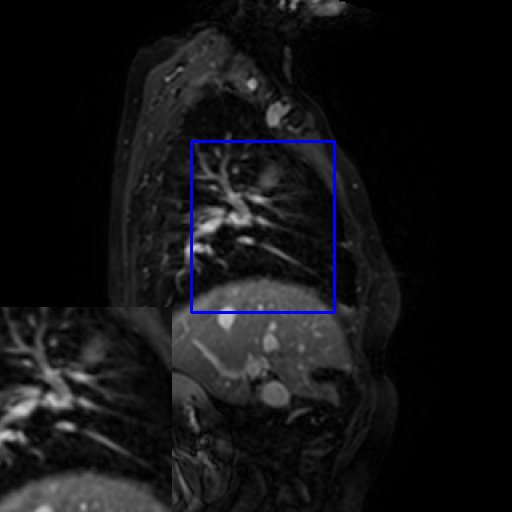}&
\includegraphics[width=0.16\linewidth]{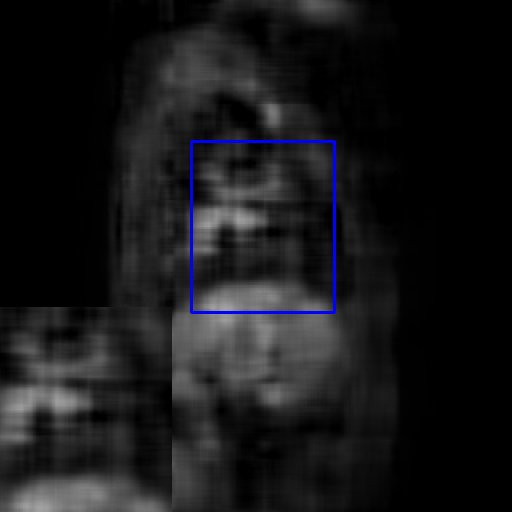}&
\includegraphics[width=0.16\linewidth]{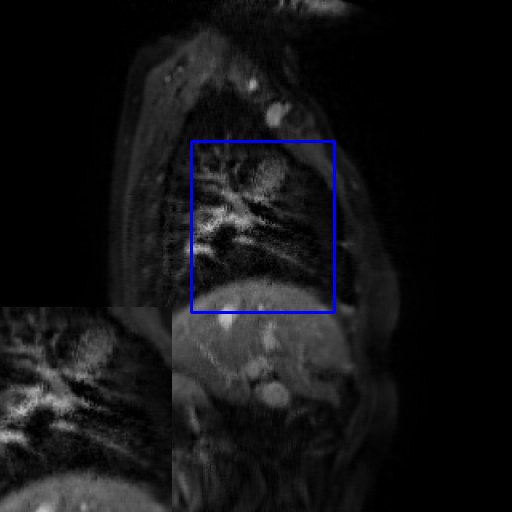}&
\includegraphics[width=0.16\linewidth]{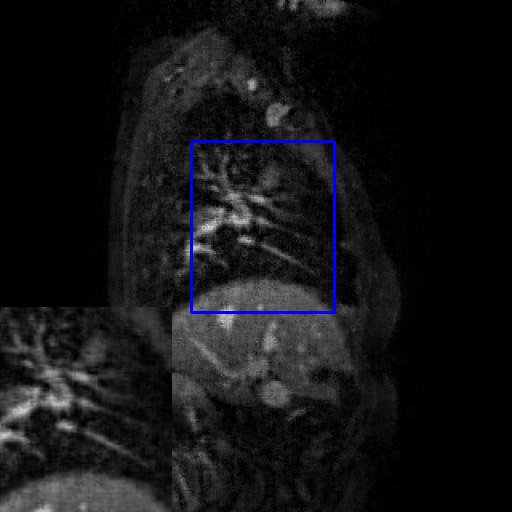}&
\includegraphics[width=0.16\linewidth]{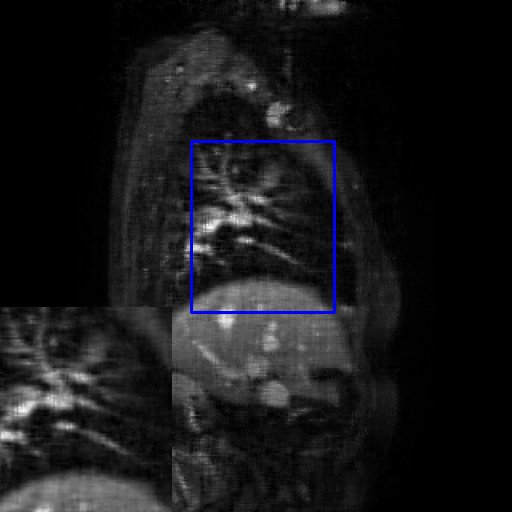}&
\includegraphics[width=0.16\linewidth]{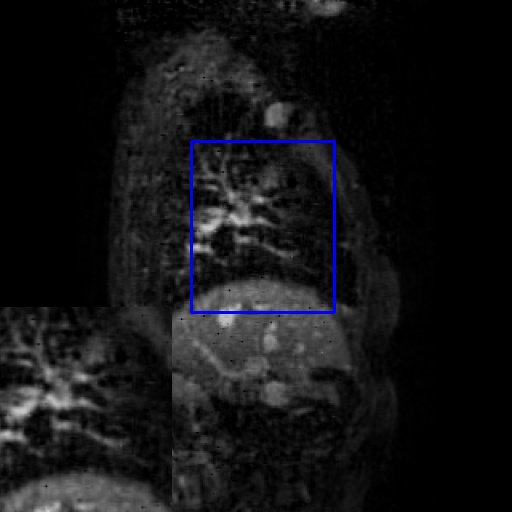} 
\end{tabular}
 \caption{\footnotesize{The visualizations of MRI data recovery are obtained by setting a tubal rank of $r = 35$ for ITCURTC with the percentage of selected lateral and horizontal slices $\delta = 27\%$ at an overall sampling rate of $30\%$. Other algorithms are applied under Bernoulli sampling models with the same overall sampling rate. Results for slices 51, 66, 86, and 106 are shown in rows 1 to 4. Each set includes a $1.3\times$ magnified area for clearer comparison, positioned at the bottom left of each result.}}\label{mridata}
 \vspace{-8mm}
\end{figure}
\Cref{mridata} shows recovery results for four frontal MRI slices using BCPF, TMac, TNN, and F-TNN, all under Bernoulli sampling model, and ITCURTC  under the t-CCS model. The ground truth serves as the actual dataset, from which missing values are to be predicted by different algorithms. BCPF shows notable artifacts and lacks the sharp edges of the heart's interior structures. TMac improves over BCPF but still presents a softer representation of cardiac anatomy. TNN enhances the detail prediction, resulting in a more accurate completion of the tensor that begins to resemble the reference more closely. F-TNN maintains improvements in detail prediction, and edges within the cardiac structure suggest a refined approach to tensor completion. ITCURTC shows a reconstruction where the cardiac structures are clearly defined, reflecting the structure present in the ground truth without implying superiority, but rather indicating effectiveness in predicting the missing values.  The highlighted regions of interest (ROIs), marked in blue, allow for a detailed comparison across the methods. In these ROIs, though ITCURTC's reconstructions may not provide the most visually appealing results, they demonstrate efficiency in preserving structural integrity and texture, which are crucial aspects for clinical applications. 

\Cref{psnr} effectively demonstrates the flexibility and feasibility of the t-CCS model, showing that the reconstruction performance of the t-CCS-based method ITCURTC generally aligns with or matches the reconstruction quality of Bernoulli-sampling-based TC methods. Furthermore, in terms of runtime, ITCURTC, implemented under the t-CCS model, demonstrates marked superiority by significantly outperforming alternatives such as BCPF, TMac, TNN, and F-TNN, all of which are applied under the Bernoulli sampling scheme. This notable advantage distinctly underscores the enhanced effectiveness of the t-CCS model in practical applications.
\begin{table}[!th]
\vspace{-5mm}
\centering
\caption{\footnotesize{The quantitative results for MRI data completion are presented, with \textbf{the best results} in bold and \underline{the second-best} underlined. ITCURTC-$\delta$ represents the ITCURTC method specifying that the selected proportion of horizontal and lateral slices is exactly $\delta\%$. The t-CCS-based algorithm ITCURTC-$\delta$ is performed using the t-CCS scheme, while other Bernoulli-based algorithms are performed using the Bernoulli sampling scheme.}}
\label{psnr}
\resizebox{0.9\linewidth}{!}{
\begin{tabular}{|c|c|cccccc|}
\toprule
\multicolumn{2}{|c|}{\textsc{Overall Observation Rate}} & 10\% & 15\% & 20\% & 25\% & 30\% & \\
\midrule
\multicolumn{1}{|c|}{\multirow{5}{*}{$\mathrm{PSNR}$}} & ITCURTC-23 & 22.4004&           24.3553&          26.9104&          \underline{29.1861}&          30.3911&\\
& ITCURTC-25 &22.2548&          24.0435&          26.9940&          29.0219&         31.1752& \\
& ITCURTC-27 & 22.1617&         23.9311&         27.0871&          29.0699&          \underline{31.2539}& \\
& BCPF & 22.6581   &24.5373 &  25.1663&   25.8111  & 26.2042 & \\
& TMac &\underline{22.8690}&          \underline{25.4225}&         \underline{27.7802}&           29.1526&          31.1648& \\
& TNN & \textbf{23.4779}  & 25.3480& \textbf{27.9423}&   28.4522  & 30.5580
 &\\
& F-TNN & 21.8172 &  \textbf{25.7453}  & 27.1969&   \textbf{29.3630}  & \textbf{31.3651} & \\
\midrule
\multicolumn{1}{|c|}{\multirow{5}{*}{$\mathrm{SSIM}$}} & ITCURTC-23 & 0.6020&         0.6821&         0.7584&        0.8160&        0.8451&\\
& ITCURTC-25 &  0.5990&         0.6769&         0.7571&        0.8084&       0.8619& \\
& ITCURTC-27 &  0.5990&          0.6751&        0.7567&        0.8086&         0.8600&\\
& BCPF & \textbf{0.6817} &   0.7151   & 0.7192 &   0.7301  &  0.7367& \\
& TMac & \underline{0.6804}&         0.7323&        \underline{0.7873}&        \underline{0.8227}&       \textbf{0.8924}& \\
& TNN & 0.6304 &   \underline{0.7494}   & 0.7677&    0.7984  &  0.8793& \\
& F-TNN & 0.6442 &   \textbf{0.7507} &  \textbf{0.8181} &   \textbf{0.8562} &   \underline{0.8871}& \\
\midrule
\multicolumn{1}{|c|}{\multirow{5}{*}{$\mathrm{Runtime}$}} & ITCURTC-23 & \underline{5.7908}  &  \textbf{7.0230}&    \textbf{4.8030}  &  \textbf{5.3058}  &  \textbf{5.9484}&\\
& ITCURTC-25 & \textbf{5.4241}&                \underline{8.3488}&          \underline{5.5303}&                \underline{5.9375}&          \underline{6.7111}& \\
& ITCURTC-27 & 5.8075&          8.8408&          6.0685&          6.4371&          7.4916&
 \\
& BCPF & 53.1651  & 88.2777  &111.1949&  180.6596 & 279.2789& \\
& TMac &30.4813&                28.0944&          28.6216& 28.9400&          30.0219& \\
& TNN & 87.7591  & 84.0952  & 56.9761&   57.9823  & 58.2098&\\
& F-TNN & 91.2048&   86.3112 &  84.0228&   82.2064  & 81.2119& \\
\bottomrule
\end{tabular}
}
\end{table}
\subsubsection{\textnormal{Seismic data reconstruction}}
Geophysical 3D seismic data is often modeled as a tensor with inline, crossline, and depth dimensions. In our analysis, we focus on a seismic dataset\footnote{\url{https://terranubis.com/datainfo/F3-Demo-2020}} of size $51 \times 191 \times 146$, where compact t-SVD with tubal-rank \(3\) yields less than \(5\%\) error, suggesting a low-tubal-rank property of the dataset. The corresponding results are detailed in \Cref{semesicdata} and \Cref{psnr2}.
\begin{figure}[!th]
\centering
\setlength{\tabcolsep}{0.5pt}
\small
\begin{tabular}{cccccc}
\footnotesize{Ground truth}
&\footnotesize{BCPF} 
&\footnotesize{TMac} 
&\footnotesize{TNN} 
&\footnotesize{F-TNN}
&\footnotesize{ITCURTC}\\
\includegraphics[width=0.16\linewidth]{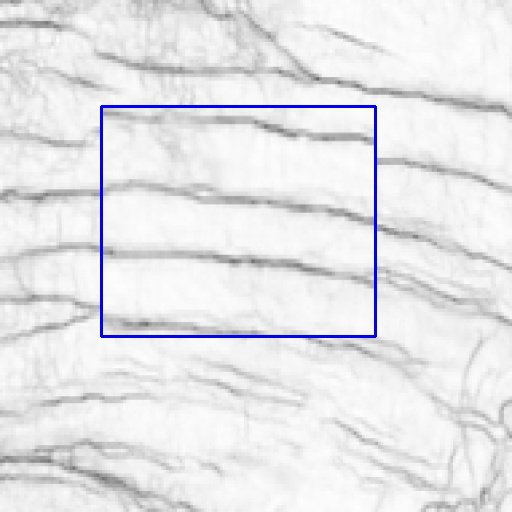}&
\includegraphics[width=0.16\linewidth]{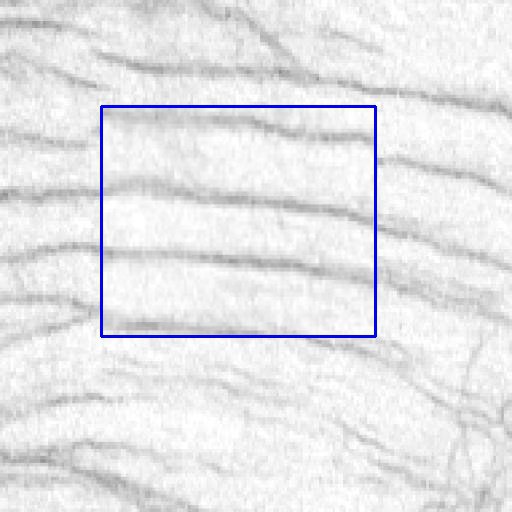}&
\includegraphics[width=0.16\linewidth]{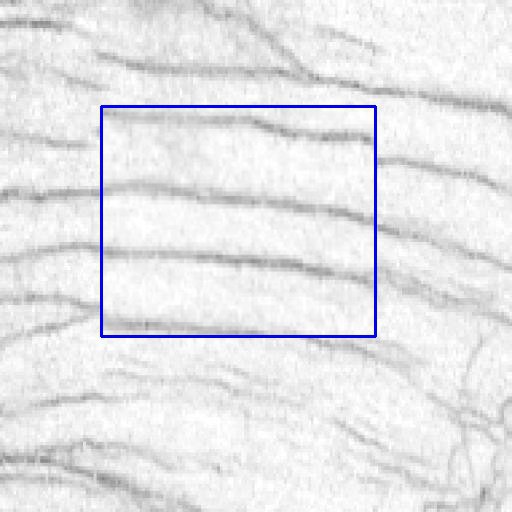}&
\includegraphics[width=0.16\linewidth]{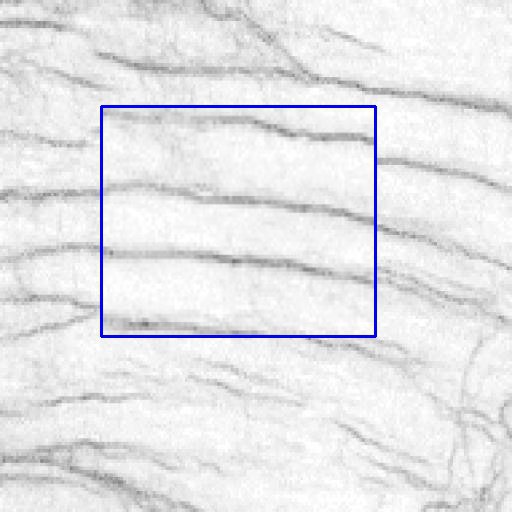}&
\includegraphics[width=0.16\linewidth]{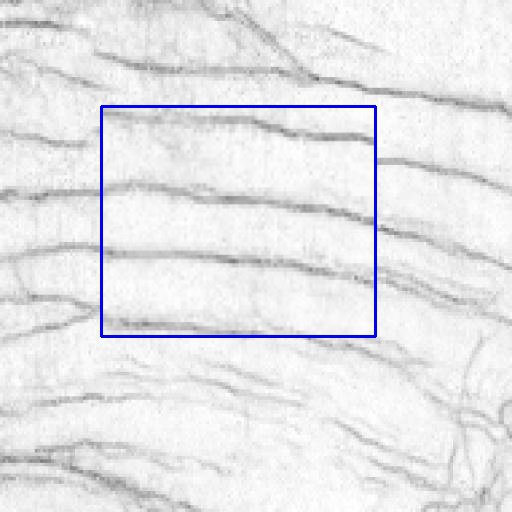}&
\includegraphics[width=0.16\linewidth]{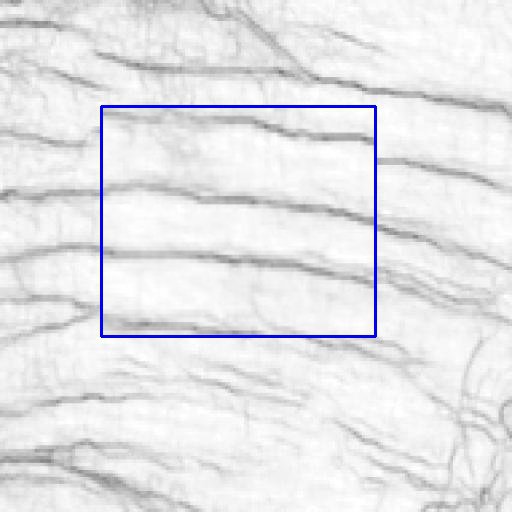}
\\   
\includegraphics[width=0.16\linewidth]{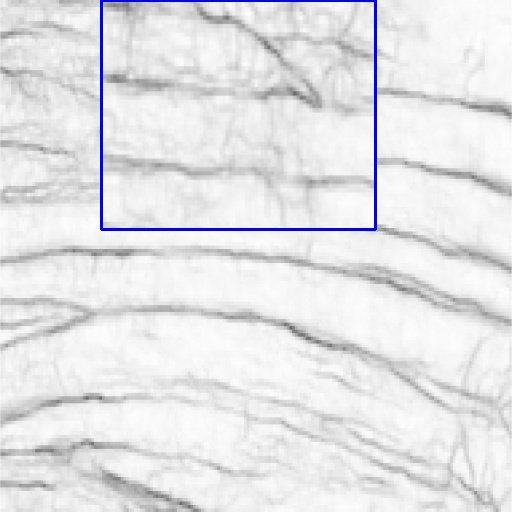}&
\includegraphics[width=0.16\linewidth]{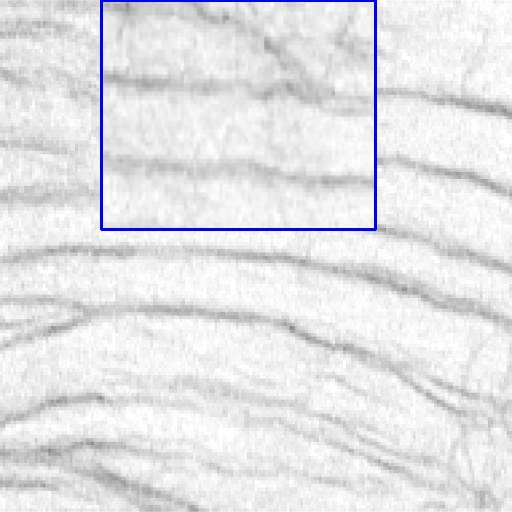}&
\includegraphics[width=0.16\linewidth]{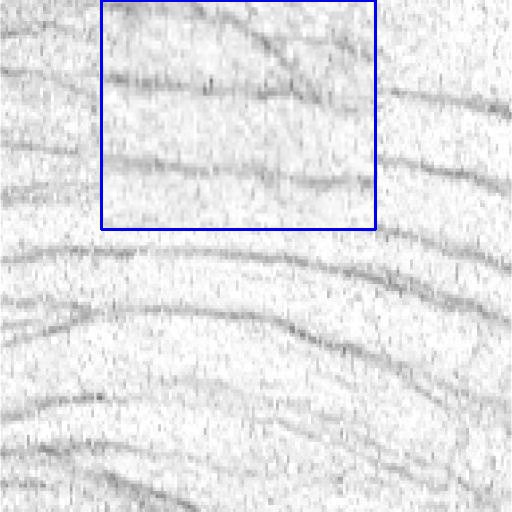}&
\includegraphics[width=0.16\linewidth]{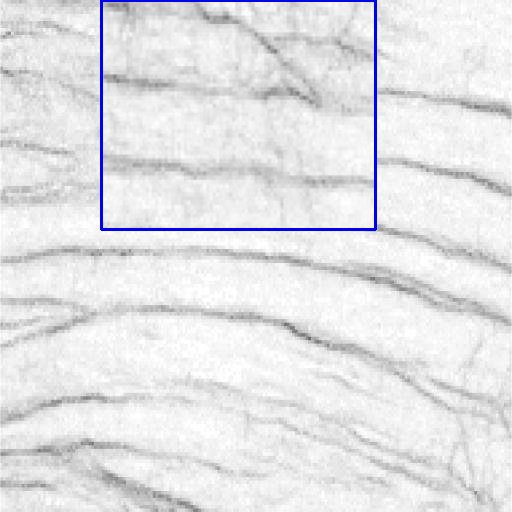}&
\includegraphics[width=0.16\linewidth]{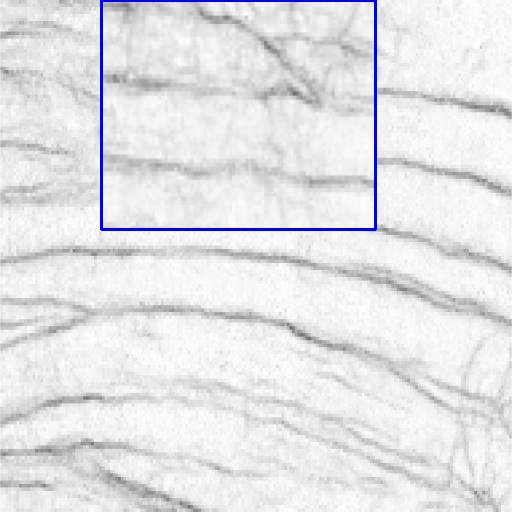}&
\includegraphics[width=0.16\linewidth]{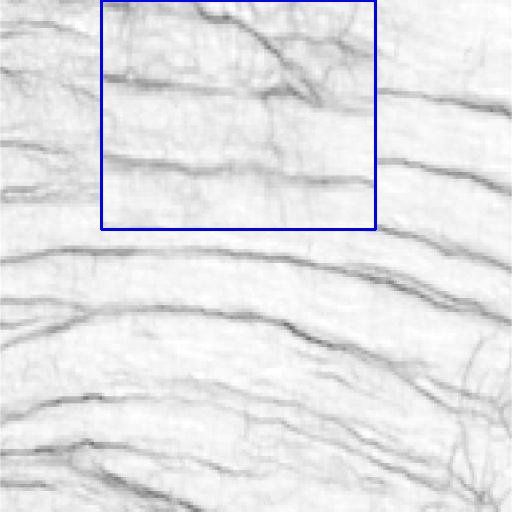}
\\  
\includegraphics[width=0.16\linewidth]{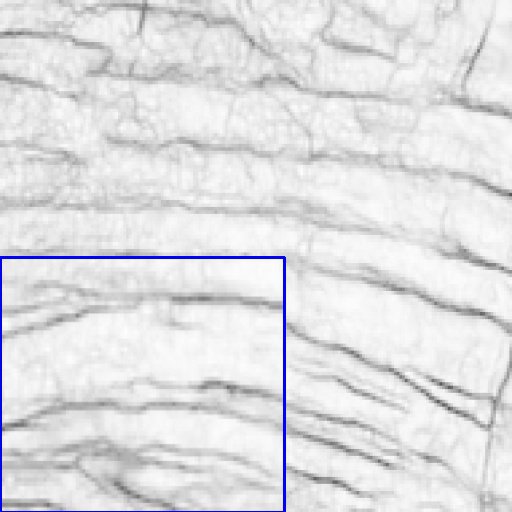}&
\includegraphics[width=0.16\linewidth]{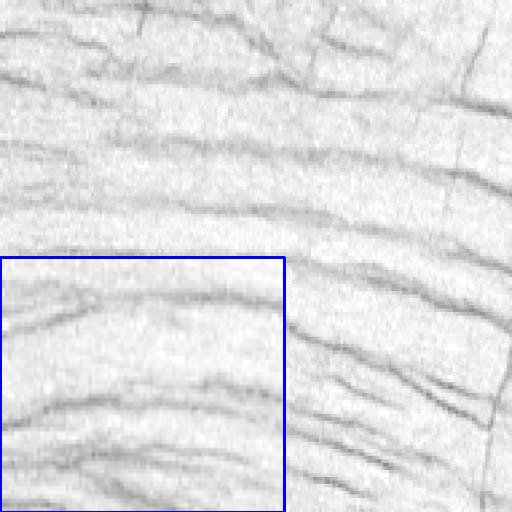}&
\includegraphics[width=0.16\linewidth]{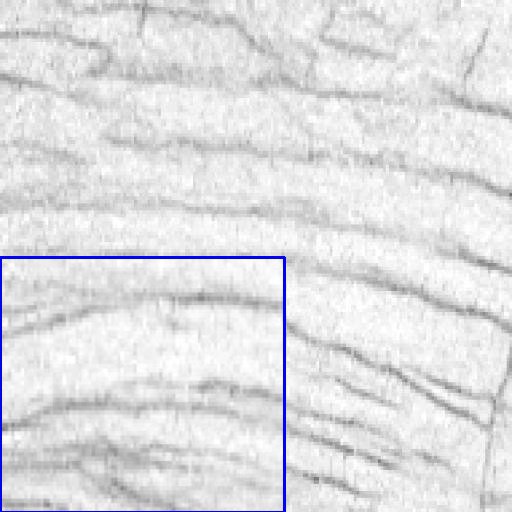}&
\includegraphics[width=0.16\linewidth]{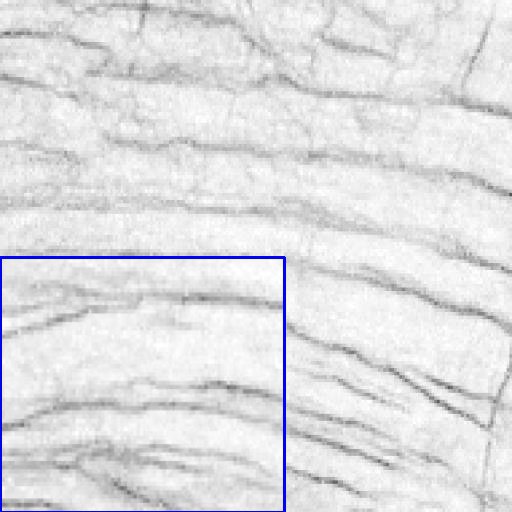}&
\includegraphics[width=0.16\linewidth]{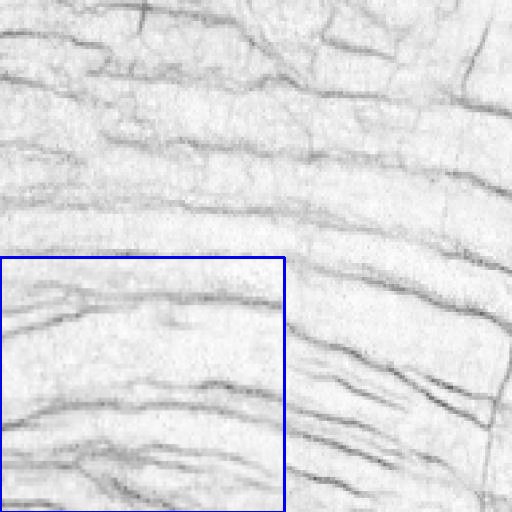}&
\includegraphics[width=0.16\linewidth]{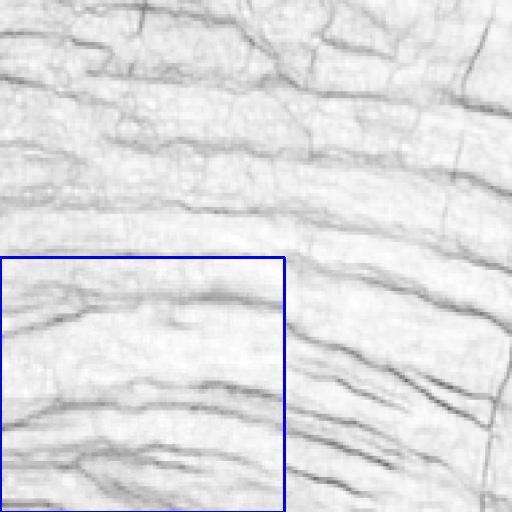} 
\\  
\end{tabular}
\caption{\footnotesize{Visualization of seismic data recovery results by setting tubal-rank $r = 3$ for ITCURTC with a percentage of the selected horizontal and lateral slices $\delta = 17\%$ and an overall sampling rate  $\alpha=28\%$, while other methods are applied based on  Bernoulli sampling models with the same overall sampling rate $28\%$. Displayed are slices $15$, $25$, and $35$ from top to bottom, with a $1.2\times$ magnified area in each set for a clearer comparison.}}\label{semesicdata}
\vspace{-8mm}
\end{figure}

\Cref{semesicdata} presents the comparative analysis of seismic completion algorithms: BCPF, TMac, TNN, and F-TNN, applied based on the Bernoulli sampling model, in contrast to ITCURTC, which is applied based on the t-CCS model. The ground truth serves as the definitive reference, with its stark textural definition. BCPF falls short of delivering optimal fidelity, with finer details lost in translation. TMac is commendable for preserving the texture's integrity, providing a cohesive image. TNN improves upon this, sharpening textural nuances and closing in on the ground truth's visual quality. F-TNN excels visually, capturing essential texture information effectively, a significant advantage when the emphasis is on recognizing general features. ITCURTC demonstrates comparable visual results, though less effective than other methods in terms of PSNR and SSIM.

\Cref{psnr2} shows that the t-CCS-based method, ITCURTC, achieves the fastest processing speeds while preserving satisfactory levels of PSNR and SSIM. This underscores the suitability of the t-CCS model for applications where rapid processing is essential without significant loss in visual accuracy. Furthermore, the consistent performance of ITCURTC across various subtensor sizes and sampling rates further emphasizes the flexibility and feasibility of the t-CCS model in diverse operational environments.
\begin{table}[!th]
\vspace{-5mm}
\centering
\caption{\footnotesize{Quantitative results for seismic data completion: TMac, TNN, F-TNN with Bernoulli sampling, and our method with t-CCS. \textbf{Best results} are in bold, and \underline{second-best} are underlined. 
 ITCURTC-$\delta$ refers to the ITCURTC method with the percentages of selected horizontal and lateral slices  set at a fixed rate of $\delta\%$. The t-CCS based algorithm ITCURTC-$\delta\%$s are performed on t-CCS scheme while other Bernoulli based algorithms are performed on Bernoulli Sampling scheme.}}\label{psnr2}
\resizebox{0.9\linewidth}{!}{
\begin{tabular}{|c|c|cccccc|}
\toprule
\multicolumn{2}{|c|}{\textsc{Overall Observation Rate}} & \multicolumn{1}{c}{12 \%} & \multicolumn{1}{c}{16 \%} & \multicolumn{1}{c}{20 \%} & \multicolumn{1}{c}{24 \%} & \multicolumn{1}{c}{28 \%} & \\
\midrule
\multicolumn{1}{|c|}{\multirow{7}{*}{$\mathrm{PSNR}$}} & ITCURTC-15 & 24.8020&  26.4092 &  27.4143 &  29.3053  & 30.6585 & \\
& ITCURTC-16 & 24.7386 &  26.1054 &  27.4737 &  29.3542  & 30.6905& \\
& ITCURTC-17 & \underline{24.8381}  & 26.1176  & 27.4768 &  28.8953 &  30.5312& \\ 
& BCPF & 24.0733  & 24.1905  & 24.2084&   24.2454  & 24.3015& \\ 
& TMac & \textbf{24.8859}  & \underline{26.5349} &  26.9970&   28.4662  & 30.7237& \\
& TNN & 23.7395  & 26.3806&   \underline{27.7428}  & \underline{29.5430}  & \underline{30.9172}& \\
& F-TNN & 24.0688 &  \textbf{27.5890}&   \textbf{28.6408} &  \textbf{29.7987} &  \textbf{31.2791} & \\
\cline{1-8}
\multicolumn{1}{|c|}{\multirow{6}{*}{$\mathrm{SSIM}$}} & ITCURTC-15 &\underline{0.5732} &   0.6691   & 0.7338  &  \underline{0.8143}  &  0.8596& \\
& ITCURTC-16 & 0.5691 &   0.6507  &  0.7349  &  0.8129   & \underline{0.8610}
& \\
& ITCURTC-17 & 0.5724  &  0.6491    &0.7321  &  0.7939   & 0.8523& \\
& BCPF & 0.5304  &  0.5407  &  0.5420 &   0.5494   & 0.5532& \\
& TMac & 0.5566  & \underline{0.6738}  &  0.6962 &   0.7612  &  0.8504 & \\
& TNN & 0.5165  &  0.6442 &   \underline{0.7577}   & 0.8080   & 0.8486& \\
& F-TNN & \textbf{0.6607}  &  \textbf{0.7551} &   \textbf{0.8142}  & \textbf{0.8479}  &  \textbf{0.8814}& \\
\cline{1-8}
\multicolumn{1}{|c|}{\multirow{6}{*}{$\mathrm{Runtime}$}} & ITCURTC-15 & \underline{6.4327} &   6.7701  &  \textbf{6.2598} & \textbf{6.8633}   & \textbf{6.8212}& \\
& ITCURTC-16 & \textbf{6.3825} & \textbf{6.3579}  &  \underline{6.7522}&   \underline{ 7.0789}  & \underline{ 7.0215} &\\
& ITCURTC-17 &7.0379 &   \underline{6.6306}& 6.8325 &   7.1480&7.3253& \\
&BCPF & 33.5759 &  33.1832 &  32.1258 &  31.7875  & 31.2663& \\
& TMac &  16.6135&   14.3412 &  16.8581 &  13.7124  & 13.1142&\\
& TNN &34.5718  & 31.3138&   29.2464 &  26.1727  & 23.9876& \\
& F-TNN & 22.1019  & 21.4482 &  22.1420&   17.8848  & 18.0547& \\
\bottomrule
\end{tabular}
}
\vspace{-5mm}
\end{table}

\subsection{Discussions on the results of real-world datasets}
 From the above results,  it is evident that our method surpasses others in runtime   with significantly lower computational costs. Consider a tensor of size \(n_1 \times n_2 \times n_3\). When a framelet transform matrix is constructed using \(n\) filters and \(l\) levels, the computational cost per iteration for framelet-based Tensor Nuclear Norm (F-TNN) is $\mathcal{O}((nl-l+1)n_1n_2n_3(n_3 + \min(n_1, n_2)))$. This formulation incorporates the processes involved in generating a framelet transform matrix, as elaborated in seminal works such as  \cite{cai2008framelet} and   \cite{jiang2018matrix}. While enhancing the number of levels and filters in F-TNN can improve the quality of results, it also escalates the computational burden, particularly for tensors of substantial size. In our experiments, we set both the framelet level and the number of filters to $1$ for the F-TNN implementation.
For comparison, the computational cost per iteration for TNN is
\(
\mathcal{O}(\min(n_1, n_2)n_1n_2n_3 + n_1n_2n_3\log(n_3)),
\)
and for the TMac, it is
\(
\mathcal{O}((r_1+r_2+r_3)n_{1}n_{2}n_{3})
\)
where \((r_1, r_2, r_3)\) denotes the Tucker rank. For BCPF, it is 
\(
\mathcal{O}(R^3(n_{1}n_{2}n_{3})+R^2(n_{1}n_{2}+n_{2}n_{3}+n_{3}n_{1})),
\)
where $R$ is the CP rank. In contrast, the computational expense per iteration of our proposed method is significantly reduced to
\(
\mathcal{O}(r|I|n_2n_3 + r|J|n_1n_3),
\)
assuming \(|I| \ll n_1\) and \(|J| \ll n_2\), indicating a substantial efficiency improvement over traditional methods.
  
Note that for F-TNN,  \cite{jiang2020framelet} formulated the tensor nuclear norm utilizing the $M$-product \citep{kilmer2019tensor}, a generalization of the t-product for 3-order tensors. In \citep{jiang2020framelet}, they  incorporated a tight wavelet frame (framelet) as the transformation matrix \(M\). This meticulous design of the $M$ transformation contributes to the superior reconstruction quality of F-TNN. However, the absence of a rapid implementation for multiplying the tensor with matrix \(M\) along the third mode leads to F-TNN requiring significantly more computational time compared to other evaluated methods.

It is worth noting that our current approach provides an effective balance between runtime efficiency and reconstruction quality, making it well-suited for potential real-world applications. This balanced approach is particularly relevant in practical settings where it is essential to consider both runtime and quality in big data applications.
\vspace{-5mm}
\section{Conclusion}
In this work, we present the t-CCS model, an extension of the matrix CCS model to a tensor framework. We provide both theoretical and experimental evidence demonstrating the model's  flexibility and feasibility. The ITCURTC algorithm, designed for the t-CCS model, provides a balanced trade-off between runtime efficiency and reconstruction quality. While it is not as effective as the state-of-the-art Bernoulli-based TC algorithm, it is still comparable in terms of PSNR and SSIM.  Thus, one of the directions of our future research will focus on enhancing reconstruction quality through the integration of the $M$-product. 

From the theoretical perspective, our current theoretical result shows that the t-CUR sampling scheme, as a special case of t-CCS model,  requires a complexity of $\mathcal{O}(\mu_{0}rn_{3}(n_{2}\log(n_{1}n_{3})+n_{1}\log(n_{2}n_{3})))$ which is more sampling-efficient than that of a general t-CCS scheme. This finding suggests potential to further improve the theoretical sampling complexity for the t-CCS model, an aspect we plan to explore in future work. Additionally, there is a need for a comprehensive theoretical analysis of the convergence behavior of the ITCURTC  within the t-CCS framework. Evaluating the algorithm’s robustness against additive noise will also be a critical focus for future research. Furthermore, while our current work is limited to third-order tensors, we aim to extend our approach to higher-order tensor configurations in future studies.
\vspace{-4mm}
 \section*{Acknowledgement}
 This work is partially supported by NSF DMS 2304489 and the computational resources provided by the Institute for Cyber-Enabled Research at Michigan State University. We are indebted to Xuanhao Huang Design for making our sketches of \Cref{fig:1x4layout1} a reality.
\appendix

\section{Proof of \Cref{thm:main-tCUR}}
In this section, we provide a detailed proof of \Cref{thm:main-tCUR}, which is one of the two important supporting theorems to our main result \Cref{thm:samplingcmp4ccs}. 

\subsection{Supporting lemmas for \Cref{thm:main-tCUR}}
Before proceeding to prove \Cref{thm:main-tCUR}, we will first introduce and discuss several supporting lemmas. These lemmas are crucial to establishing the foundation for the proof of \Cref{thm:main-tCUR}.
\begin{lemma}\label{row_selection} 
 Let $\mathcal{T}\in\mathbb{K}^{n_1\times n_2\times n_3}$, $I\subseteq[n_1]$, and $J\subseteq[n_2]$. 
 $\mathcal{S}_I$  and $\mathcal{S}_J$ are the horizontal and lateral sampling tensors associated with indices $I$ and $J$, respectively (see \Cref{def:samp_tensor}).  
 Then the following results hold
\begin{equation}\label{eqn:row_selection}
\overline{\mathcal{S}_{I}*\mathcal{T}}= 
\begin{bmatrix}
&[\mathcal{S}_{I}]_{:,:,1}\cdot
[\widehat{\mathcal{T}}]_{:,:,1} &  & &  \\
&&[\mathcal{S}_{I}]_{:,:,1}\cdot
[\widehat{\mathcal{T}}]_{:,:,2}&& \\ 
& &  &\ddots& \\
& & & &[\mathcal{S}_{I}]_{:,:,1}\cdot
[\widehat{\mathcal{T}}]_{:,:,n_3}
\end{bmatrix},
\end{equation}
\begin{equation}\label{eqn:col_selection}
    \overline{\mathcal{T}\ast \mathcal{S}_{J}}= 
\begin{bmatrix}
&[ \widehat{\mathcal{T}}]_{:,:,1} \cdot[\mathcal{S}_{J}]_{:,:,1}&&&\\
&&[ \widehat{\mathcal{T}}]_{:,:,2} \cdot[\mathcal{S}_{J}]_{:,:,1}&&\\
&&&\ddots& \\
&&&&[ \widehat{\mathcal{T}}]_{:,:,n_3} \cdot[\mathcal{S}_{J}]_{:,:,1}
\end{bmatrix}.
\end{equation}
\end{lemma}
\begin{proof}
Here, we will only focus on the proof of \eqref{eqn:row_selection}. First, it is easy to see that $\overline{\mathcal{S}_I*\mathcal{T}} = \overline{\mathcal{S}_I}\cdot \overline{\mathcal{T}}$. In addition, \[ \overline{\mathcal{S}_I}= 
\begin{bmatrix}
[\mathcal{S}_{I}]_{:,:,1}
 & & & \\
& [\mathcal{S}_{I}]_{:,:,1}& & \\
& & \ddots & \\
& & & [\mathcal{S}_{I}]_{:,:,1}
\end{bmatrix}\text{ and } \overline{\mathcal{T}}=
\begin{bmatrix}
    [\widehat{\mathcal{T}}]_{:,:,1}&&&\\
    &[\widehat{\mathcal{T}}]_{:,:,2}&&\\
    &&\ddots&\\
    &&&[\widehat{\mathcal{T}}]_{:,:,n_3}
\end{bmatrix}. 
\]   
 The result can thus be derived. 
\end{proof}

\begin{theorem}[Matrix Chernoff inequality 
\cite{Tropp2010UserFriendly}]  \label{chernoff}Consider a finite sequence $\left\{X_k\right\}$ of independent, random, Hermitian matrices with common dimension $d$. Assume that
$$
0 \leq \lambda_{\min }\left(X_k\right) \text { and } \lambda_{\max }\left(X_k\right) \leq L \text { for each index } k .
$$
Set $Y=\sum_k X_k$. Let $\mu_{\operatorname{min}}$ and   $\mu_{\max }$ be the minimum and maximum eigenvalues of $\mathbb{E}(Y)$, respectively, where $\mathbb{E}$ is the expectation operator. Then,
$$
\begin{array}{l}
\mathbb{P}\left\{\lambda_{\min }(Y) \leq(1-\varepsilon) \mu_{\min }\right\} \leq d\left[\frac{\mathrm{e}^{-\varepsilon}}{(1-\varepsilon)^{1-\varepsilon}}\right]^{\mu_{\min } / L} \quad \text { for } \varepsilon \in[0,1) \text {, and } \\
\mathbb{P}\left\{\lambda_{\max }(Y) \geq(1+\varepsilon) \mu_{\max }\right\} \leq d\left[\frac{\mathrm{e}^{\varepsilon}}{(1+\varepsilon)^{1+\varepsilon}}\right]^{\mu_{\max } / L} \quad \text { for } \varepsilon \geq 0. \\
\end{array}
$$
\end{theorem}
\begin{lemma}\label{lmm:key} 
Suppose $A$ is a block diagonal matrix, i.e., $A= 
\begin{bmatrix}
A_{1}  & && \\
& A_{2}  &&\\
& &\ddots & \\
& & &A_{n_3}
\end{bmatrix}$,   
  where  $A_i\in\mathbb{K}^{n_1\times r_i}$ and $A_{i}^{\top}A_{i} = \mathbb{I}_{r_i}$ ($r_i\times r_i$ identity matrix). Set $\vec{r}=(r_1,\cdots,r_{n_3})$. Let $I$ be a random subset of $[n_1]$.  Then for any $\delta\in [0,1)$, the $\|\vec{r}\|_1$-th   singular value of the matrix
\[N =:
\begin{bmatrix}
[\mathcal{S}_{I}]_{:,:,1}
 & & & \\
& [\mathcal{S}_{I}]_{:,:,1}& & \\
& & \ddots & \\
& & & [\mathcal{S}_{I}]_{:,:,1}
\end{bmatrix}
\begin{bmatrix}
A_{1}  & && \\
& A_{2}  &&\\
& &\ddots & \\
& & &A_{n_3}
\end{bmatrix}
\]
will be no less  than $\sqrt{\frac{(1-\delta)|I|}{n_1}}$ with probability at least 
\[1-\|\vec{r}\|_1e^{-(\delta+(1-\delta)\log(1-\delta))\frac{|I|}  {n_{1}\max_{i\in[n_1n_3]}{\|[A]_{i,:}\|_{\fro}^2}}}.\]
\begin{proof}
Firstly, it is easy to check that 
\begin{align*}
N=\left[\mathbb{I}_{n_3}\otimes[\S_{I}]_{:,:,1}\right]\cdot A =\sum\limits_{i \in I} \sum_{j=1}^{n_3} \mathbf{e}_{(j-1) n_1+i}\cdot [A]_{(j-1) n_1+i,:},
\end{align*}
where $\mathbf{e}_{(j-1) n_1+i}\in\mathbb{K}^{n_{1}n_3}$ is the standard column basis vector. 
Consider $\|\vec{r}\|_1\times \|\vec{r}\|_1$ Gram matrix
\begin{align*}
Y:= N^{\top}\cdot N
=&\sum_{i\in I}\sum_{j=1}^{n_3} (\mathbf{e}_{(j-1) n_1+i}\cdot [A]_{(j-1) n_1+i,:})^{\top}\cdot \mathbf{e}_{(j-1) n_1+i}\cdot [A]_{(j-1) n_1+i,:}\\
=&\sum_{i\in I}\sum_{j=1}^{n_3} [A]_{(j-1) n_1+i,:}^{\top} [A]_{(j-1) n_1+i,:}=:\sum_{i\in I} X_i,
\end{align*} 
where $X_i=\sum\limits_{j=1}^{n_3} [A]_{(j-1) n_1+i,:}^{\top} [A]_{(j-1) n_1+i,:}$. It is easy to see that $Y$ is a random matrix due to randomness inherited from the random set $I$. It is easy to see that each $X_i$ is a positive semidefinite matrix of size $\|\vec{r}\|_1\times \|\vec{r}\|_1$. Thus, the random matrix $Y$ in fact is a sum of $|I|$
random matrices 
sampled without replacement from the   set $\left\{X_1, X_2, \cdots X_{n_1}\right\}$ of positive semi-definite matrices. 
Notice that 
\begin{align*}
\lambda_{\max}\left(X_i\right) 
    &= \lambda_{\max}\left( \sum_{j=1}^{n_3} (\mathbf{e}_{(j-1) n_1+i}\cdot [A]_{(j-1) n_1+i,:})^{\top}\cdot \mathbf{e}_{(j-1) n_1+i}\cdot [A]_{(j-1) n_1+i,:}\right)\\
&=\left|\sigma_{\max}\left(\sum_{j=1}^{n_3} \mathbf{e}_{(j-1) n_1+i}\cdot [A]_{(j-1) n_1+i,:}\right)\right|^{2} 
    \leq \max_{i}\|[A]_{i,:}\|_{\fro}^{2}.
\end{align*}
By the orthogonal property of matrix $A$, it is easy to see that $\mathbb{E}\left(X_i\right)=\frac{1}{n_1}\mathbb{I}_{\|\vec{r}\|_1}$ and thus $\mathbb{E}(Y)=\frac{|I|}{n_1}$.  Thus, by the fact that $\lambda_{\min}(Y) = \sigma_{\min}^2(N)$ and by the matrix Chernoff inequality, we have
\[
\mathbb{P}\left(\sigma_{\min}(N) \leq \sqrt{\frac{(1-\delta) |I|}{n_1}}\right) \leq
\|\vec{r}\|_1e^{-(\delta+(1-\delta)\log(1-\delta))\frac{|I|}  {n_{1}\max_{i\in[n_1n_3]}{\|[A]_{i,:}\|_{\fro}^2}}}, \forall \delta \in [0,1).
\]
\end{proof}
\end{lemma}
\vspace{-35pt}
In the following section, we mainly delve into the proof of \Cref{thm:main-tCUR} to  explain the likelihood of the t-CUR decomposition holding.
\subsection{The proof of \Cref{thm:main-tCUR}}
\begin{proof}
According to \Cref{tensor_cur},   $\mathcal{T}=\mcC*\mathcal{U}^\dagger*\mcR$ is equivalent to  $ \rank_m(\X)=\rank_m(\mcC)=\rank_m(\mcR)$. 
Therefore, it suffices to prove that $\rank_m(\X)=\rank_m(\mcC)=\rank_m(\mcR)$ holds with probability at least $1-\frac{1}{n_1^{\beta_1}}-\frac{1}{n_2^{\beta_2}}$ under the given conditions.  Notice that 
\begin{align}
\overline{\mathcal{T}} &= \begin{bmatrix}
 [\widehat{\mathcal{T}}]_{:,:,1}&&&\\
 &[\widehat{\mathcal{T}}]_{:,:,2}&&\\
 &&\ddots&\\
 &&&[\widehat{\mathcal{T}}]_{:,:,n_3}
\end{bmatrix} \notag\\
& =\begin{bmatrix}
W_{1} \Sigma_{1}  V_{1}^{\top}&&&\\
 &W_{2} \Sigma_{2} V_{2}^{\top}&&\\
&&\ddots&\\
 &&&W_{n_3} \Sigma_{n_3} V_{n_3}^{\top}
\end{bmatrix}\label{definition:wsv}\\
& =\begin{bmatrix}
W_1&&&\\
 &W_{2} &&\\
&&\ddots&\\
 &&&W_{n_3} 
\end{bmatrix}\cdot \begin{bmatrix}
\Sigma_{1}&&&\\
 &\Sigma_{2} &&\\
 &&\ddots&\\
 &&&\Sigma_{n_3} 
\end{bmatrix}\cdot\begin{bmatrix}
V_{1}^{\top}&&&\\
 &V_{2}^{\top} &&\\
 &&\ddots&\\
 &&&V_{n_3}^{\top} 
\end{bmatrix}\notag\\
&=: W \cdot \Sigma \cdot V^{\top},\notag
\end{align}
 where  $W_i \Sigma_{i}  V_{i}^{\top}$ in \eqref{definition:wsv} is the compact SVD of  $[\overline{\mathcal{T}}]_{:,:,i}$ for $i\in[n_3]$. And $\overline{\mcR}=\overline{\mathcal{S}_I}W\Sigma V^{\top}$. 
By the definition of tensor multi-rank,  we have $W_i\in\mathbb{K}^{n_1\times r_i}$, $\Sigma_i\in\mathbb{K}^{r_i\times r_i}$, $V_i\in\mathbb{K}^{n_2\times r_i}$,  $\mathrm{W} \in \mathbb{K}^{n_{1}n_{3}\times \|\vec{r}\|_1}, \Sigma \in \mathbb{K}^{\|\vec{r}\|_1\times \|\vec{r}\|_1}$, and $V \in \mathbb{K}^{n_{2}n_{3}\times \|\vec{r}\|_1}$. 

Consequently, demonstrating that $\rank(\overline{\mcR}) = \|\vec{r}\|_{1}$ suffices to ensure the condition that $\rank_m(\X) = \rank_m(\mcR)$. Observe that $\Sigma$ is a square matrix with full rank and $V$ has full column rank. By the Sylvester rank inequality,   $\rank(\overline{\mcR})=\|\vec{r}\|_1$ can be guaranteed by showing $\rank(\overline{\mathcal{S}}_I\cdot W)=\|\vec{r}\|_1$. By applying  \Cref{lmm:key}, we have that for all $\delta\in[0,1)$,
\[
\mathbb{P}\left(\sigma_{\|\vec{r}\|_{1}}(\overline{S}_{I}\cdot W) \leq \sqrt{(1-\delta)|I|/{n_1}}\right) \leq 
\|\vec{r}\|_1e^{-(\delta+(1-\delta)\log(1-\delta))\frac{|I|}{\mu_0\|\vec{r}\|_{\infty}}}.
\]
 $|I| \geq \frac{\beta_1\mu_0 \|r\|_{\infty} \log \left(n_1\|\vec{r}\|_{1}\right)}{\delta+(1-\delta)\log(1-\delta)}$ implies
$\mathbb{P}\left(\sigma_{\|\vec{r}\|_{1}}(\overline{S}_{I}\cdot W) \leq \sqrt{ {(1-\delta)|I|}/{{n_1}}}\right) \leq  \frac{1}{{n_1}^{\beta_1}}$. Note that $\mathbb{P}\left(\rank(\overline{S}_{I}\cdot W)< \|\vec{r}\|_{1}\right) \leq \mathbb{P}\left(\sigma_{\|\vec{r}\|_{1}}(\overline{S}_{I}\cdot W) \leq \sqrt{\frac{(1-\delta)|I|}{{n_1}}}\right)$.
We thus have when $|I| \geq \frac{\beta_1\mu_0 \|r\|_{\infty} \log \left(n_1\|\vec{r}\|_{1}\right)}{\delta+(1-\delta)\log(1-\delta)}$, 
\begin{align*}
\mathbb{P}\left(\rank(\overline{S}_{I}\cdot W)= \|\vec{r}\|_{1}\right)
= & 1 -\mathbb{P}\left(\rank(\overline{S}_{I}\cdot W)< \|\vec{r}\|_{1}\right) \\
\geq &1- \mathbb{P}\left(\sigma_{\|\vec{r}\|_{1}}(\overline{S}_{I}\cdot W) \leq \sqrt{\frac{(1-\delta)|I|}{{n_1}}}\right) 
\geq 1- \frac{1}{{n_1}^{\beta_1}}.
\end{align*}
Similarly, one can show that $\rank(\overline{\mathcal{S}}_{J}\cdot \overline{\mathcal{V}})=\|\vec{r}\|_1$ holds with probability at least $1-\frac{1}{{n_2}^{\beta_2}}$ provided that $|J|\geq  \frac{\beta_{2}\mu_0 \|r\|_{\infty} \log \left(n_2\|\vec{r}\|_{1}\right)}{\delta+(1-\delta)\log(1-\delta)}$. 

Combining all the statements and setting $\delta = 0.815$ and $\beta_1 = \beta_2=\beta$, we  conclude that $\mathcal{T}=\mcC*\mathcal{U}^\dagger*\mcR$ holds with   probability  at least $1-\frac{1}{n_{1}^{\beta}}-\frac{1}{n_{2}^{\beta}}$, provided that $|I|\geq 2\beta \mu_0 \|\vec{r}\|_{\infty} \log \left(n_1\|\vec{r}\|_{1}\right)$ and  $|J|\geq 2\beta \mu_0 \|\vec{r}\|_{\infty} \log \left(n_2\|\vec{r}\|_{1}\right)$. 
\end{proof}
Next we will present the proof to our another theoretical foundation, \Cref{thm:sr4generalTCB}, which also plays a vital role in our main theoretical result, \Cref{thm:samplingcmp4ccs}.
\section{Proof of \Cref{thm:sr4generalTCB} }
In this section, we provide a detailed proof of \Cref{thm:sr4generalTCB},  another important supporting theorem for our main theoretical result, \Cref{thm:samplingcmp4ccs}.
 To the best of our knowledge, there is no existing tensor version of the result found in \citep[Theorem 1.1]{recht2011simpler}, which provides an explicit expression of numerical constants within the theorem's statement.  Existing results related to tensor versions, such as \citep[Theorem 3.1]{zhang2017exact} in the context of tensor completion, typically imply numerical constants implicitly. One can see that \citep[Theorem 3.1]{zhang2017exact} does not give an explicit expression of the numerical constants of $c_0, c_1$ and $c_2$. 
\begin{theorem}{\citep[Theorem 3.1]{zhang2017exact}}\label{thm2}
 Suppose that the reduced t-SVD of \( \mathcal{Z}\in\mathbb{K}^{n_1\times n_2\times n_3} \)   is given by \( \mathcal{Z} = \mathcal{W} * \mathcal{S} * \mathcal{V}^\top \) where \( \mathcal{W} \in \mathbb{R}^{n_1 \times r \times n_3} \), \( \mathcal{S} \in \mathbb{R}^{r \times r \times n_3} \),  \( \mathcal{V} \in \mathbb{R}^{n_2 \times r \times n_3} \), and \( \mathcal{Z} \) also satisfies the standard tensor $\mu_0$-incoherence condition. Suppose the entries in $\Omega$ are sampled according to the Bernoulli model with probability $p$. Then there exist constants \( c_0, c_1, c_2 > 0 \) such that if
\[
p \geq c_0 \frac{\mu_0 r \log(n_3(n_1 + n_2))}{\min\{n_1, n_2\}},
\]
then \( \mathcal{Z} \) is the unique minimizer to the follow optimization 
\[
\min_{\mathcal{X}}~
 \|\mathcal{X}\|_{\TNN},~ \text{subject to }
 \mathcal{P}_{\Omega}(\mathcal{X}) = \mathcal{P}_{\Omega}(\mathcal{Z}),
\]
with probability at least
\[
1 - c_1((n_1 + n_2)n_3)^{-c_2}.
\]
\end{theorem}
Our work makes a substantial contribution by meticulously analyzing these numerical constants and providing explicit formulations for their expressions. These theoretical advancements are comprehensively elaborated in our theoretical section.
\subsection{Notation}
Before moving forward, let us introduce several notations and definitions used throughout the rest of the supplemental material but not covered in earlier sections. We introduce two specific tensors: $\mrme_i \in \mathbb{K}^{n \times 1 \times n_3}$, which contains a $1$ in the $(i,1,1)$-th position and zeros in all other positions; and $\dme_k \in \mathbb{K}^{1 \times 1 \times n_3}$, which includes a $1$ in the $(1,1,k)$-th position with zeros elsewhere. Additionally, for a specified set $\Omega$, we use $\delta_{i,j,k}$ to denote the indicator function $\textbf{1}_{(i,j,k)\in \Omega}$.

 \begin{definition} 
Suppose $\mathcal{Z}$ is an $n_1 \times n_2 \times n_3$ tensor and its compact t-SVD is given by $\mathcal{Z}=\mathcal{W} * \mathcal{S} * \mathcal{V}^{\top}$ where $\mathcal{W} \in \mathbb{K}^{n_1 \times r \times n_3}, \mathcal{S} \in \mathbb{K}^{r \times r \times n_3}$ and $\mathcal{V} \in \mathbb{K}^{n_2 \times r \times n_3}$. 
Define projection space $\mathbb{T}$ as
$$
\left\{\sum\limits_{k=1}^{r}([\mathcal{W}]_{:,k,:}*[\mathcal{X}]_{:,k,:}^{\top}+[\mathcal{Y}]_{:,k,:}*[\mathcal{V}]_{:,k,:}^{\top}) : \mathcal{X} \in \mathbb{K}^{n_2\times r\times n_3}, \mathcal{Y} \in \mathbb{K}^{n_1 \times r \times n_3} \right\}
$$ and 
the orthogonal projection space $\mathbb{T}^{\perp}$ is 
the orthogonal complement $\mathbb{T}$ in $\mathbb{K}^{n_1\times n_2 \times n_3}$.
Define $\mathcal{P}_{\mathbb{T}}(\mathcal{X})$ and $\mathcal{P}_{\mathbb{T}^{\perp}}(\mathcal{X})$ as
\begin{align*}  
\mathcal{P}_{\mathbb{T}}(\mathcal{X})&=\mathcal{W} * \mathcal{W}^{\top} * \mathcal{X}+\mathcal{X} * \mathcal{V} * \mathcal{V}^{\top}-\mathcal{W} * \mathcal{W}^{\top} * \mathcal{X} * \mathcal{V} * \mathcal{V}^{\top}, \\ 
\text{and~}~\mathcal{P}_{\mathbb{T}^{\perp}}(\mathcal{X})
 &=\left(\mathcal{I}_{n_1}-\mathcal{W} * \mathcal{W}^{\top}\right) * \mathcal{X} *\left(\mathcal{I}_{n_2}-\mathcal{V} * \mathcal{V}^{\top}\right),   \end{align*}
where $\mathcal{I}_{n_1}\in\mathbb{K}^{n_1\times n_1\times n_3}$ and $\mathcal{I}_{n_2}\in\mathbb{K}^{n_2\times n_2\times n_3}$  are the identity tensors.
\end{definition}
\begin{definition} 
Define the operator $\mathcal{R}_{\Omega}: \mathbb{K}^{n_1\times n_2 \times n_3}\rightarrow \mathbb{K}^{n_1\times n_2\times n_3}$ as:
$$
\mathcal{R}_{\Omega}(\mathcal{X}) = \sum \limits_{i,j,k}\frac{1}{p}\delta_{i,j,k}[\mathcal{X}]_{i,j,k}\mrme_i * \dme_k * \mrme_j^{\top}.
$$
\end{definition}
\begin{definition}  
Given two tensors $\mathcal{A}\in \mathbb{K}^{n_1\times n_2 \times n_3}$ and $\mathcal{B}\in \mathbb{K}^{n_1\times n_2 \times n_3}$, the inner product of these two tensors is defined as: 
    $$
    \langle\mathcal{A}, \mathcal{B}\rangle=\frac{1}{n_3} \operatorname{trace}\left(\overline{\mathcal{B}}^{\top}\cdot\overline{\mathcal{A}}\right).
    $$
\end{definition}
Before introducing the tensor operator norm, we need to define a transformed version of a tensor operator. 
Given a tensor operator $\mathcal{F}:\mathbb{K}^{n_1\times n_2 \times n_3} \rightarrow \mathbb{K}^{n_1\times n_2\times n_3}$, the associated transformed operator $\overline{\mathcal{F}}: \mathbb{B} \rightarrow \mathbb{B}$ is defined as $\overline{\mathcal{\mathcal{F}}}(\overline{\mathcal{X}}) = \overline{\mathcal{F}(\mathcal{X})}$, where $\mathbb{B}=\left\{\overline{\mathcal{B}}: \mathcal{B} \in \mathbb{K}^{n_1 \times n_2 \times n_3}\right\}$. 
 
\begin{definition}[Tensor operator norm]\label{def:tensor-operator-norm}
Given an operator $\mathcal{F}:\mathbb{K}^{n_1\times n_2 \times n_3} \rightarrow \mathbb{K}^{n_1\times n_2\times n_3}$, the operator norm $\|\mathcal{F}\|$ is defined as $\|\mathcal{F}\| = \|\overline{\mathcal{F}}\| = \max\limits_{\|\overline{\mathcal{X}}\|_{\fro}=1}\|\overline{\mathcal{F}}(\overline{\mathcal{X}})\|_{\fro}=\max\limits_{\|\overline{\mathcal{X}}\|_{\fro}=1}\left\|\overline{\mathcal{F}(\mathcal{X})}\right\|_{\fro}$.
\end{definition}
\begin{definition}[$\ell_{\infty,2}$ norm \citep{zhang2017exact}] Given a tensor $\mathcal{X}\in \mathbb{K}^{n_1\times n_2 \times n_3}$, its $\ell_{\infty,2}$ norm is defined as 
$$\|\mathcal{X}\|_{\infty,2}:=\max\left\{\max\limits_{i}\sqrt{\sum\limits_{b,k}[\mathcal{X}]_{i,b,k}^2}, \max\limits_{j}\sqrt{\sum\limits_{a,k}[\mathcal{Z}]_{a,j,k}^2}\right\}.
$$
\end{definition}
\begin{definition}[Tensor infinity norm]Given a tensor $\mathcal{X}\in \mathbb{K}^{n_1\times n_2\times n_3}$, its  infinity norm  is defined as $\|\mathcal{X}\|_{\infty}:=\max\limits_{i,j,k} |[\mathcal{X}]_{i,j,k}|$.
\end{definition}
\subsection{Problem statement}
In this section, we present a formal definition of the tensor completion problem based on the Bernoulli sampling model. Consider a third-order tensor \( \mathcal{Z} \in \mathbb{K}^{n_1 \times n_2 \times n_3} \) with  tubal-rank   \( r \). We denote $\Omega$ as the set of indices of the observed entries, generated according to the Bernoulli sampling model with probability $p$. 
We define the sampling operator $\mathcal{P}_{\Omega}$ such that for a given tensor $\mathcal{X}$ in $\mathbb{K}^{n_1 \times n_2 \times n_3}$, \[\mathcal{P}_{\Omega}(\mathcal{X})=\sum_{(i,j,k)\in \Omega}[\mathcal{X}]_{i,j,k}\mathcal{E}_{i,j,k},\]
where  $\mathcal{E}_{i,j,k}$ is a tensor in $\{0,1\}^{n_1 \times n_2 \times n_3}$ with all elements zero except for the one at the position indexed by $(i,j,k)$. The primary goal of the tensor completion problem is to reconstruct the tensor \( \mathcal{Z} \) from the entries in $\Omega$. 

We utilize the approach proposed by \cite{
zhang2017exact}, which addresses the TC issue through a specific convex optimization problem formulated as follows: 
\begin{equation}\label{optimize}
\begin{gathered}
\min_{\mathcal{X}}\|\mathcal{X}\|_{\TNN}, 
\text { subject to } \mathcal{P}_{\Omega}(\mathcal{X})=\mathcal{P}_{\Omega}(\mathcal{Z}).
\end{gathered}
\end{equation}
Notice that $\TNN$ is convex but not strictly convex. Thus, there might be more than one local minimizer to   \eqref{optimize}. Therefore, we need to establish conditions to ensure that our optimization problem has a unique minimizer, which is exactly the tensor we seek to recover. The question of under what conditions $\mathcal{Z}$  is the unique minimizer of  \eqref{optimize}  naturally arises. In response, \Cref{conditions} gives an affirmative answer. 

Before proceeding, it is important to highlight that in the following context, for convenience, we will interchangeably use $\|\cdot\|$ to denote the tensor spectral norm, 
tensor operator norm or the matrix spectral norm, depending on the specific situation.

\begin{proposition}[\cite{Lu18}] \label{conditions} Assume that $\mathcal{Z}\in\mathbb{K}^{n_1 \times n_2 \times n_3}$ of tubal-rank $r$ satisfies the tensor $\mu_0$-incoherence condition and its compact t-SVD is given by $\mathcal{Z}=\mathcal{W} * \mathcal{S} * \mathcal{V}^{\top}$, where $\mathcal{W} \in \mathbb{K}^{n_1 \times r \times n_3}, \mathcal{S}\in \mathbb{K}^{r \times r \times n_3}$, and $\mathcal{V} \in \mathbb{K}^{n_2 \times r \times n_3}$. Suppose that $\Omega$ is generated according to the Bernoulli sampling model with probability $p$. 
Then tensor $\mathcal{Z}$ is the unique minimizer of  \eqref{optimize} if the following two conditions hold:
\begin{condition}
    \label{CondI}  $\left\|\mathcal{P}_{\mathbb{T}}\mathcal{R}_{\Omega}\mathcal{P}_{\mathbb{T}}-\mathcal{P}_{\mathbb{T}}\right\| \leq \frac{1}{2}$.
\end{condition}
\begin{condition}
    \label{CondII}  There exists a tensor $\mathcal{Y}$ such that $\mathcal{P}_{\Omega}(\mathcal{Y})=\mathcal{Y}$ and
\begin{enumerate}[ leftmargin=.5in]
    \item[(a)] $\left\|\mathcal{P}_{\mathbb{T}}(\mathcal{Y})-\mathcal{W}*\mathcal{V}^{\top}\right\|_{\fro} \leq \frac{1}{4}\sqrt{\frac{p}{n_3}}$;
    \item[(b)] $\left\|\mathcal{P}_{\mathbb{T}^{\perp}}(\mathcal{Y})\right\|\leq \frac{1}{2}$.
\end{enumerate}
\end{condition}
\end{proposition}
Based on  \Cref{conditions}, our main result is derived through probabilistic estimation of the Condition~\ref{CondI} and Condition~\ref{CondII}. Throughout this computation, we explicitly determine both the lower bound of the sampling probability
$p$ and the  probability of the exact recovery of  
$\mathcal{Z}$. The architecture of the entire proof is described in \ref{proof}. 
\subsection{Architecture of the  proof of \Cref{thm:sr4generalTCB}}\label{proof}
 The proof of \Cref{thm:sr4generalTCB} follows the pipeline developed in \citep{Lu18,
 zhang2017exact}. We first  state a sufficient condition for $\mathcal{Z}$ to be the unique solution to \eqref{optimize} by constructing a dual certificate $\mathcal{Y}$ obeying two conditions. This result is summarized in \Cref{conditions}. To obtain our  result \Cref{thm:sr4generalTCB}, we just need to show that the conditions in \Cref{conditions} hold with high probability. \Cref{thm:sr4generalTCB} is built on the basis of \Cref{condtion1}, \Cref{conditionII1}, \Cref{conditionII2}, and \Cref{conditionII33}. A detailed roadmap of the proof towards \Cref{thm:sr4generalTCB} is outlined in \Cref{architecture}. 
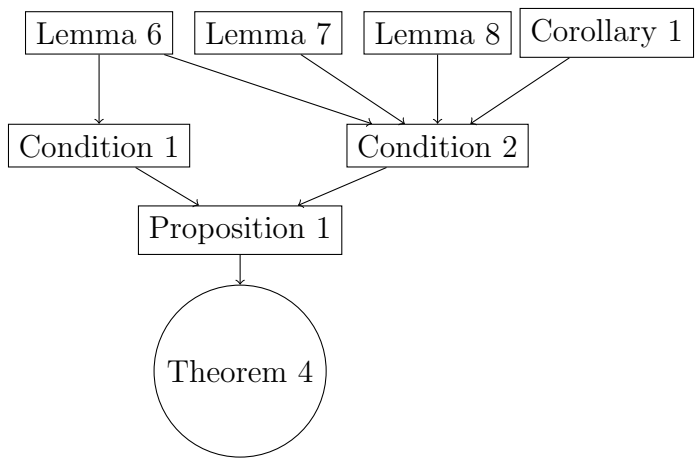
\begin{figure}[!th]
\centering
\begin{tikzpicture}[scale=0.75]
\node[draw, rectangle] (circ) at (4.5,-4.5) {\Cref{conditions}};
\node[draw, rectangle] (rect5) at (2,-3) {\Cref{CondI}};
\node[draw, rectangle] (rect6) at (8,-3) {\Cref{CondII}};
\node[draw,rectangle](rect7) at (2,-1){\Cref{condtion1}};
\node[draw,rectangle](rect8) at (5,-1){\Cref{conditionII1}};
\node[draw,rectangle](rect9) at (8,-1){\Cref{conditionII2}};
\node[draw,rectangle](rect10) at (11,-1){\Cref{conditionII33}};
\node[draw,circle, minimum size=2cm](rect11) at (4.5,-7){\Cref{thm:sr4generalTCB}};
  \draw[->] (rect7) -- (rect5);
  \draw[->] (rect7) -- (rect6);
  \draw[->] (rect8) -- (rect6);
  \draw[->] (rect9) -- (rect6);
  \draw[->] (rect10) -- (rect6);
 
  \draw[->]  (rect5) -- (circ);
  \draw[->] (rect6) -- (circ);
  \draw[->] (circ) -- (rect11);
\end{tikzpicture}
\caption{\footnotesize The structure of the proof of \Cref{thm:sr4generalTCB}:  
The core of the proof for  \Cref{thm:sr4generalTCB} relies on assessing the probability that certain conditions, specified in \Cref{conditions}, are met. \Cref{CondI} and \Cref{CondII} serve as sufficient criteria to ensure the applicability of \Cref{conditions}. Thus, the proof of \Cref{thm:sr4generalTCB} primarily involves determining the likelihood that \Cref{CondI} and \Cref{CondII} are satisfied. The probabilistic assessment of \Cref{CondI} utilizes  \Cref{condtion1} as a fundamental instrument. Similarly, the evaluation of \Cref{CondII} employs \Cref{condtion1,conditionII1,conditionII2}, and \Cref{conditionII33} as essential tools.}\label{architecture}
\vspace{-5mm}
\end{figure}

\subsection{Proof of \Cref{thm:sr4generalTCB}} 
In this section, we present a comprehensive proof of \Cref{thm:sr4generalTCB}. Our primary objective is to  demonstrate that, under the stipulated condition on $p$ as specified in \eqref{ppp} of the main text, Conditions 1 and 2 in \Cref{conditions} are satisfied with high probability. Before delving into the proof of \Cref{thm:sr4generalTCB}, we will introduce several supporting lemmas to establish the necessary theoretical foundation.
\subsubsection{\textnormal{Supporting lemmas for the proof of \Cref{thm:sr4generalTCB}}}\label{usefullemmathm1}
\begin{lemma}[Non-commutative Bernstein inequality \citep{Tropp2010UserFriendly}] \label{thm:tropp}  
Let $X_1$,  $\dots$, $X_{L}\in\mathbb{K}^{n_1\times n_2}$ be independent zero-mean random matrices. Suppose $$\sigma^2=\max\left\{\left\|\mathbb{E}[\sum\limits_{k=1}^{L}X_{k}X_{k}^{\top}]\right\|,\left\|\mathbb{E}[\sum\limits_{k=1}^{L}X_{k}^{\top}X_{k}]\right\|\right\}$$ and $\left\|X_{k}\right\|\leq M$. Then for any $\tau \geq 0,$
$$
\mathbb{P}\left(\left\|\sum\limits_{k=1}^{L}X_{k}\right\|\geq \tau\right)\leq (n_1+n_2)\exp\left(\frac{-\tau^2/2}{\sigma^2+\frac{M\tau}{3}}\right).
$$
\end{lemma}

 \noindent{The following lemma is a variant of Non-commutative Bernstein inequality by choosing  $\tau = \sqrt{2c \sigma^2 \log \left(n_1+n_2\right)}+c M\log \left(n_1+n_2\right)$.}
\begin{lemma}\label{BernsteinInequality2} \rm
 Let $X_1,\cdots,X_{L}\in\mathbb{K}^{n_1\times n_2}$ be independent zero-mean random matrices. Suppose $$\sigma^2=\max\left\{\left\|\mathbb{E}[\sum\limits_{k=1}^{L}X_{k}X_{k}^{\top}]\right\|,\left\|\mathbb{E}[\sum\limits_{k=1}^{L}X_{k}^{\top}X_{k}]\right\|\right\}$$ and $\left\|X_{k}\right\|\leq M$. Then  
 $$
  \mathbb{P}\left(\left\|\sum\limits_{k=1}^{L} X_k\right\| \geq  \sqrt{2c \sigma^2 \log \left(n_1+n_2\right)}+c M\log \left(n_1+n_2\right)\right)\leq (n_1+n_2)^{1-c}
 ,$$
 where $c$ is any positive number greater than $1$. 
\end{lemma}
  The following fact is very useful, and we will frequently use  the result for the proofs of \Cref{thm:sr4generalTCB} and   \Cref{conditions}.
\begin{lemma}\label{ptbound} Suppose that $\mathcal{Z}$ is an $n_1 \times n_2 \times n_3$ tensor with its compact t-SVD given by $\mathcal{Z}=\mathcal{W} * \mathcal{S} * \mathcal{V}^{\top}$ and satisfies   the tensor $\mu_0$-incoherence condition. Then,
$$\|\mathcal{P}_{\mathbb{T}}(\mrme_i * \dme_k * \mrme_j^{\top})\|_{\fro}^2 \leq \frac{(n_1+n_2)\mu_{0}r}{n_{1}n_{2}}.$$
\end{lemma}
 The next lemma determines the likelihood that   $\|\mathcal{P}_{\mathbb{T}}\mathcal{R}_{\Omega}\mathcal{P}_{\mathbb{T}}-\mathcal{P}_{\mathbb{T}}\|\leq \frac{1}{2}$, which is essential for  verifying Condition \ref{CondI} in \Cref{conditions}.
\begin{lemma}  \rm
 Assume that $\Omega$ is generated according to the Bernoulli distribution with probability $p$, then 
 $$\label{condtion1}
\left\|\mathcal{P}_{\mathbb{T}}\mathcal{R}_{\Omega}\mathcal{P}_{\mathbb{T}}-\mathcal{P}_{\mathbb{T}}\right\|\leq \frac{1}{2}
 $$ holds with probability at least $1-2n_1n_2n_3  \exp\left(-\frac{3pn_1n_2}{28(n_1+n_2)\mu_0 r}\right)$.
\end{lemma}

The following lemma establishes that for an arbitrary tensor $\mathcal{X} \in \mathbb{K}^{n_1 \times n_2 \times n_3}$,  $\|\mathcal{R}_{\Omega}(\mathcal{X}) - \mathcal{X}\|$ can be bounded by a combination of the tensor infinity norm and the $\ell_{\infty,2}$ norm with high probability.
\begin{lemma}  \label{conditionII1}\rm
Let $\mathcal{X}\in \mathbb{K}^{n_1\times n_2 \times n_3}.$  Assume that $\Omega$ is generated according to the Bernoulli distribution with probability $p$. Then, for any constant $c_2>1$,  
\begin{equation}\left\|\mathcal{R}_{\Omega}(\mathcal{X}) - \mathcal{X}\right\| \leq \sqrt{\frac{2c_2}{p}\log((n_1+n_2)n_3)}\|\mathcal{X}\|_{\infty,2}+\frac{c_2\log((n_1+n_2)n_3)}{p}\|\mathcal{X}\|_{\infty}
 \end{equation}
holds with probability at least $1-((n_1+n_2)n_3)^{1-c_2}$.
\end{lemma}
 
The following lemma states that for $\mathcal{X}\in\mathbb{K}^{n_1\times n_2\times n_3}$, $\left\|(\mathcal{P}_{\mathbb{T}}\mathcal{R}_{\Omega}(\mathcal{X})-\mathcal{P}_{\mathbb{T}}(\mathcal{X}))\right\|_{\infty,2}$ can be controlled by a combination of the $\ell_{\infty,2}$ norm  and the tensor infinity norm of $\mathcal{X}$ with high probability.
 
\begin{lemma}\label{conditionII2}\rm
Assume that $\Omega$ is generated according to the Bernoulli distribution with probability $p$.
For any positive number  $c_1 \geq 2$ and $\mathcal{X}\in\mathbb{K}^{n_1\times n_2\times n_3}$, then  
\begin{align*}
&\mathbb{P}\left(\left\|(\mathcal{P}_{\mathbb{T}}\mathcal{R}_{\Omega}(\mathcal{X})-\mathcal{P}_{\mathbb{T}}(\mathcal{X}))\right\|_{\infty,2}\leq\sqrt{\frac{4c_1(n_{1}+n_{2})\mu_{0}r\log\left((n_1+n_2)n_3\right)}{pn_{1}n_{2}}}\cdot\|\mathcal{X}\|_{\infty,2}   \right.\\
&+\left.\frac{c_1\log((n_1+n_2)n_3)}{p}\sqrt{\frac{(n_1+n_2)\mu_{0}r}{n_{1}n_{2}}}\|\mathcal{X}\|_{\infty}\right)\geq 1-((n_1+n_2)n_3)^{2-c_1}.
\end{align*} 
\end{lemma}
The following lemma states that, for $\mathcal{X}\in \mathbb{R}^{n_1\times n_2\times n_3}$,  $\|\mathcal{P}_{\mathbb{T}} \mathcal{R}_{\Omega} \mathcal{P}_{\mathbb{T}}(\mathcal{X})-\mathcal{P}_{\mathbb{T}}(\mathcal{X})\|_{\infty}$ can be bounded by the tensor infinity norm of $\mathcal{P}_{\mathbb{T}}(\mathcal{X})$ with high probability.
\begin{lemma}\label{conditionII3}\rm
Assume that $\Omega$ is generated according to the Bernoulli distribution with probability $p$.
For any $\mathcal{X}\in \mathbb{K}^{n_1\times n_2\times n_3}$,  then
$$
\left\|\left(\mathcal{P}_{\mathbb{T}} \mathcal{R}_{\Omega} \mathcal{P}_{\mathbb{T}}-\mathcal{P}_{\mathbb{T}}\right)(\mathcal{X})\right\|_{\infty} \leq \frac{1}{2}\|\mathcal{P}_{\mathbb{T}}(\mathcal{X})\|_{\infty}
$$
holds with probability at least $1-2n_{1}n_{2}n_{3}\exp\left(\frac{-3pn_{1}n_{2}}{16(n_1+n_2)\mu_{0}r}\right)$.
\end{lemma}
 {When $\mathcal{P}_{\mathbb{T}}(\mathcal{X})=\mathcal{X}$, we can easily achieve the following corollary. }
\begin{corollary} \label{conditionII33}\rm 
Assume that $\Omega$ is generated according to the Bernoulli distribution with probability $q$. For any $\mathcal{X}\in \mathbb{K}^{n_1\times n_2\times n_3}$, if $\mathcal{P}_{\mathbb{T}}(\mathcal{X}) = \mathcal{X}$, then
$$
\left\|\left(\mathcal{P}_{\mathbb{T}} \mathcal{R}_{\Omega_{t}} \mathcal{P}_{\mathbb{T}}-\mathcal{P}_{\mathbb{T}}\right)(\mathcal{X})\right\|_{\infty} \leq \frac{1}{2}\|\mathcal{X}\|_{\infty}
$$
holds with probability at least $1-2n_{1}n_{2}n_{3}\exp\left(\frac{-3qn_{1}n_{2}}{28(n_1+n_2)\mu_{0}r}\right)$.
\end{corollary}
\Cref{conditionII33} is used to provide a probabilistic estimate for  the lower bound of $\|\mathcal{D}_{t}\|_{\infty}$, where $\mathcal{D}_{t}$ is defined in \eqref{Dt} later. With this groundwork, we are now prepared to present the proof of \Cref{thm:sr4generalTCB}.


\subsubsection{\textnormal{Proof of \Cref{thm:sr4generalTCB}}}
First of all, one can get that Condition~\ref{CondI} holds with probability  at least  
\begin{equation}\label{eqn:probability4CI}
    1-2n_1n_2n_3  \exp\left(-\frac{3pn_1n_2}{28(n_1+n_2)\mu_0 r}\right) 
\end{equation}
according to \Cref{condtion1}.

Next, our main goal is to construct a dual certificate $\mathcal{Y}$ that satisfies Condition~\ref{CondII}. We do this using the Golfing Scheme \citep{candes2011robust}. 
Choose  $t_0$ as
\begin{equation}
\label{t0}
 t_0 \geq \left\lceil\log_2\left(4\sqrt{\frac{n_3r}{p}}\right)\right\rceil, 
\end{equation}
where $\lceil\cdot
\rceil$ is the ceiling function. Suppose that the set $\Omega$ is generated from $\Omega = \cup_{t=1}^{t_0} \Omega_{t}$. For each $t$, we assume that $\mathbb{P}[(i,j,k)\in \Omega_t] = q := 1 - (1-p)^{\frac{1}{t_0}}$. It is easy to see that for any $(i,j,k) \in [n_1]\times [n_2] \times [n_3]$, 
\begin{align*}
  \mathbb{P}[(i,j,k)\in \Omega] = & 1 - \mathbb{P}[(i,j,k)\notin \cup_{t=1}^{t_0}\Omega_{t}]  \\
= &1 - \prod\limits_{t=1}^{t_0}\mathbb{P}[(i,j,k)\in \Omega_{t}^{c}]  = 1 -\prod\limits_{t=0}^{t_0}(1-p)^{\frac{1}{t_0}} 
= p.
\end{align*}
Therefore, the construction of $\Omega = \bigcup_{t=1}^{t_0}\Omega_{t}$  shares the same distribution as that of the original $\Omega$. 
Let $\{\mathcal{A}_{t}\in \mathbb{K}^{n_1\times n_2 \times n_3}: t \in \{0\}\cup[t_0]\}$ be a sequence of tensors with 
$$
\mathcal{A}_{0} = \mathbf{0} ~\text{and}~  \mathcal{A}_{t} = \mathcal{A}_{t-1} + \mathcal{R}_{\Omega_t}\mathcal{P}_{\mathbb{T}}(\mathcal{W}*\mathcal{V}^{\top}-\mathcal{P}_{\mathbb{T}}(\mathcal{A}_{t-1})), \text{for } t\geq 1,
$$
where $\mathcal{R}_{\Omega_t}(\mathcal{Z}):= \sum_{i\in [n_1],j\in[n_2],k\in [n_3]}\frac{1}{q} 1_{\Omega_t}(i,j,k)[\mathcal{Z}]_{i,j,k}\mrme_i * \dme_k * \mrme_j^{\top}$. 
Set $\mathcal{Y} := \mathcal{A}_{t_0}$.

Next, we aim to prove that $\mathcal{P}_{\Omega}(\mathcal{Y}) = \mathcal{Y}$ using mathematical induction.
For $t=0,$ $\mathcal{P}_{\Omega}(\mathcal{A}_0) =\mathcal{P}_{\Omega}(\mathbf{0}) = \mathbf{0} = \mathcal{A}_0$.
Notice that  
\begin{align*}
\mathcal{A}_1  &= \mathcal{A}_0 + \mathcal{R}_{\Omega_1}\mathcal{P}_{\mathbb{T}}(\mathcal{W}*\mathcal{V}^{\top}-\mathcal{P}_{\mathbb{T}}(\mathcal{A}_{0})) \\ 
&= \mathcal{A}_0 + \mathcal{R}_{\Omega_1}\mathcal{P}_{\mathbb{T}}(\mathcal{W}*\mathcal{V}^{\top}) \\
& = \mathcal{R}_{\Omega_1}(\mathcal{W} * \mathcal{W}^{\top} * (\mathcal{W}*\mathcal{V}^{\top})+(\mathcal{W}*\mathcal{V}^{\top})* \mathcal{V} * \mathcal{V}^{\top}-\mathcal{W} * \mathcal{W}^{\top} * (\mathcal{W}*\mathcal{V}^{\top}) * \mathcal{V} * \mathcal{V}^{\top}) \\
& = \mathcal{R}_{\Omega_1}(\mathcal{W}*\mathcal{V}^{\top}). 
\end{align*}
Since $\Omega_1 \subseteq \Omega$, it is evident that $  
\mathcal{P}_{\Omega}(\mathcal{A}_1) = \mathcal{P}_{\Omega}(\mathcal{R}_{\Omega_1}(\mathcal{W}*\mathcal{V}^{\top})) = \mathcal{R}_{\Omega_1}(\mathcal{W}*\mathcal{V}^{\top})=\mathcal{A}_1$. 
 Assume that for $k\leq t_{0}-1$, it holds that $\mathcal{P}_{\Omega}(\mathcal{A}_k) = \mathcal{A}_k$. Given the linearity of the operator $\mathcal{P}_{\Omega}$ and the inclusion $\Omega_{t_0}\subseteq \Omega$, it follows that
 \begin{align*}
\mathcal{P}_{\Omega}(\mathcal{Y}) = \mathcal{P}_{\Omega}(\mathcal{A}_{t_0}) 
& = \mathcal{P}_{\Omega}(\mathcal{A}_{{t_0}-1} + \mathcal{R}_{\Omega_{t_0}}\mathcal{P}_{\mathbb{T}}(\mathcal{W}*\mathcal{V}^{\top}-\mathcal{P}_{\mathbb{T}}(\mathcal{A}_{{t_0}-1}))) \\
& = \mathcal{P}_{\Omega}(\mathcal{A}_{{t_0}-1}) +\mathcal{P}_{\Omega}( \mathcal{R}_{\Omega_{t_0}}\mathcal{P}_{\mathbb{T}}(\mathcal{W}*\mathcal{V}^{\top}-\mathcal{P}_{\mathbb{T}}(\mathcal{A}_{{t_0}-1}))) \\
&=\mathcal{A}_{{t_0}-1} + \mathcal{R}_{\Omega_{t_0}}\mathcal{P}_{\mathbb{T}}(\mathcal{W}*\mathcal{V}^{\top}-\mathcal{P}_{\mathbb{T}}(\mathcal{A}_{{t_0}-1})) = \mathcal{A}_{t_0}  = \mathcal{Y}.
\end{align*}
Therefore $\mathcal{Y}=\mathcal{A}_{t_0}$ is the dual certificate.

Now let's prove that $\|\mathcal{P}_{\mathbb{T}}(\mathcal{Y})-\mathcal{W}*\mathcal{V}^{\top}\|_{\fro}\leq \frac{1}{4}\sqrt{\frac{p}{n_3}}$. 
For $t \in \{0\}\cup[t_0]$,  set \begin{equation}\label{Dt}\mathcal{D}_{t} = \mathcal{W}*\mathcal{V}^{\top}-\mathcal{P}_{\mathbb{T}}(\mathcal{A}_t).\end{equation}
Notice that \begin{equation}\label{q}
    q = 1 - (1-p)^{\frac{1}{t_0}}\geq 1 - (1-\frac{p}{t_0})  = \frac{p}{t_0}.
\end{equation}
Thus, one can derive the following results  by \Cref{condtion1}: for each $t$,
\begin{equation}\label{god_inequality}
\left\|\mathcal{D}_t\right\|_{\fro} \leq\left\|\mathcal{P}_{\mathbb{T}}-\mathcal{P}_{\mathbb{T}} \mathcal{R}_{\Omega_t} \mathcal{P}_{\mathbb{T}}\right\|\left\|\mathcal{D}_{t-1}\right\|_{\fro} \leq \frac{1}{2}\left\|\mathcal{D}_{t-1}\right\|_{\fro}
\end{equation} holds with probability at least 
\begin{equation*}
    1-2n_1n_2n_3\exp(-\frac{3qn_1n_2}{28(n_1+n_2)\mu_0r}).
\end{equation*}
Applying   \eqref{god_inequality}  from $t=t_0$ to $t=1$, we will have that
\begin{align}
\|\mathcal{P}_{\mathbb{T}}(\mathcal{Y}-\mathcal{W}*\mathcal{V}^{\top})\|_{\fro}  = \|\mathcal{D}_{t_0}\|_{\fro} \leq \frac{1}{2}\|\mathcal{D}_{t_{0}-1}\|_{\fro} 
 \leq \cdots   
 \leq \frac{1}{2^{t_0}}\|\mathcal{W}*\mathcal{V}^{\top}\|_{\fro}  \leq \frac{\sqrt{r} }{2^{t_0}} \label{trick}
\end{align}
holds with probability  at least 
\begin{equation*}
    1-2t_0n_1n_2n_3\exp(-\frac{3qn_1n_2}{28(n_1+n_2)\mu_0r}).
\end{equation*} 
Since $t_0 \geq \left\lceil\log_{2}\left(4\sqrt{\frac{n_3r}{p}}\right)\right\rceil$, we have 
 \begin{equation}\label{eqn:probability4Fro}
 \mathbb{P}\left(\|\mathcal{P}_{\mathbb{T}}(\mathcal{Y}-\mathcal{W}*\mathcal{V}^{\top})\|_{\fro} \leq \frac{1}{4}\sqrt{\frac{p}{n_3}}\right)\geq    1-2t_0n_1n_2n_3\exp(-\frac{3qn_1n_2}{28(n_1+n_2)\mu_0r}).
\end{equation} 
Next, we proceed to prove that  $\|\mathcal{P}_{\mathbb{T}^{\perp}}(\mathcal{Y})\|\leq \frac{1}{2}$. Recall that $\mathcal{Y} = \sum\limits_{t=1}^{t_0}\mathcal{R}_{\Omega_t}\mathcal{P}_{\mathbb{T}}\mathcal{D}_{t-1}$. By applying \Cref{conditionII1} for $t_0$ times, we can get

\begin{align}
& \left\|\mathcal{P}_{\mathbb{T}^{\perp}} (\mathcal{Y})\right\|  \leq \sum_{t=1}^{t_0}\left\|\mathcal{P}_{\mathbb{T}^{\perp}}\left(\mathcal{R}_{{\Omega}_t} \mathcal{P}_{\mathbb{T}}-\mathcal{P}_{\mathbb{T}}\right) (\mathcal{D}_{t-1})\right\|  
\leq \sum_{t=1}^{t_0}\left\|\left(\mathcal{R}_{{\Omega}_t}-\mathcal{I}\right)\left(\mathcal{P}_{\mathbb{T}} (\mathcal{D}_{t-1})\right)\right\| \notag\\
 \leq & \sum_{t=1}^{t_0}\left(\frac{c_{2}\log((n_1+n_2)n_3)}{q}\|\mathcal{P}_{\mathbb{T}}(\mathcal{D}_{t-1})\|_{\infty}+\sqrt{\frac{2c_{2}\log((n_1+n_2)n_3)}{q}}\left\|\mathcal{P}_{\mathbb{T}} (\mathcal{D}_{t-1})\right\|_{\infty,2}\right)\label{q_inequality}\\
= & \sum_{t=1}^{t_0}\left(\frac{c_{2}\log((n_1+n_2)n_3)}{q}\|\mathcal{D}_{t-1}\|_{\infty}+\sqrt{\frac{2c_{2}\log((n_1+n_2)n_3)}{q}}\left\|\mathcal{D}_{t-1}\right\|_{\infty,2}\right)\label{god_}
\end{align}
holds with probability at least 
\begin{equation}\label{eqn:probability4bound4proj}
    1-\frac{t_0} {((n_1+n_2)n_3)^{c_2-1}},
\end{equation}
where \eqref{q_inequality} holds due to \eqref{q}, and \eqref{god_} holds because $\mathcal{P}_{\mathbb{T}}(\mathcal{D}_{t}) = \mathcal{D}_{t}$.

Next, we will bound  \eqref{god_} by bounding   these two terms:
\begin{enumerate}[label=(\roman*),leftmargin=.5in]
    \item $\sum_{t=1}^{t_0}\frac{c_{2}\log((n_1+n_2)n_3)}{q}\|\mathcal{D}_{t-1}\|_{\infty}$ 
    \item $\sum_{t=1}^{t_0}\sqrt{\frac{2c_{2}\log((n_1+n_2)n_3)}{q}}\left\|\mathcal{D}_{t-1}\right\|_{\infty,2}$
\end{enumerate} via   estimating the upper bounds of $\|\mathcal{D}_{t-1}\|_{\infty}$ and $\|\mathcal{D}_{t-1}\|_{\infty,2}$.

By applying \Cref{conditionII33} for $t-1$ times, where $2\leq t\leq t_{0}$, we  have 
\begin{align*}
 \left\|\mathcal{D}_{t-1}\right\|_{\infty} 
= & ~ \left\|\left(\mathcal{P}_{\mathbb{T}}-\mathcal{P}_{\mathbb{T}} \mathcal{R}_{\Omega_{t-1}} \mathcal{P}_{\mathbb{T}}\right) \cdots\left(\mathcal{P}_{\mathbb{T}}-\mathcal{P}_{\mathbb{T}} \mathcal{R}_{\Omega_1} \mathcal{P}_{\mathbb{T}}\right) \mathcal{D}_0\right\|_{\infty} \leq \frac{\left\|\mathcal{D}_0\right\|_{\infty} }{2^{t-1}}
\end{align*}
holds with probability  at least  $1-2n_{1}n_{2}n_{3}(t-1)\exp\left(\frac{-3qn_{1}n_{2}}{28(n_1+n_2)\mu_{0}r}\right)$.
Therefore, $$ \sum_{t=1}^{t_0}\frac{c_{2}\log((n_1+n_2)n_3)}{q}\|\mathcal{D}_{t-1}\|_{\infty}\leq \frac{2c_{2}\log((n_1+n_2)n_3)}{q}\cdot  \left\|\mathcal{D}_0\right\|_{\infty}$$
holds with probability at least 
\begin{equation}\label{eqn:probablity4boundinfty}
    1-2n_{1}n_{2}n_{3}(t_0-1)\exp\left(\frac{-3qn_{1}n_{2}}{28(n_1+n_2)\mu_{0}r}\right).
\end{equation}
Now we are going to estimate the upper bound for $\sum_{t=1}^{t_0}\sqrt{\frac{2c_{2}\log((n_1+n_2)n_3)}{q}}\left\|\mathcal{D}_{t-1}\right\|_{\infty,2}$. 
by bounding $\left\|\mathcal{D}_{t-1}\right\|_{\infty,2}$. For simplicity, we set
\[a = 2\cdot\sqrt{\frac{c_1\log((n_1+n_2)n_3)}{q}} \sqrt{\frac{(n_1+n_2)\mu_{0}r}{n_{1}n_{2}}}
\text{ and } b = \frac{c_{1}(\log((n_1+n_2)n_3))}{q} \sqrt{\frac{(n_1+n_2)\mu_{0}r}{n_{1}n_{2}}}.\]  
By applying \Cref{conditionII2} for $t-1$ times and considering the fact that $\mathcal{P}_{\mathbb{T}}(\mathcal{D}_{s})=\mathcal{D}_s$ for all $0\leq s\leq t_0$, we obtain that
\begin{align*}
  \left\|\mathcal{D}_{t-1}\right\|_{\infty, 2}  
=  & \left\|\left(\mathcal{P}_{\mathbb{T}}-\mathcal{P}_{\mathbb{T}} \mathcal{R}_{\Omega_{t-1}}\mathcal{P}_{\mathbb{T}}\right)\left(\mathcal{D}_{t-2}\right)\right\|_{\infty, 2}  \\
\leq&    a\|\mathcal{D}_{t-2}\|_{\infty,2}+b\|\mathcal{D}_{t-2}\|_{\infty}  
\leq   \cdots\leq a^{t-1}\|\mathcal{D}_0\|_{\infty,2} + b\sum\limits_{i=0}^{t-2}a^{i}\|\mathcal{D}_{t-2-i}\|_{\infty}
\end{align*}
holds with probability at least  $1-\frac{t-1}{(n_1n_3+n_2n_3)^{c_1-2}}$   for $2\leq t\leq t_0$. 
Therefore, 
\begin{align*}
&\sum_{t=1}^{t_0}\sqrt{\frac{2c_{2}\log((n_1+n_2)n_3)}{q}}\left\|\mathcal{D}_{t-1}\right\|_{\infty,2} \\
\leq&\sqrt{\frac{2c_{2}\log((n_1+n_2)n_3)}{q}}\left(\left(\sum_{t=1}^{t_0}a^{t-1}\|\mathcal{D}_0\|_{\infty,2} \right)+ \sum_{t=2}^{t_0} b\sum\limits_{i=0}^{t-2}a^{i}\|\mathcal{D}_{t-2-i}\|_{\infty}\right)  \\
    =& \sqrt{\frac{2c_{2}\log((n_1+n_2)n_3)}{q}} \cdot \left(\|\mathcal{D}_0\|_{\infty,2} \frac{1-a^{t_0}}{1-a}  +  b\cdot  \sum_{t=2}^{t_0}\sum_{i=0}^{t-2}a^i\|\mathcal{D}_{t-2-i}\|_{\infty}\right) \\
\end{align*}
holds with probability at least   
\begin{equation}\label{eqn:probability42infty}
    1-\frac{t_0-1}{(n_1n_3+n_2n_3)^{c_1-2}}.
\end{equation}
Considering the process of estimating the upper bound for $\|\mathcal{D}_{t}\|_{\infty}$, we  have
\begin{align*}
    \eqref{god_}\leq& \frac{2c_{2}\log((n_1+n_2)n_3)}{q}\cdot\|\mathcal{D}_0\|_{\infty} 
    +\sqrt{\frac{2c_{2}\log((n_1+n_2)n_3)}{q}} \cdot  \frac{\|\mathcal{D}_0\|_{\infty,2}}{1-a} \notag \\
    &+\sqrt{\frac{2c_{2}\log((n_1+n_2)n_3)}{q}}\cdot b\cdot \sum_{t=2}^{t_0}\sum_{i=0}^{t-2}a^i\left(\frac{1}{2}\right)^{t-2-i}\|\mathcal{D}_0\|_{\infty} \\
    \leq& \frac{2c_{2}\log((n_1+n_2)n_3)}{q}\cdot 
 \|\mathcal{D}_0\|_{\infty} 
    +\sqrt{\frac{2c_{2}\log((n_1+n_2)n_3)}{q}}   \cdot\frac{\|\mathcal{D}_0\|_{\infty,2}}{1-a} \\
    &+\sqrt{\frac{2c_{2}\log((n_1+n_2)n_3)}{q}}  \cdot \frac{2b}{1-2a}\cdot \|\mathcal{D}_0\|_{\infty} 
\end{align*}
holds with probability at least $ 1-  \frac{t_0-1}{(n_1n_3+n_2n_3)^{c_1-2}}-2(t_0-1)n_{1}n_{2}n_{3}\exp\left(\frac{-3qn_{1}n_{2}}{28(n_1+n_2)\mu_{0}r}\right)$ when  
$0<a\leq \frac{1}{4}$. Therefore,
\begin{align*}
\|\mathcal{P}_{\mathbb{T}^{\perp}}\left(\mathcal{Y}\right)\|\leq& \sum\limits_{t=1}^{t_0}\left(\frac{c_{2}\log((n_1+n_2)n_3)}{q}\|\mathcal{D}_{t-1}\|_{\infty}+\sqrt{\frac{2c_{2}\log((n_1+n_2)n_3)}{q}}\left\|\mathcal{D}_{t-1}\right\|_{\infty,2}\right)\\
\leq& \frac{2c_{2}\log((n_1+n_2)n_3)}{q}\cdot 
 \|\mathcal{D}_0\|_{\infty} 
    +\sqrt{\frac{2c_{2}\log((n_1+n_2)n_3)}{q}}   \cdot\frac{\|\mathcal{D}_0\|_{\infty,2}}{1-a} \\
    &+\sqrt{\frac{2c_{2}\log((n_1+n_2)n_3)}{q}}  \cdot \frac{2b}{1-2a}\cdot \|\mathcal{D}_0\|_{\infty} 
\end{align*}
holds with probability at least 
\[ 1-  \frac{t_0-1}{(n_1n_3+n_2n_3)^{c_1-2}}-2(t_0-1)n_{1}n_{2}n_{3}\exp\left(\frac{-3qn_{1}n_{2}}{28(n_1+n_2)\mu_{0}r}\right)-\frac{t_0} {(n_1n_3+n_2n_3)^{c_2-1}}\] provided that  
$0<a\leq \frac{1}{4}$. 
Note that $\|\mathcal{D}_0\|_{\infty}=\|\mathcal{W}*\mathcal{V}^\top\|_{\infty}\leq \frac{(n_1+n_2)\mu_0r}{2n_1n_2}$ and 
$\|\mathcal{D}_0\|_{\infty,2}=\|\mathcal{W}*\mathcal{V}^\top\|_{\infty,2}\leq\sqrt{\frac{(n_1+n_2)\mu_0r}{n_1n_2}}$. Combining \eqref{eqn:probability4bound4proj}, \eqref{eqn:probablity4boundinfty}, and \eqref{eqn:probability42infty}, we thus have
\begin{align*}
\|\mathcal{P}_{\mathbb{T}^{\perp}}\left(\mathcal{Y}\right)\|
\leq&  
 \frac{c_{2}\log((n_1+n_2)n_3)}{q}\cdot \frac{(n_1+n_2)\mu_{0}r}{n_{1}n_{2}}
    +\sqrt{\frac{2c_{2}\log((n_1+n_2)n_3)}{q}}   \cdot\frac{\sqrt{(n_1+n_2)\mu_0r}}{(1-a)\sqrt{n_1n_2}} \\
    &+\sqrt{\frac{2c_{2}\log((n_1+n_2)n_3)}{q}}  \cdot  \frac{b}{1-2a}\cdot \frac{ (n_1+n_2)\mu_{0}r}{n_{1}n_{2}}\\
    \leq &\frac{c_{2}\log((n_1+n_2)n_3)}{q}\cdot \frac{(n_1+n_2)\mu_{0}r}{n_{1}n_{2}}
    +\frac{4\sqrt{2c_2}}{3}\cdot\sqrt{\frac{\log((n_1+n_2)n_3)}{q}  \cdot\frac{ (n_1+n_2)\mu_0r}{ n_1n_2}} \\
    &+2c_1\sqrt{2c_2}\left(\frac{\log((n_1+n_2)n_3)}{q}     \cdot \frac{ (n_1+n_2)\mu_{0}r}{n_{1}n_{2}}\right)^{3/2}
\end{align*}
holds with probability at least 
\[ 1-  \frac{t_0-1}{(n_1n_3+n_2n_3)^{c_1-2}}-2(t_0-1)n_{1}n_{2}n_{3}\exp\left(\frac{-3qn_{1}n_{2}}{28(n_1+n_2)\mu_{0}r}\right)-\frac{t_0} {(n_1n_3+n_2n_3)^{c_2-1}}\] provided that  
 {$0<a\leq \frac{1}{4}$}.

\noindent Since $a = 2\cdot\sqrt{\frac{c_1\log((n_1+n_2)n_3)}{q}} \sqrt{\frac{(n_1+n_2)\mu_{0}r}{n_{1}n_{2}}}$, the restriction  $0<a\leq \frac1 4$ is equivalent to \[q\geq 64c_1\log((n_1+n_2)n_3)\cdot\frac{(n_1+n_2)\mu_0r}{n_1n_2}.\]
Notice that
 \[
p\geq \frac{256(n_1+n_2)\mu_0\beta r\log^2((n_1+n_2)n_3)}{n_1n_2}.
\]
 We thus have
\begin{align*}
    t_0=&\lceil \log_2\left(4\sqrt{\frac{n_3r}{p}}\right) \rceil\\
    \leq&\left\lceil \log_2\left(4\sqrt{\frac{n_1n_2n_3}{256(n_1+n_2)\mu_0\beta\log^2((n_1+n_2)n_3)}}\right) \right\rceil\\
    \leq &\left\lceil \frac{1}{2}\log_2\left(  {\frac{n_1n_2n_3}{  (n_1+n_2)\mu_0\beta }}\right)-2 \right\rceil\leq  \log((n_1+n_2)n_3)
\end{align*}

 In addition   $q\geq   \frac{p}{t_0}$, we have
\begin{align*}
    q\geq  & \frac{256(n_1+n_2)\mu_0\beta r\log^2((n_1+n_2)n_3)}{n_1n_2}\cdot\frac{1}{\log\left(    (n_1+n_2)n_3\right)}\\
=&\frac{256(n_1+n_2)\mu_0\beta r\log((n_1+n_2)n_3)}{n_1n_2},
\end{align*}
i.e., $\frac{(n_1+n_2)\mu_0r\log((n_1+n_2)n_3)}{qn_1n_2}\leq \frac{1}{256\beta}$. \\
Therefore the condition that  $q\geq 64c_1\log((n_1+n_2)n_3)\frac{(n_1+n_2)\mu_0r}{n_1n_2}$ holds when $c_1=4\beta$. Hence, the condition   $a\leq \frac 1 4$ holds. In addition, by setting $c_2=12\beta$, we have
\begin{align*}
\|\mathcal{P}_{\mathbb{T}^{\perp}}\left(\mathcal{Y}\right)\|
    \leq &\frac{c_{2}\log((n_1+n_2)n_3)}{q}\cdot \frac{(n_1+n_2)\mu_{0}r}{n_{1}n_{2}}
    +\frac{4\sqrt{2c_2}}{3}\cdot\sqrt{\frac{\log((n_1+n_2)n_3)}{q}  \cdot\frac{ (n_1+n_2)\mu_0r}{ n_1n_2}} \\
&+2c_1\sqrt{2c_2}\left(\frac{\log((n_1+n_2)n_3)}{q}     \cdot \frac{ (n_1+n_2)\mu_{0}r}{n_{1}n_{2}}\right)^{3/2}\\
    \leq&\frac{c_2}{256\beta}+\frac{4\sqrt{2c_2}}{3}\sqrt{\frac{1}{256\beta}}+2c_1\sqrt{2c_2}\left(\frac{1}{256\beta}\right)^{3/2} <\frac{1}{2}.
\end{align*}
with probability at least
\begin{equation}\label{prob:conditionII(b)}
1-  \frac{\log(n_1n_3+n_2n_3)}{(n_1n_3+n_2n_3)^{4\beta-2}}-\frac{\log(n_1n_3+n_2n_3)}{(n_1n_3+n_2n_3)^{27\beta-2}} 
 -\frac{\log(n_1n_3+n_2n_3)} {(n_1n_3+n_2n_3)^{12\beta-1}}
 \geq 1-  \frac{3\log(n_1n_3+n_2n_3)}{(n_1n_3+n_2n_3)^{4\beta-2}}.
 \end{equation}
Notice that the probabilistic estimation for the validity of Condition~\ref{CondII} is based on the assumption that  Condition~\ref{CondI} holds, where we show $\|\mathcal{D}\|_{\infty} \leq (\frac{1}{2})^{t-1}\|\mathcal{D}_0\|_{\infty}$ based on  Condition ~\ref{CondI}. Thus, $\mathcal{Z}$ is the unique minimizer with probability at least \[1-  \frac{3\log(n_1n_3+n_2n_3)}{(n_1n_3+n_2n_3)^{4\beta-2}}.\]
\subsubsection{\textnormal{Proof of supporting lemmas}}\label{proofsoflemmas}
\paragraph{Proof of \Cref{BernsteinInequality2}}
    Substitute $\tau =  \sqrt{2c \sigma^2 \log \left(n_1+n_2\right)}+c M\log \left(n_1+n_2\right)$ into $\frac{-\tau^2/2}{\sigma^2+\frac{M\tau}{3}}$ in \Cref{thm:tropp}. We  get
    \begin{align*}
  \frac{-\tau^2/2}{\sigma^2+\frac{M\tau}{3}} &=-\frac{2c\sigma^{2}\log(n_1+n_2)+2\sqrt{2}c^{\frac{3}{2}}\sigma M\log^{\frac{3}{2}}(n_1+n_2)  +c^{2}M^{2}\log^2(n_1+n_2)}{2\sigma^2+\frac{2\sqrt{2}}{3}c^{\frac{1}{2}} \sigma M  \log^{\frac{1}{2}}\left(n_1+n_2\right)+\frac{2cM^2}{3}\log \left(n_1+n_2\right)}\\
  &\leq  -c\log(n_1+n_2).
    \end{align*}
\hfill\BlackBox\\[2mm]
\vspace{-25pt}
\paragraph{Proof of \Cref{ptbound}}
\begin{align*}
\|\mathcal{P}_{\mathbb{T}}(\mrme_i * \dme_k * \mrme_j^{\top})\|_{\fro}^2   = &~\langle\mathcal{P}_{\mathbb{T}}(\mrme_i * \dme_k * \mrme_j^{\top}),\mathcal{P}_{\mathbb{T}}(\mrme_i * \dme_k * \mrme_j^{\top}) \rangle \\
=&~\langle\mathcal{P}_{\mathbb{T}}(\mrme_i * \dme_k * \mrme_j^{\top}),\mrme_i * \dme_k * \mrme_j^{\top}\rangle \\
=&~ \left\|\mathcal{W}^{\top} * \mathring{\mathfrak{e}_i}\right\|_{\fro}^2+\left\|\mathcal{V}^{\top} * \mathring{\mathfrak{e}_j}\right\|_{\fro}^2-\left\|\mathcal{W}^{\top} * \mrme_i * \dme_k * \mathring{\mathfrak{e}_j}^{\top} * \mathcal{V}\right\|_{\fro}^2, \\
 \leq  &~  \left\|\mathcal{W}^{\top} * \mathring{\mathfrak{e}_i}\right\|_{\fro}^2+\left\|\mathcal{V}^{\top} * \mathring{\mathfrak{e}_j}\right\|_{\fro}^2 
=\frac{(n_1+n_2)\mu_{0}r}{n_{1}n_{2}} 
\end{align*}
\hfill\BlackBox\\[2mm]
\paragraph{Proof of \Cref{condtion1}}
By the fact that  $\mathcal{P}_{\mathbb{T}}$ is a self-adjoint and idempotent operator, we can get that
$\mathbb{E}[\mathcal{P}_{\mathbb{T}}\mathcal{R}_{\Omega}\mathcal{P}_{\mathbb{T}}] = \mathcal{P}_{\mathbb{T}}(\mathbb{E}\mathcal{R}_{\Omega})\mathcal{P}_{\mathbb{T}} = \mathcal{P}_{\mathbb{T}}$.  
It is easy to check that 
\[\mathcal{P}_{\mathbb{T}} \mathcal{R}_{\Omega} \mathcal{P}_{\mathbb{T}}(\mathcal{X})
=  \sum_{i,j,k}\frac{1}{p}\delta_{i,j,k} \left\langle\mathcal{X}, \mathcal{P}_{\mathbb{T}}\left(\mrme_i * \dme_k * \mathring{\mathfrak{e}_j}^{\top}\right)\right\rangle\mathcal{P}_{\mathbb{T}}\left(\mrme_i * \dme_k * \mrme_j^{\top}\right).\]
Fix a tensor $\mathcal{X}\in \mathbb{K}^{n_1\times n_2\times n_3}$, we can write
\begin{align*}
 \left(\mathcal{P}_{\mathbb{T}} \mathcal{R}_{\Omega} \mathcal{P}_{\mathbb{T}}-\mathcal{P}_{\mathbb{T}}\right)(\mathcal{X}) 
&= \sum\limits_{i,j,k}\left( \frac{1}{p}\delta_{i j k}-1\right)\left\langle\mrme_i * \dme_k * \mrme_j^{\top}, \mathcal{P}_{\mathbb{T}}(\mathcal{X})\right\rangle \mathcal{P}_{\mathbb{T}}\left(\mrme_i * \dme_k * \mrme_j^{\top}\right)\\
&=:  \sum\limits_{i,j,k} \mathcal{H}_{i j k}(\mathcal{X})
\end{align*}
where $\mathcal{H}_{i j k}: \mathbb{K}^{n_1 \times n_2 \times n_3} \rightarrow \mathbb{K}^{n_1 \times n_2 \times n_3}$ is a self-adjoint random operator and $\delta_{i,j,k}$ is the indicator function. One can see that $\mathbb{E}\left[\mathcal{H}_{i j k}\right] = 0$ as $\mathbb{E}(\frac{1}{p}\delta_{i,j,k}-1) = \frac{1}{p}\mathbb{E}(\delta_{i,j,k})-1 = 0$.
Define the operator $\overline{\mathcal{H}}_{i j k}: \mathbb{B} \rightarrow \mathbb{B}$ with $\mathbb{B}=\left\{\overline{\mathcal{B}}: \mathcal{B} \in \mathbb{K}^{n_1 \times n_2 \times n_3}\right\}$. As 
\[
\overline{\mathcal{H}}_{i j k}(\mathcal{\overline{X}}):=\overline{\mathcal{H}_{i j k}(\mathcal{X})}= (\frac{1}{p}\delta_{i j k}-1)\left\langle\mrme_i * \dme_k * \mrme_j^{\top}, \mathcal{P}_{\mathbb{T}}(\mathcal{X})\right\rangle \overline{\mathcal{P}_{\mathbb{T}}\left(\mrme_i * \dme_k * \mrme_j^{\top}\right)}.\]
It is easy to check that $\overline{\mathcal{H}}_{i j k}$ is also self-adjoint by using the fact that the operator $\mathcal{P}_{\mathbb{T}}(\cdot)$ is self-adjoint. Using the fact that $\mathbb{E}(\frac{1}{p}\delta_{i,j,k}-1) = 0$ again, we have $\mathbb{E}\left[\overline{\mathcal{H}}_{i j k}\right]=0$. To prove the result using the non-commutative Bernstein inequality, we need to bound $\left\|\overline{\mathcal{H}}_{i j k}\right\|$ and $\left\|\sum\limits_{i,j,k} \mathbb{E}\left[\overline{\mathcal{H}}_{i,j,k}^2\right]\right\|$.
Firstly, we have 
\begin{align*}
 \left\|\overline{\mathcal{H}}_{i j k}\right\| 
= &\sup _{\|\overline{\mathcal{X}}\|_{\fro}=1}\left\|\overline{\mathcal{H}}_{i, j, k}(\overline{\mathcal{X}})\right\|_{\fro} \\
= & \sup _{\|\overline{\mathcal{X}}\|_{\fro}=1} \left\|(\frac{1}{p}\delta_{i j k}-1)\left\langle\mrme_i * \dme_k * \mrme_j^{\top}, \mathcal{P}_{\mathbb{T}}(\mathcal{X})\right\rangle \overline{\mathcal{P}_{\mathbb{T}}\left(\mrme_i * \dme_k * \mrme_j^{\top}\right)}\right\|_{\fro}\\
= & \sup _{\|\overline{\mathcal{X}}\|_{\fro}=1} \left\|(\frac{1}{p}\delta_{i j k}-1)\left\langle\mathcal{P}_{\mathbb{T}}(\mrme_i * \dme_k * \mrme_j^{\top}), \mathcal{X}\right\rangle \overline{\mathcal{P}_{\mathbb{T}}\left(\mrme_i * \dme_k * \mrme_j^{\top}\right)}\right\|_{\fro}\\
\leq & \sup _{\|\overline{\mathcal{X}}\|_{\fro}=1}\frac{1}{p}\left\|\mathcal{P}_{\mathbb{T}}(\mrme_i * \dme_k * \mrme_j^{\top})\right\|_{\fro}\|\mathcal{X}\|_{\fro} \left\|\overline{\mathcal{P}_{\mathbb{T}}(\mrme_i * \dme_k * \mrme_j^{\top})}\right\|_{\fro}\\
=& \sup _{\|\overline{\mathcal{X}}\|_{\fro}=1}\frac{1}{p}\left\|\mathcal{P}_{\mathbb{T}}(\mrme_i * \dme_k * \mrme_j^{\top})\right\|_{\fro}\|\mathcal{X}\|_{\fro} \sqrt{n_3}\left\|\mathcal{P}_{\mathbb{T}}(\mrme_i * \dme_k * \mrme_j^{\top})\right\|_{\fro}\\
=& \sup _{\|\overline{\mathcal{X}}\|_{\fro}=1}\frac{1}{p}\left\|\mathcal{P}_{\mathbb{T}}(\mrme_i * \dme_k * `   \mrme_j^{\top})\right\|_{\fro}\|\mathcal{\overline{X}}\|_{\fro} \left\|\mathcal{P}_{\mathbb{T}}(\mrme_i * \dme_k * \mrme_j^{\top})\right\|_{\fro}\\
= &\frac{1}{p}\left\|\mathcal{P}_{\mathbb{T}}(\mrme_i * \dme_k * \mrme_j^{\top})\right\|_{\fro}^2 
\leq   \frac{\mu_{0}(n_1+n_{2})r}{n_{1}n_{2}p}.
\end{align*}
Next, we move on to bound $\left\|\sum\limits_{i,j,k} \mathbb{E}\left[\overline{\mathcal{H}}_{i,j,k}^2\right]\right\|$.
By using the fact that $\mathcal{P}_{\mathbb{T}}$ is a self-adjoint and an idempotent operator, we can get that 
\[
\overline{\mathcal{H}}_{i,j,k}^2(\overline{\mathcal{X}})=\left(\frac{1}{p}\delta_{i j k}-1 \right)^2 \langle\mrme_i * \dme_k * \mrme_j^{\top}, \mathcal{P}_{\mathbb{T}}(\mathcal{X})\rangle
\langle\mathfrak{e}_i * \dme_k * \mrme_j^{\top}, \mathcal{P}_{\mathbb{T}}(\mrme_i * \dme_k * \mrme_j^{\top})\rangle
\overline{\mathcal{P}_{\mathbb{T}}
\left(\mathfrak{e}_i * \dme_k * \mrme_j^{\top}\right)}.\]
Note that $\mathbb{E}\left[\left( \frac{1}{p}\delta_{i j k}-1\right)^2\right] =\frac{1-p}{p}\leq \frac{1}{p}$.
Notice that
\begin{align*}
&~\left\|\sum_{i,j,k} \mathbb{E}\left[\overline{\mathcal{H}}_{ijk}^2(\overline{\mathcal{X}})\right]\right\|_{\fro} \\
\leq&~ \frac{1}{p} \left\|\sum_{i,j,k}\langle\mrme_{i} * \dme_{k} * \mrme_j^{\top}, \mathcal{P}_{\mathbb{T}}(\mathcal{X})\rangle \langle\mrme_{i} * \dme_{k} * \mrme_{j}^{\top}, \mathcal{P}_{\mathbb{T}}\left(\mrme_{i} * \dme_k * \mrme_{j}^{\top}\right)\rangle
\overline{\mathcal{P}_{\mathbb{T}}\left(\mrme_i * \dme_k * \mrme_j^{\top}\right)}\right\|_{\fro} \\ 
=&~\frac{\sqrt{n_3}}{p}\left\|\sum\limits_{i,j,k}\left\langle\mrme_{i} * \dme_{k} * \mrme_j^{\top}, \mathcal{P}_{\mathbb{T}}(\mathcal{X})\right\rangle\left\langle\mrme_{i} * \dme_{k} * \mrme_{j}^{\top}, \mathcal{P}_{\mathbb{T}}\left(\mrme_{i} * \dme_k * \mrme_{j}^{\top}\right)\right\rangle\mathcal{P}_{\mathbb{T}}\left(\mrme_{i} * \dme_k * \mrme_{j}^{\top}\right)\right\|_{\fro}\\
\leq&~\frac{\sqrt{n_3}}{p}\left\|\sum\limits_{i,j,k}\left\langle\mrme_{i} * \dme_{k} * \mrme_j^{\top}, \mathcal{P}_{\mathbb{T}}(\mathcal{X})\right\rangle\left\langle\mrme_{i} * \dme_{k} * \mrme_{j}^{\top}, \mathcal{P}_{\mathbb{T}}\left(\mrme_{i} * \dme_k * \mrme_{j}^{\top}\right)\right\rangle\cdot \left(\mrme_{i} * \dme_k * \mrme_{j}^{\top}\right)\right\|_{\fro}\\
\leq&~ \frac{\sqrt{n_3}}{p}\cdot \max\limits_{i,j,k}\left\{\left\langle\mrme_{i} * \dme_{k} * \mrme_{j}^{\top}, \mathcal{P}_{\mathbb{T}}\left(\mrme_{i} * \dme_k * \mrme_{j}^{\top}\right)\right\rangle\right\}\left\|\sum\limits_{i,j,k}\left\langle\mrme_{i} * \dme_{k} * \mrme_j^{\top}, \mathcal{P}_{\mathbb{T}}(\mathcal{X})\right\rangle\cdot \left(\mrme_{i} * \dme_k * \mrme_{j}^{\top}\right)\right\|_{\fro}\\
=&\frac{\sqrt{n_3}}{p}\cdot(\max\limits_{i,j,k}\left\|\mathcal{P}_{\mathbb{T}}\left(\mrme_{i} * \dme_k * \mrme_{j}^{\top}\right)\right\|_{\fro}^2)\cdot \left\|\mathcal{P}_{\mathbb{T}}(\mathcal{X})\right\|_{\fro}\\
\leq& \frac{\sqrt{n_3} (n_1 +n_2) \mu_{0} r}{ pn_{1}n_{2}}\|\mathcal{P}_{\mathbb{T}}(\mathcal{X})\|_{\fro} \text{ (By \Cref{ptbound})} \\
\leq&\frac{\sqrt{n_3}(n_1+n_2)\mu_{0} r}{ pn_{1}n_{2}}\|\mathcal{X}\|_{\fro}=\frac{(n_1+n_2)\mu_{0} r}{ pn_{1}n_{2}}\|\overline{\mathcal{X}}\|_{\fro}.
\end{align*}
We thus have $\|\sum\limits_{i, j, k} \mathbb{E}\left[\overline{\mathcal{H}}_{i j k}^2\right]\|$ is bounded above by $\frac{(n_1+n_2)\mu_{0}r}{pn_{1}n_{2}}$.
Thus, we use the non-commutative Bernstein inequality to derive the following result: 
\begin{align*}
\mathbb{P}\left[\left\|\mathcal{P}_{\mathbb{T}} \mathcal{P}_{\Omega} \mathcal{P}_{\mathbb{T}}-\mathcal{P}_{\mathbb{T}}\right\|\geq\frac{1}{2}\right] 
= &  \mathbb{P}\left[\left\|\sum\limits_{i,j,k} \mathcal{H}_{i j k}\right\|\geq \frac{1}{2}\right] \\
= & \mathbb{P}\left[\left\|\sum\limits_{i,j,k} \overline{\mathcal{{H}}}_{i j k}\right\|
\geq \frac 1 2\right] \\
\leq&   2n_1n_2n_3  \exp\left(-\frac{3pn_1n_2}{28(n_1+n_2)\mu_0 r}\right).
\end{align*}.
\hfill\BlackBox

\paragraph{Proof of \Cref{conditionII1}}
It is easy to check that 
\begin{align*}
\mathcal{R}_{\Omega}(\mathcal{X})-\mathcal{X}& = \sum\limits_{i,j,k}\left(\frac{1}{p}\delta_{i,j,k}-1\right)[\mathcal{X}]_{i,j,k}\mrme_i * \dme_k * \mrme_j^{\top} =:\sum\limits_{i,j,k}\mathcal{E}_{i,j,k}.
\end{align*}
Notice that $\mathbb{E}[\overline{\mathcal{E}_{i,j,k}}]= 0$ and $\|\overline{\mathcal{E}_{i,j,k}}\|\leq \frac{1}{p}\|\mathcal{X}\|_{\infty}$. In order to use the non-commutative Bernstein inequality, we just need to check the uniform boundness of $\|\mathbb{E}(\sum\limits_{i,j,k}(\overline{\mathcal{E}_{i,j,k}})^{\top}\overline{\mathcal{E}_{i,j,k}})\|$ and $\|\mathbb{E}(\sum\limits_{i,j,k}\overline{\mathcal{E}_{i,j,k}}(\overline{\mathcal{E}_{i,j,k}})^{\top})\|$.
Using the fact that $\mathring{\mathfrak{e}_{k}}^{\top}*\mathring{\mathfrak{e}_{k}} = \dme_{k}^{\top}*\dme_{k}=\dme_{1}$,$\dme_{1}*\dme_{k} = \dme_{k},$ and $\mrme_{j}*\dme_{1} = \mrme_{j}$, we have the following result:
 \begin{align*}
\mathcal{E}_{i,j,k}^{\top}*\mathcal{E}_{i,j,k}
=&~ \left(\frac{1}{p}\delta_{i,j,k} -1\right)^2[\mathcal{X}]_{i,j,k}^2\left(\mrme_i*\dme_k*\mrme_j^{\top}\right)^{\top}*\left(\mrme_i*\dme_k*\mrme_j^{\top}\right) \\
=&~\left(\frac{1}{p}\delta_{i,j,k} -1\right)^2[\mathcal{X}]_{i,j,k}^2\left(\mrme_j*\dme_k^{\top}*\mrme_i^{\top}\right)*\left(\mrme_i*\dme_k*\mrme_j^{\top}\right) \\
=&~\left(\frac{1}{p}\delta_{i,j,k} -1\right)^2[\mathcal{X}]_{i,j,k}^2\mrme_j*\dme_k^{\top}*\left(\mrme_i^{\top}*\mrme_{i}\right)*\dme_k*\mrme_j^{\top} \\
=&~\left(\frac{1}{p}\delta_{i,j,k} -1\right)^2[\mathcal{X}]_{i,j,k}^2\mrme_j*\dme_k^{\top}*\left(\dme_{1}*\dme_k \right)*\mrme_j^{\top} \\
=&~\left(\frac{1}{p}\delta_{i,j,k} -1\right)^2[\mathcal{X}]_{i,j,k}^2\mrme_j*\left(\dme_{k}^{\top}*\dme_k \right)*\mrme_j^{\top}  \\
=&~\left(\frac{1}{p}\delta_{i,j,k} -1\right)^2[\mathcal{X}]_{i,j,k}^2\mrme_j*\left(\dme_
{1}\right)*\mrme_j^{\top}
=\left(\frac{1}{p}\delta_{i,j,k} -1\right)^2[\mathcal{X}]_{i,j,k}^2\mrme_j*\mrme_j^{\top}.
\end{align*}
Notice that $\mrme_j*\mrme_j^{\top}$ returns a zero tensor except for $(j,j,1)$-th entry being $1$. We have
\begin{align*}
\left\|\sum_{i,j,k} \mathbb{E}[\overline{\mathcal{E}}_{i,j,k}^{\top}\overline{\mathcal{E}}_{i,j,k}]\right\| 
=& \left\|\sum_{i,j,k} \mathbb{E}[\mathcal{E}_{i,j,k}^{\top}*\mathcal{E}_{i,j,k}]\right\| 
\leq \frac{1}{p}\left\|\sum\limits_{i,j,k}[\mathcal{X}]_{i,j,k}^2\mrme_j*\mrme_j^{\top} \right\| \\
=&\frac{1}{p}\left\|\sum\limits_{i,j,k}[\mathcal{X}]_{i,j,k}^2\overline{\mrme}_j\overline{\mrme}_j^{\top} \right\|  \text{(by the definition of spectral norm of tensor)} \\
= & \frac{1}{p} \left\| \left( \begin{array}{cccccc}
\sum\limits_{i,k}[\mathcal{X}]_{i,1,k}^2 &  &  &  & & \\
 & \sum\limits_{i,k}[\mathcal{X}]_{i,2,k}^2  &  &  & & \\
 &  & \ddots & &  &\\
 &  &  &  & \sum\limits_{i,k}[\mathcal{X}]_{i,n_2,k}^2  &\\
 &  &  &  & & O_{n_2(n_3-1)\times n_2(n_{3}-1)}
\end{array} \right) \right\| \\
 = & \frac{1}{p} \max\limits_{j\in [n_3]} \left\{\sum_{i,k} [\mathcal{X}]_{i,j,k}^2\right\} = \frac{1}{p} \max\limits_{j\in [n_3]} \left \|[\mathcal{X}]_{:,j,:}\right\|_{\fro}^2.
\end{align*}

Similarly, we can get $\left\|\sum_{i,j,k} \mathbb{E}[\overline{\mathcal{E}}_{i,j,k}\overline{\mathcal{E}}_{i,j,k}^{\top}]\right\| \leq \frac{1}{p} \max\limits_{i\in [n_3]}\left\|[\mathcal{X}]_{i,:,:}\right\|_{\fro}^2$. Thus, $$\max\left\{\mathbb{E}(\sum\limits_{i,j,k}(\overline{\mathcal{E}_{i,j,k}})^{\top}\overline{\mathcal{E}_{i,j,k}}), \mathbb{E}(\sum\limits_{i,j,k}\overline{\mathcal{E}_{i,j,k}}(\overline{\mathcal{E}_{i,j,k}})^{\top})\right\} \leq \frac{1}{p}\left\|\mathcal{X}\right\|_{\infty,2}^2.$$
By \Cref{BernsteinInequality2}, for any $c>1$, 
\begin{align*}
    \left\|\mathcal{R}_{\Omega}(\mathcal{X})-\mathcal{X}\right\| &=\left\|\overline{\mathcal{R}_{\Omega}(\mathcal{X})}-\overline{\mathcal{X}}\right\| =\left\|\sum\limits_{i,j,k}\overline{\mathcal{E}_{i,j,k}}\right\| \\&\leq \sqrt{\frac{2c_2}{p}\|\mathcal{X}\|_{\infty,2}^2\log((n_1+n_2)n_3)}+\frac{c_2\log((n_1+n_2)n_3)}{p}\|\mathcal{X}\|_{\infty}
\end{align*}
holds with probability at least $1-((n_1+n_2)n_3)^{1-c_2}$.
\hfill\BlackBox\\[2mm]
\vspace{-30pt}
\paragraph{Proof of \Cref{conditionII2}}
Consider any $b$-th lateral column of $\mathcal{P}_{\mathbb{T}}\mathcal{R}_{\Omega}(\mathcal{X})-\mathcal{P}_{\mathbb{T}}(\mathcal{X})$:
\begin{align*}
\left(\mathcal{P}_{\mathbb{T}}\mathcal{R}_{\Omega}(\mathcal{X})-\mathcal{P}_{\mathbb{T}}(\mathcal{X})\right)*\mrme_{b} =\sum\limits_{i,j,k}(\frac{1}{p}\delta_{i,j,k}-1)[\mathcal{X}]_{i,j,k}\mathcal{P}_{\mathbb{T}}(\mrme_i*\dme_k*\mrme_j^{\top})*\mrme_b  =: \sum\limits_{i,j,k}\mathfrak{a}_{i,j,k},
\end{align*}
where $\mathfrak{a}_{i,j,k}\in \mathbb{K}^{n_1\times 1 \times n_3}$ are zero-mean independent lateral tensor columns. Denote $\vec{\mathfrak{a}}_{i,j,k}\in \mathbb{K}^{n_{1}n_{3}\times 1}$ as the vectorized column vector of $\mathfrak{a}_{i,j,k}$. Then, we have 
\begin{align*}
    \left\|\vec{\mathfrak{a}}_{i,j,k}\right\|  = \left\|\mathfrak{a}_{i,j,k}\right\|_{\fro} 
    &\leq \frac{1}{p}|[\mathcal{X}]_{i,j,k}|\left\|\mathcal{P}_{\mathbb{T}}(\mrme_i*\dme_k*\mrme_j^{\top})*\mrme_b\right\|_{\fro}\leq \frac{1}{p}\sqrt{\frac{(n_1+n_2)\mu_{0}r}{n_{1}n_{2}}}\|\mathcal{X}\|_{\infty}.
\end{align*}
We also have 
\begin{align*}
  \left|\mathbb{E}(\sum\limits_{i,j,k}\vec{\mathfrak{a}}_{i,j,k}^{\top}\vec{\mathfrak{a}}_{i,j,k})\right|  = \mathbb{E}(\sum\limits_{i,j,k}\left\|\mathfrak{a}_{i,j,k}\right\|_{\fro}^2)  = \frac{1-p}{p}\sum\limits_{i,j,k}[\mathcal{X}]_{i,j,k}^2\left\|\mathcal{P}_{\mathbb{T}}(\mrme_i*\dme_k*\mrme_j^{\top})*\mrme_b\right\|_{\fro}^2.
\end{align*}
By the definition of $\mathcal{P}_{\mathbb{T}}$ and the tensor $\mu_0$-incoherence condition, we have:
\begin{align*}
&~\left\|\mathcal{P}_{\mathbb{T}}(\mrme_i*\dme_k*\mrme_j^{\top})*\mrme_b)\right\|_{\fro} \\
=&~\left\|(\mathcal{W}*\mathcal{W}^{\top}*\mrme_i*\dme_k)*\mrme_j^{\top}*\mrme_b+(\mathcal{I}_{n_1}-\mathcal{W}*\mathcal{W}^{\top})*\mrme_i*\dme_k*\mrme_j^{\top}*\mathcal{V}*\mathcal{V}^{\top}*\mrme_{b}\right\|_{\fro} \\
\leq&~ \sqrt{\frac{\mu_0 r}{n_1}}\left\|\mrme_j^{\top} * \mrme_b\right\|_{\fro}+\left\|(\mathcal{I}_{n_1}-\mathcal{W}*\mathcal{W}^{\top})*\mrme_i*\dme_k\right\|\left\|\mrme_j^{\top} * \mathcal{V} * \mathcal{V}^{\top} * \mrme_b\right\|_{\fro} \\
\leq&~  \sqrt{\frac{\mu_0 r}{n_1}}\left\|\mrme_j^{\top} * \mrme_b\right\|_{\fro}+\left\|\mrme_j^{\top} * \mathcal{V} * \mathcal{V}^{\top} * \mrme_b\right\|_{\fro} 
\end{align*}
By Cauchy-Schwartz inequality, we have 
$\left\|\mathcal{P}_{\mathbb{T}}(\mrme_i*\dme_k*\mrme_j^{\top})*\mrme_b)\right\|_{\fro}^2 \leq \frac{2\mu_{0}r}{n_1}\left\|\mrme_j^{\top} * \mrme_b\right\|_{\fro}^2+2\left\|\mrme_j^{\top} * \mathcal{V} * \mathcal{V}^{\top} * \mrme_b\right\|_{\fro}^2$. 
  Thus,
\begin{align*}
\left|\mathbb{E}(\sum\limits_{i,j,k}\vec{\mathfrak{a}}_{i,j,k}^{\top}\vec{\mathfrak{a}}_{i,j,k})\right|
\leq&~ \frac{2\mu_0r}{pn_{1}}\sum\limits_{i,j,k}[\mathcal{X}]_{i,j,k}^2\left\|\mrme_{j}^{\top}*\mathring{\mathfrak{e}_b}\right\|_{\fro}^2 + \frac{2}{p}\sum\limits_{i,j,k}[\mathcal{X}]_{i,j,k}^2\left\|\mrme_{j}^{\top}*\mathcal{V}*\mathcal{V}^{\top}*\mathring{\mathfrak{e}_b}\right\|_{\fro}^2   \\
=&~\frac{2\mu_0r}{pn_{1}}\sum\limits_{i,j,k}[\mathcal{X}]_{i,j,k}^2\left\|\mrme_{j}^{\top}*\mathring{\mathfrak{e}_b}\right\|_{\fro}^2 + \frac{2}{p}\sum_{j}\left\|\mrme_{j}^{\top}*\mathcal{V}*\mathcal{V}^{\top}*\mathring{\mathfrak{e}_b}\right\|_{\fro}^2\sum_{i,k}[\mathcal{X}]_{i,j,k}^2 \\
\leq &~\frac{2\mu_0r}{pn_{1}}\sum\limits_{i,j,k}[\mathcal{X}]_{i,j,k}^2\left\|\mrme_{j}^{\top}*\mathring{\mathfrak{e}_b}\right\|_{\fro}^2 + \frac{2}{p}\sum_{j}\left\|\mrme_{j}^{\top}*\mathcal{V}*\mathcal{V}^{\top}*\mathring{\mathfrak{e}_b}\right\|_{\fro}^2\|\mathcal{X}\|_{\infty,2}^2\\
=&~\frac{2\mu_{0}r}{pn_{1}}\sum\limits_{i,k}[\mathcal{X}]_{i,b,k}^2 + \frac{2}{p}\left\|\mathcal{V}*\mathcal{V}^{\top}*\mathring{\mathfrak{e}_b}\right\|_{\fro}^2 \|\mathcal{X}\|_{\infty,2}^2\\
\leq&~ \frac{2\mu_{0}r}{pn_{1}}\|\mathcal{X}\|_{\infty,2}^2 +\frac{2\mu_{0}r}{pn_{2}}\|\mathcal{X}\|_{\infty,2}^2 
\leq \frac{2(n_1+n_2)\mu_{0}r}{pn_{1}n_{2}}\|\mathcal{X}\|_{\infty,2}^2
\end{align*}
Similarly, we can bound $\left|\mathbb{E}(\sum\limits_{i,j,k}\vec{\mathfrak{a}}_{i,j,k}\vec{\mathfrak{a}}_{i,j,k}^{\top})\right|$ by the same quantity.
\\For simplicity, we let $M = \frac{1}{p}\sqrt{\frac{(n_1+n_2)\mu_{0}r}{n_{1}n_{2}}}\|\mathcal{X}\|_{\infty}$ and $\sigma^2 = \frac{2(n_1+n_2)\mu_{0}r}{pn_{1}n_{2}}\|\mathcal{X}\|_{\infty,2}^2$.
By \Cref{BernsteinInequality2}, for any $c_1 > 1$, we have 
\begin{align*}
&\mathbb{P}\left(\left\|\sum\limits_{i,j,k}\vec{\mathfrak{a}}_{i,j,k}\right\|\leq \sqrt{2c_1\sigma^2\log((n_{1}+n_{2})n_3)}+c_{1}M\log((n_{1}+n_{2})n_3)\right) \\
=&~\mathbb{P}\left(\left\|\sum\limits_{i,j,k}\vec{\mathfrak{a}}_{i,j,k}\right\|\leq \sqrt{\frac{4c_1\log\left((n_1+n_2)n_3\right)(n_{1}+n_{2})\mu_{0}r}{pn_{1}n_{2}}}\cdot\|\mathcal{X}\|_{\infty,2} \right.\\
&~\left.+  \frac{c_1\log((n_1+n_2)n_3)\sqrt{(n_1+n_2)\mu_{0}r}}{p\sqrt{n_{1}n_{2}}}\|\mathcal{X}\|_{\infty}\right)\geq 1-((n_1+n_2)n_3)^{1-c_1}. 
\end{align*}
Notice that 
\begin{align*}
&\left\|(\mathcal{P}_{\mathbb{T}}\mathcal{R}_{\Omega}(\mathcal{X})-\mathcal{P}_{\mathbb{T}}(\mathcal{X}))*\mrme_{b}\right\|_{\fro}=\left\|\sum\limits_{i,j,k}\mathfrak{a}_{i,j,k}\right\|_{\fro}=\left\|\sum\limits_{i,j,k}\vec{\mathfrak{a}}_{i,j,k}\right\|.
\end{align*}
Therefore, 
\begin{align*}
&\mathbb{P}\left(\left\|(\mathcal{P}_{\mathbb{T}}\mathcal{R}_{\Omega}(\mathcal{X})-\mathcal{P}_{\mathbb{T}}(\mathcal{X}))*\mrme_{b}\right\|_{\fro}\leq\sqrt{\frac{4c_1(n_{1}+n_{2})\log\left((n_1+n_2)n_3\right)\mu_{0}r}{pn_{1}n_{2}}}\cdot\|\mathcal{X}\|_{\infty,2} + \right.\\
&~ \left.\frac{c_1\log((n_1+n_2)n_3)\sqrt{(n_1+n_2)\mu_{0}r}}{p\sqrt{n_{1}n_{2}}}\|\mathcal{X}\|_{\infty}\right) \geq 1-((n_1+n_2)n_3)^{1-c_1}.
\end{align*}
Using a union bound over all the tensor lateral slices, we have 
\begin{align*}
&\mathbb{P}\left(\max\limits_{b}\left\{\left\|(\mathcal{P}_{\mathbb{T}}\mathcal{R}_{\Omega}(\mathcal{X})-\mathcal{P}_{\mathbb{T}}(\mathcal{X}))*\mrme_{b}\right\|_{\fro}\right\}\leq\sqrt{\frac{4c_1(n_{1}+n_{2})\log\left((n_1+n_2)n_3\right)\mu_{0}r}{pn_{1}n_{2}}}\cdot\|\mathcal{X}\|_{\infty,2} \right.\\
&\left.+  \frac{c_1\log((n_1+n_2)n_3)\sqrt{(n_1+n_2)\mu_{0}r}}{p\sqrt{n_{1}n_{2}}}\|\mathcal{X}\|_{\infty}\right)  \geq 1-n_2((n_1+n_2)n_3)^{1-c_1}.
\end{align*}
Similarly, we can show that 
\begin{align*}
&\mathbb{P}\left( \max_{b}\left\{\left\|\mrme_{b}^{\top}*(\mathcal{P}_{\mathbb{T}}\mathcal{R}_{\Omega}(\mathcal{X})-\mathcal{P}_{\mathbb{T}}(\mathcal{X}))\right\|_{\fro}\right\}\leq\sqrt{\frac{4c_1(n_{1}+n_{2})\log\left((n_1+n_2)n_3\right)\mu_{0}r}{pn_{1}n_{2}}}\cdot\|\mathcal{X}\|_{\infty,2} + \right.\\
&\left.\frac{c_1\log((n_1+n_2)n_3)\sqrt{(n_1+n_2)\mu_{0}r}}{p\sqrt{n_{1}n_{2}}}\|\mathcal{X}\|_{\infty}\right)\geq 1-n_1((n_1+n_2)n_3)^{1-c_1}.
\end{align*}
Thus, we can get 
\begin{align*}
&\mathbb{P}\left( \left\|(\mathcal{P}_{\mathbb{T}}\mathcal{R}_{\Omega}(\mathcal{X})-\mathcal{P}_{\mathbb{T}}(\mathcal{X}))\right\|_{\infty,2}\leq\sqrt{\frac{4c_1(n_{1}+n_{2})\log\left((n_1+n_2)n_3\right)\mu_{0}r}{pn_{1}n_{2}}}\cdot\|\mathcal{X}\|_{\infty,2} \right. \\ 
&+\left.\frac{c_1\log((n_1+n_2)n_3)\sqrt{(n_1+n_2)\mu_{0}r}}{p\sqrt{n_{1}n_{2}}}\|\mathcal{X}\|_{\infty}\right)\geq 1- ((n_1+n_2)n_3)^{2-c_1}.
\end{align*}
\hfill\BlackBox\\[2mm]
\vspace{-30pt}
\paragraph{Proof of \Cref{conditionII3}}
Notice that
\begin{align}
\mathcal{P}_{\mathbb{T}} \mathcal{R}_{\Omega} \mathcal{P}_{\mathbb{T}}(\mathcal{X})&= \mathcal{P}_{\mathbb{T}} \mathcal{R}_{\Omega} \left(\sum\limits_{i,j,k}\left\langle\mathcal{P}_{\mathbb{T}}(\mathcal{X}),\mrme_i * \dme_k * \mrme_j^{\top}\right\rangle \mrme_i * \dme_k * \mrme_j^{\top}\right)\label{eq1} \\
&= \mathcal{P}_{\mathbb{T}} \left(\sum\limits_{i,j,k}\frac{1}{p} \delta_{i, j, k} \left\langle\mathcal{P}_{\mathbb{T}}(\mathcal{X}),\mrme_i * \dme_k * \mrme_j^{\top}\right\rangle \mrme_i * \dme_k * \mrme_j^{\top}\right) \notag\\
&= \sum\limits_{i,j,k}\frac{1}{p} \delta_{i, j, k} \left\langle\mathcal{P}_{\mathbb{T}}(\mathcal{X}),\mrme_i * \dme_k * \mrme_j^{\top}\right\rangle \mathcal{P}_{\mathbb{T}} \left(\mrme_i * \dme_k * \mrme_j^{\top}\right)\label{eq2},
\end{align}
where \eqref{eq1} holds because $\mathcal{P}_{\mathbb{T}}( \mathcal{X}) = \sum\limits_{i,j,k}\left\langle\mathcal{P}_{\mathbb{T}}(\mathcal{X}),\mrme_i * \dme_k * \mrme_j^{\top}\right\rangle \mrme_i * \dme_k * \mrme_j^{\top}$, and \eqref{eq2}  follows from the linearity of the operator $\mathcal{P}_{\mathbb{T}}$.
Notice that 
\begin{align*}
    \mathcal{P}_{\mathbb{T}}(\mathcal{X})  = \mathcal{P}_{\mathbb{T}}\left(\mathcal{P}_{\mathbb{T}}(\mathcal{X})\right) 
&=\mathcal{P}_{\mathbb{T}}\left(\sum\limits_{i,j,k}\left\langle\mathcal{P}_{\mathbb{T}}(\mathcal{X}),\mrme_i * \dme_k * \mrme_j^{\top}\right\rangle \mrme_i * \dme_k * \mrme_j^{\top}\right)\\
&=\sum\limits_{i,j,k}\left\langle\mathcal{P}_{\mathbb{T}}(\mathcal{X}),\mrme_i * \dme_k * \mrme_j^{\top}\right\rangle \mathcal{P}_{\mathbb{T}}\left(\mrme_i * \dme_k * \mrme_j^{\top}\right)
\end{align*}
Thus, we  have that  any $(a,b,c)$-th entry of $\mathcal{P}_{\mathbb{T}} \mathcal{R}_{\Omega} \mathcal{P}_{\mathbb{T}}(\mathcal{X})-\mathcal{P}_{\mathbb{T}}(\mathcal{X})$ can be given by 
\begin{align*}
&\left\langle\mathcal{P}_{\mathbb{T}} \mathcal{R}_{\Omega} \mathcal{P}_{\mathbb{T}}(\mathcal{X})-\mathcal{P}_{\mathbb{T}}(\mathcal{X}), \mrme_a*\dme_c*\mrme_b^{\top}\right\rangle \\
=&~\left\langle\sum\limits_{i, j, k} (\frac{1}{p} \delta_{i, j, k}-1)\left\langle\mathcal{P}_{\mathbb{T}}(\mathcal{X}),\mrme_i * \dme_k * \mrme_j^{\top}\right\rangle\mathcal{P}_{\mathbb{T}}\left(\mrme_i * \dme_k * \mrme_j^{\top}\right),\mrme_a * \dme_c * \mrme_b^{\top}\right\rangle \\
=&~\sum_{i, j, k} (\frac{1}{p} \delta_{i, j, k}-1)\left\langle\mathcal{P}_{\mathbb{T}}(\mathcal{X}),\mrme_i * \dme_k * \mrme_j^{\top}\right\rangle\left\langle\mathcal{P}_{\mathbb{T}}\left(\mrme_i * \dme_k * \mrme_j^{\top}\right),\mrme_a * \dme_c * \mrme_b^{\top}\right\rangle 
=: \sum\limits_{i,j,k}h_{i,j,k}
\end{align*} 
It is easy to see that $\mathbb{E}(h_{i,j,k})=0$. Notice that 
\begin{align*}
 |h_{i,j,k}|
=& \left|(\frac{1}{p} \delta_{i, j, k}-1)\left\langle\mathcal{P}_{\mathbb{T}}(\mathcal{X}),\mrme_i * \dme_k * \mrme_j^{\top}\right\rangle \left\langle \mathcal{P}_{\mathbb{T}}\left(\mrme_i * \dme_k * \mrme_j^{\top}\right),\mrme_a * \dme_c * \mrme_b^{\top}\right\rangle\right|\\
=&\left|(\frac{1}{p} \delta_{i, j, k}-1)\left\langle\mathcal{P}_{\mathbb{T}}(\mathcal{X}),\mrme_i * \dme_k * \mrme_j^{\top}\right\rangle\left\langle \mathcal{P}_{\mathbb{T}}\left(\mrme_i * \dme_k * \mrme_j^{\top}\right),\mathcal{P}_{\mathbb{T}}\left(\mrme_a * \dme_c * \mrme_b^{\top}\right)\right\rangle\right| \\
\leq & \frac{1}{p}\|\mathcal{P}_{\mathbb{T}}(\mathcal{X})\|_{\infty}\left\|\mathcal{P}_{\mathbb{T}}\left(\mrme_i * \dme_k * \mrme_j^{\top}\right)\right\|_{\fro}\left\|\mathcal{P}_{\mathbb{T}}\left(\mrme_a * \dme_c * \mrme_b^{\top}\right)\right\|_{\fro}
\leq\frac{(n_1+n_2)\mu_0r}{pn_{1}n_2}\left\|\mathcal{P}_{\mathbb{T}}(\mathcal{X})\right\|_{\infty}
\end{align*}
It is easy to check that  
\vspace{-3pt}
\begin{align*}
 \left|\sum\limits_{i,j,k}\mathbb{E}[h_{i,j,k}^2]\right|  
&=\mathbb{E}\left(\left| (\frac{1}{p} \delta_{i, j, k}-1)\left\langle\mathcal{P}_{\mathbb{T}}(\mathcal{X}),\mrme_i * \dme_k * \mrme_j^{\top}\right\rangle \left\langle\mathcal{P}_{\mathbb{T}}\left(\mrme_i * \dme_k * \mrme_j^{\top}\right),\mrme_a * \dme_c * \mrme_b^{\top}\right\rangle \right|^2 \right)\\
&=\mathbb{E}\left((\frac{1}{p} \delta_{i, j, k}-1)^2\left|\left\langle\mathcal{P}_{\mathbb{T}}(\mathcal{X}),\mrme_i * \dme_k * \mrme_j^{\top}\right\rangle  \left\langle\mathcal{P}_{\mathbb{T}}\left(\mrme_i * \dme_k * \mrme_j^{\top}\right),\mrme_a * \dme_c * \mrme_b^{\top}\right\rangle \right|^2 \right) \\
&\leq \frac{1}{p}\left\|\mathcal{P}_{\mathbb{T}}(\mathcal{X})\right\|_{\infty}^2\left\|\mathcal{P}_{\mathbb{T}}\left(\mrme_i * \dme_k * \mrme_j^{\top}\right)\right\|_{\fro}^2\leq \frac{(n_1+n_2)\mu_{0}r}{pn_{1}n_{2}}\left\|\mathcal{P}_{\mathbb{T}}(\mathcal{X})\right\|_{\infty}^2.
\end{align*}
Thus, by the non-commutative Bernstein inequality, we have 
\begin{align*}
& \mathbb{P}\left[\left(\mathcal{P}_{\mathbb{T}} \mathcal{R}_{\Omega} \mathcal{P}_{\mathbb{T}}(\mathcal{X})-\mathcal{P}_{\mathbb{T}}(\mathcal{X})\right)_{a, b, c} \geq \frac{1}{2}\|\mathcal{P}_{\mathbb{T}}(\mathcal{X})\|_{\infty}\right] \\
\leq & 2 \exp \left(\frac{-\|\mathcal{P}_{\mathbb{T}}(\mathcal{X})\|_{\infty}^2 /   8 }{\frac{(n_1 + n_2) \mu_0 r}{pn_{1}n_{2}}\|\mathcal{P}_{\mathbb{T}}(\mathcal{X})\|_{\infty}^2+\frac{(n_1 + n_2) \mu_0 r}{6pn_{1}n_{2}}\|\mathcal{P}_{\mathbb{T}}(\mathcal{X})\|_{\infty}^2}\right) \\
= & 2\exp\left(\frac{-3pn_{1}n_{2}}{28(n_1+n_2)\mu_{0}r}\right).
\end{align*}
Using the union bound on every $(a,b,c)$-th entry, we have 
$$
\left\|\left(\mathcal{P}_{\mathbb{T}} \mathcal{R}_{\Omega} \mathcal{P}_{\mathbb{T}}-\mathcal{P}_{\mathbb{T}}\right)(\mathcal{P}_{\mathbb{T}}(\mathcal{X}))\right\|_{\infty} \leq \frac{1}{2}\|\mathcal{P}_{\mathbb{T}}(\mathcal{X})\|_{\infty}
$$
 with probability at least $1-2n_{1}n_{2}n_{3}\exp\left(\frac{-3pn_{1}n_{2}}{28(n_1+n_2)\mu_{0}r}\right)$.
\hfill\BlackBox\\[2mm]

Lastly, to ensure the completeness of our exposition, we provide a detailed proof of \Cref{conditions} in \ref{proposition1}, as originally presented in \cite{Lu18}.

\subsection{Proof of \Cref{conditions}}\label{proposition1}
 In the following context, the symbol $\mathcal{Z}$  represents the tensor that we aim to recover in \eqref{optimize}. Before delving into the detailed proof pipeline, we wish to reiterate the purpose of \Cref{conditions}, which asserts that $\mathcal{Z}$ is a unique minimizer to  \eqref{optimize} when Conditions~\ref{CondI} and ~\ref{CondII} are satisfied simultaneously.
Notice that $\mathcal{Z}$ is a feasible solution to  \eqref{optimize}. To show  that $\mathcal{Z}$ is the unique minimizer, it suffices to show \[\|\mathcal{X}\|_{\TNN}-\|\mathcal{Z}\|_{\TNN} > 0\] for any feasible solution $\mathcal{X}$ where $\mathcal{X}\neq \mathcal{Z}$. 

First, we  show that for any feasible solution $\mathcal{X}$ different from $\mathcal{Z}$, there exists an auxiliary tensor $\mathcal{M}$ such that
\[
\left\|\mathcal{X}\right\|_{\TNN} - \left\|\mathcal{Z}\right\|_{\TNN} \geq \left\langle \mathcal{W}*\mathcal{V}^{\top}+\mathcal{P}_{\mathbb{T}^{\perp}(\mathcal{Z})}(\mathcal{M}), \mathcal{X}-\mathcal{Z}\right\rangle.\]
In this way, we can transform the proof of  $\|\mathcal{X}\|_{\TNN} -\|\mathcal{Z}\|_{\TNN}>0$ into showing \[\left\langle \mathcal{W}*\mathcal{V}^{\top}+\mathcal{P}_{\mathbb{T}^{\perp}(\mathcal{Z})}(\mathcal{M}), \mathcal{X}-\mathcal{Z}\right\rangle>0.\] To prove $\left\langle \mathcal{W}*\mathcal{V}^{\top}+\mathcal{P}_{\mathbb{T}^{\perp}}(\mathcal{M}), \mathcal{X}-\mathcal{Z}\right\rangle>0$, we split \[\left\langle \mathcal{W}*\mathcal{V}^{\top}+\mathcal{P}_{\mathbb{T}^{\perp}}(\mathcal{M}), \mathcal{X}-\mathcal{Z}\right\rangle\] into two parts \[\left\langle \mathcal{P}_{\mathbb{T}^{\perp}}(\mathcal{M}), \mathcal{X}-\mathcal{Z}\right\rangle \text{ and } 
 \left\langle \mathcal{W}*\mathcal{V}^{\top},\mathcal{X}-\mathcal{Z}\right\rangle.\]
 By construction the auxiliary tensor $\mathcal{M}$, we can show that \[\left\langle \mathcal{P}_{\mathbb{T}^{\perp}}(\mathcal{M}), \mathcal{X}-\mathcal{Z}\right\rangle=\|\mathcal{P}_{\mathbb{T}^{\perp}}(\mathcal{X}-\mathcal{Z})\|_{\TNN}.\]
 As for the part $\left\langle \mathcal{W}*\mathcal{V}^{\top},\mathcal{X}-\mathcal{Z}\right\rangle,$ we further split it into two parts by introducing the dual certification tensor $\mathcal{Y}$: \[\left\langle \mathcal{P}_{\mathbb{T}(\mathcal{Z})}(\mathcal{Y})-\mathcal{W}*\mathcal{V}^{\top},\mathcal{X}-\mathcal{Z}\right\rangle , \left\langle \mathcal{P}_{\mathbb{T}^{\perp}(\mathcal{Z})}(\mathcal{Y}),\mathcal{X}-\mathcal{Z}\right\rangle.\]
 The reason for this separation is that we can bound these two terms by  $\frac{1}{2}\|\mathcal{P}_{\mathbb{T}^{\perp}}(\mathcal{X}-\mathcal{Z})\|_{\TNN}$ and  $\frac{\sqrt{2}}{4}\|\mathcal{P}_{\mathbb{T}^{\perp}}(\mathcal{X}-\mathcal{Z})\|_{\TNN}$ respectively. By combining the bounds of the above three separations together, we obtain  \[\langle\mathcal{W}*\mathcal{V}^{\top}+\mathcal{P}_{\mathbb{T}^{\perp}}(\mathcal{M}),\mathcal{X}-\mathcal{Z}\rangle\geq \frac{1}{8}\left\|\mathcal{P}_{\mathbb{T}^{\perp}}(\mathcal{X}-\mathcal{Z})\right\|_{\TNN}.\] 
 In the end, we prove that $\left\|\mathcal{P}_{\mathbb{T}^{\perp}(\mathcal{Z})}(\mathcal{X}-\mathcal{Z})\right\|_{\TNN}$ strictly larger than zero by contradiction. Before proceeding to the detailed proof, we will present several useful lemmas which are key to the proof of \Cref{conditions}. 
\subsubsection{\textnormal{Supporting lemmas for the proof of \Cref{conditions}}}
First, we state the characterization of the tensor nuclear norm (TNN), which can be described as  a duality to the tensor spectral norm.
\begin{lemma}[\cite{Lu18}]\label{dual_tnn}
 Given a tensor $\mathcal{T}\in \mathbb{K}^{n_{1}\times n_{2}\times n_{3}}$, we have $$
 \left\|\mathcal{T}\right\|_{\TNN} = \sup\limits_{\{ \mathcal{Q}\in \mathbb{K}^{n_{1}\times n_{2}\times n_{3}}:\left\|\mathcal{Q}\right\|\leq 1\}} \left\langle \mathcal{Q},\mathcal{T}\right\rangle.$$
\end{lemma}
 {Next, we present a characterization of the subdifferential of TNN,  which is useful for proving the uniqueness of the minimizer to  \eqref{optimize}.}
\begin{lemma}[Subdifferential of $\TNN$ \citep{Lu18}]  \label{subdifferential}
Let $\mathcal{Z}\in \mathbb{K}^{n_{1}\times n_{2}\times n_{3}}$ and its compact t-SVD be $\mathcal{Z} = \mathcal{W}*\mathcal{S}*\mathcal{V}^{\top}$. The subdifferential (the set of subgradients) of $\left\|\mathcal{Z}\right\|_{\TNN}$ is $\partial\left\|\mathcal{Z}\right\|_{\TNN} = \{\mathcal{W}*\mathcal{V}^{\top}+\mathcal{M}: \mathcal{W}^{\top}*\mathcal{M}=0, \mathcal{M}*\mathcal{V}=0,\left\|\mathcal{M}\right\|\leq 1\}$.
\end{lemma}
For the proof of  \Cref{conditions}, a significant challenge is proving the minimizer's uniqueness. To tackle it, based on the result of \Cref{subdifferential}, we can introduce an auxiliary tensor $\mathcal{M}$ such that \[\left\langle \mathcal{W}*\mathcal{V}^{\top}+\mathcal{P}_{\mathbb{T}^{\perp}(\mathcal{Z})}(\mathcal{M}), \mathcal{X}-\mathcal{Z}\right\rangle>0. \] 
\begin{lemma}[\cite{zhang2017exact}] \label{lemma19}
 Assume that $\Omega$ is generated according to the Bernoulli sampling with probability $p$. If $\left\|\mathcal{P}_{\mathbb{T}}\mathcal{R}_{\Omega}\mathcal{P}_{\mathbb{T}}-\mathcal{P}_{\mathbb{T}}\right\|\leq \frac{1}{2}$, then   
$$
\left\|\mathcal{P}_{\mathbb{T}}(\mathcal{X})\right\|_{\fro}\leq \sqrt{\frac{2n_{3}}{p}} \cdot \left\|\mathcal{P}_{\mathbb{T}^{\perp}}(\mathcal{X})\right\|_{\TNN},
$$
for any $\mathcal{X}$ with $\mathcal{P}_{\Omega}(\mathcal{X}) = 0$.
\end{lemma}
This lemma shows that  \(\|\mathcal{P}_{\mathbb{T}}(\mathcal{X})\|_{\fro}\) can be bounded above by \(\sqrt{\frac{2n_3}{p}}\|\mathcal{P}_{\mathbb{T}^{\perp}}(\mathcal{X})\|_{\TNN}\), a key element in our analysis. Further details are in \ref{proof:proposition1}.
\vspace{-5pt}
\subsubsection{\textnormal{Proof of \Cref{conditions}}}\label{proof:proposition1}
\vspace{-5pt}
\begin{proof}
Consider any feasible solution $\mathcal{X}\neq \mathcal{Z}$ to  \eqref{optimize} with $\mathcal{P}_{\Omega}(\mathcal{X}) = \mathcal{P}_{\Omega}(\mathcal{Z})$. By  \Cref{dual_tnn}, we have an auxiliary tensor $\mathcal{M}$ with $\left\|\mathcal{M}\right\| \leq 1$ such that 
\begin{align*}\left\|\mathcal{P}_{\mathbb{T}^{\perp}}\left(\mathcal{X}-\mathcal{Z}\right) \right\|_{\TNN} &= \left\langle \mathcal{M},\mathcal{P}_{\mathbb{T}^{\perp}}(\mathcal{X}-\mathcal{Z})\right\rangle=\left\langle \mathcal{P}_{\mathbb{T}^{\perp}}(\mathcal{M}),\mathcal{P}_{\mathbb{T}^{\perp}}(\mathcal{X}-\mathcal{Z})\right\rangle.
\end{align*}
First, it is easy to check that $\mathcal{W}^{\top}*\mathcal{P}_{\mathbb{T}^{\perp}}(\mathcal{M})=0$ and $\mathcal{P}_{\mathbb{T}^{\perp}}(\mathcal{M})*\mathcal{V}= 0$ by the definition of the operator $\mathcal{P}_{\mathbb{T}^{\perp}}$.
By \Cref{subdifferential}, we have that $\mathcal{W}*\mathcal{V}^{\top}+\mathcal{P}_{\mathbb{T}^{\perp}}(\mathcal{M})$ is a subgradient of $\mathcal{Z}$ in terms of tensor nuclear norm.  Therefore, we have 
$$
\left\|\mathcal{X}\right\|_{\TNN} - \left\|\mathcal{Z}\right\|_{\TNN} \geq \left\langle \mathcal{W}*\mathcal{V}^{\top}+\mathcal{P}_{\mathbb{T}^{\perp}}(\mathcal{M}), \mathcal{X}-\mathcal{Z}\right\rangle.$$ 
\noindent To prove $ 
\left\|\mathcal{X}\right\|_{\TNN} - \left\|\mathcal{Z}\right\|_{\TNN} \geq 0$, it is sufficient to show 
\[\left\langle \mathcal{W}*\mathcal{V}^{\top}+\mathcal{P}_{\mathbb{T}^{\perp}}(\mathcal{M}), \mathcal{X}-\mathcal{Z}\right\rangle \geq 0.\] 
Notice that for any $\mathcal{Y}$ where $\mathcal{P}_{\Omega}(\mathcal{Y})=\mathcal{Y}$, we have $$\langle \mathcal{Y},\mathcal{X}-\mathcal{Z}\rangle=\langle \mathcal{P}_{\Omega}(\mathcal{Y}),\mathcal{X}-\mathcal{Z}\rangle=\langle \mathcal{P}_{\Omega}(\mathcal{Y}),\mathcal{P}_{\Omega}(\mathcal{X}-\mathcal{Z})\rangle=0.$$
Thus, we have 
\[\left\langle \mathcal{W}*\mathcal{V}^{\top}+\mathcal{P}_{\mathbb{T}^{\perp}}(\mathcal{M}),\mathcal{X}-\mathcal{Z}\right\rangle=\left\langle \mathcal{W}*\mathcal{V}^{\top}+\mathcal{P}_{\mathbb{T}^{\perp}}(\mathcal{M})-\mathcal{Y},\mathcal{X}-\mathcal{Z}\right\rangle. \]
Furthermore, we have 
\begin{align}
   & ~\left\langle \mathcal{W}*\mathcal{V}^{\top}+\mathcal{P}_{\mathbb{T}^{\perp}}(\mathcal{M})-\mathcal{Y},\mathcal{X}-\mathcal{Z}\right\rangle \notag\\
   =& ~ \left\langle \mathcal{W}*\mathcal{V}^{\top}+\mathcal{P}_{\mathbb{T}^{\perp}}(\mathcal{M})-\mathcal{P}_{\mathbb{T}^{\perp}}(\mathcal{Y})-\mathcal{P}_{\mathbb{T}}(\mathcal{Y}),\mathcal{X}-\mathcal{Z}\right\rangle \notag \\
   =& ~ \left\langle \mathcal{P}_{\mathbb{T}^{\perp}}(\mathcal{M}),\mathcal{X}-\mathcal{Z}\right\rangle - \left\langle \mathcal{P}_{\mathbb{T}}(\mathcal{Y})-\mathcal{W}*\mathcal{V}^{\top},\mathcal{X}-\mathcal{Z}\right\rangle- \left\langle \mathcal{P}_{\mathbb{T}^{\perp}}(\mathcal{Y}),\mathcal{X}-\mathcal{Z}\right\rangle \notag\\
    =& ~\left\|\mathcal{P}_{\mathbb{T}^{\perp}}(\mathcal{X}-\mathcal{Z})\right\|_{\TNN} - \left\langle \mathcal{P}_{\mathbb{T}}(\mathcal{Y})-\mathcal{W}*\mathcal{V}^{\top},\mathcal{P}_{\mathbb{T}}(\mathcal{X}-\mathcal{Z})\right\rangle- \left\langle \mathcal{P}_{\mathbb{T}^{\perp}}(\mathcal{Y}),\mathcal{P}_{\mathbb{T}^{\perp}}(\mathcal{X}-\mathcal{Z})\right\rangle \notag\\
    \geq &~\left\|\mathcal{P}_{\mathbb{T}^{\perp}}(\mathcal{X}-\mathcal{Z})\right\|_{\TNN}- \left\| \mathcal{P}_{\mathbb{T}}(\mathcal{Y})-\mathcal{W}*\mathcal{V}^{\top}\right\|_{\fro}\left\|\mathcal{P}_{\mathbb{T}}(\mathcal{X}-\mathcal{Z})\right\|_{\fro} \notag\\&- \left\|\mathcal{P}_{\mathbb{T}^{\perp}}(\mathcal{Y})\right\|\left\|\mathcal{P}_{\mathbb{T}^{\perp}}(\mathcal{X}-\mathcal{Z})\right\|_{\TNN} \notag\\
     \geq&~ \frac{1}{2}\left\|\mathcal{P}_{\mathbb{T}^{\perp}}(\mathcal{X}-\mathcal{Z})\right\|_{\TNN} - \frac{1}{4}\sqrt{\frac{p}{n_3}}\cdot\sqrt{\frac{2n_{3}}{p}}\left\|\mathcal{P}_{\mathbb{T}^{\perp}}(\mathcal{X}-\mathcal{Z})\right\|_{\TNN}\label{reason4} \\
    \geq&~\frac{1}{8}\left\|\mathcal{P}_{\mathbb{T}^{\perp}}(\mathcal{X}-\mathcal{Z})\right\|_{\TNN} \notag,
\end{align}
 where Inequality \eqref{reason4} results from Condition \ref{CondI} ($\left\|\mathcal{P}_{\mathbb{T}}\mathcal{R}_{\Omega}\mathcal{P}_{\mathbb{T}}-\mathcal{P}_{\mathbb{T}}\right\| \leq \frac{1}{2}$), Condition \ref{CondII} ($\left\|\mathcal{P}_{\mathbb{T}}(\mathcal{Y})-\mathcal{W}*\mathcal{V}^{\top}\right\|_{\fro} \leq \frac{1}{4}\sqrt{\frac{p}{n_3}}$, 
$\left\|\mathcal{P}_{\mathbb{T}^{\perp}}(\mathcal{Y})\right\|\leq \frac{1}{2}$),  and \Cref{lemma19}.

Next, we need to show that $\left\|\mathcal{P}_{\mathbb{T}^{\perp}}(\mathcal{X}-\mathcal{Z})\right\|_{\TNN}$ is strictly positive. We show this by contradiction. 
Suppose $\left\|\mathcal{P}_{\mathbb{T}^{\perp}}(\mathcal{X}-\mathcal{Z})\right\|_{\TNN} = 0$. Then, 
$\mathcal{P}_{\mathbb{T}}(\mathcal{X}-\mathcal{Z}) = \mathcal{X}-\mathcal{Z}$ and $\mathcal{P}_{\mathbb{T}}\mathcal{R}_{\Omega}\mathcal{P}_{\mathbb{T}}(\mathcal{X}-\mathcal{Z}) = 0$. Therefore, we have \[\|\mathcal{X}-\mathcal{Z}\|=\|(\mathcal{P}_{\mathbb{T}}\mathcal{R}_{\Omega}\mathcal{P}_{\mathbb{T}}-\mathcal{P}_{\mathbb{T}})(\mathcal{X}-\mathcal{Z})\|\leq \|\mathcal{P}_{\mathbb{T}}\mathcal{R}_{\Omega}\mathcal{P}_{\mathbb{T}}-\mathcal{P}_{\mathbb{T}}\|\|\mathcal{X}-\mathcal{Z}\|,\]
which implies that $\|\mathcal{P}_{\mathbb{T}}\mathcal{R}_{\Omega}\mathcal{P}_{\mathbb{T}}-\mathcal{P}_{\mathbb{T}}\|\geq 1$. This   contradicts with the assumption that $\|\mathcal{P}_{\mathbb{T}}\mathcal{R}_{\Omega}\mathcal{P}_{\mathbb{T}}-\mathcal{P}_{\mathbb{T}}\|\leq \frac{1}{2}$.
Thus, $\mathcal{Z}$ is the unique minimizer to  \eqref{optimize}.
\end{proof}
\vspace{-21pt}
\subsubsection{\textnormal{Proof of \Cref{lemma19}}}
\begin{proof}
Let $\mathcal{X}$ be a tensor satisfying $\mathcal{P}_{\Omega}(\mathcal{X}) = 0$.
Using the self-adjoint property of the operator $\mathcal{P}_{\mathbb{T}}$, we have 
\begin{align*}
\left\|\mathcal{R}_{\Omega}\mathcal{P}_{\mathbb{T}}(\mathcal{X})\right\|_{\fro}^2 
& = \left\langle \mathcal{R}_{\Omega}\mathcal{P}_{\mathbb{T}}(\mathcal{X}),\mathcal{R}_{\Omega}\mathcal{P}_{\mathbb{T}}(\mathcal{X})\right\rangle \\
& = \left\langle \mathcal{R}_{\Omega}\mathcal{P}_{\mathbb{T}}(\mathcal{X}),\sum\limits_{i,j,k}\frac{1}{p}\delta_{i,j,k}[\mathcal{P}_{\mathbb{T}}(\mathcal{X})]_{i,j,k}\cdot\mrme_i * \dme_k * \mrme_j^{\top}\right\rangle \\ 
& = \frac{1}{p}\left\langle \mathcal{R}_{\Omega}\mathcal{P}_{\mathbb{T}}(\mathcal{X}),\sum\limits_{i,j,k}\delta_{i,j,k}[\mathcal{P}_{\mathbb{T}}(\mathcal{X})]_{i,j,k}\cdot\mrme_i * \dme_k * \mrme_j^{\top}\right\rangle \\ 
& = \frac{1}{p}\left\langle \mathcal{R}_{\Omega}\mathcal{P}_{\mathbb{T}}(\mathcal{X}),\mathcal{P}_{\Omega}(\mathcal{P}_{\mathbb{T}}(\mathcal{X}))\right\rangle=\frac{1}{p} \left\langle \mathcal{P}_{\mathbb{T}}\mathcal{R}_{\Omega}\mathcal{P}_{\mathbb{T}}(\mathcal{X}),\mathcal{X}\right\rangle \\
& = \frac{1}{p}\left\langle \mathcal{P}_{\mathbb{T}}\mathcal{R}_{\Omega}\mathcal{P}_{\mathbb{T}}(\mathcal{X})-\mathcal{P}_{\mathbb{T}}(\mathcal{X}),\mathcal{X}\right\rangle +  \frac{1}{p}\left\langle \mathcal{P}_{\mathbb{T}}(\mathcal{X}),\mathcal{X}\right\rangle\\
&=\frac{1}{p}\left\|\mathcal{P}_{\mathbb{T}}(\mathcal{X})\right\|_{\fro}^2+\frac{1}{p}\left\langle \left(\mathcal{P}_{\mathbb{T}}\mathcal{R}_{\Omega}\mathcal{P}_{\mathbb{T}}-\mathcal{P}_{\mathbb{T}}\right)(\mathcal{X}),\mathcal{P}_{\mathbb{T}}(\mathcal{X})\right\rangle\\
&\geq \frac{1}{p}\left\|\mathcal{P}_{\mathbb{T}}(\mathcal{X})\right\|_{\fro}^2-\frac{1}{p}\left\|\mathcal{P}_{\mathbb{T}}\mathcal{R}_{\Omega}\mathcal{P}_{\mathbb{T}}-\mathcal{P}_{\mathbb{T}}\right\|\cdot\left\|\mathcal{P}_{\mathbb{T}}(\mathcal{X})\right\|_{\fro}^2\geq \frac{1}{2p}\left\|\mathcal{P}_{\mathbb{T}}(\mathcal{X})\right\|_{\fro}^2.
\end{align*}
Notice that if $\mathcal{P}_{\Omega}(\mathcal{X})=0,$ then $\mathcal{R}_{\Omega}(\mathcal{X})$ must be the zero tensor. Thus, we have
\begin{align*}
    \left\|\mathcal{R}_{\Omega}\mathcal{P}_{\mathbb{T}}(\mathcal{X})\right\|_{\fro} = &\left\|\mathcal{R}_{\Omega}\mathcal{P}_{\mathbb{T}^{\perp}}(\mathcal{X})\right\|_{\fro} 
    \leq \frac{1}{p}\left\|\mathcal{P}_{\mathbb{T}^{\perp}}(\mathcal{X})\right\|_{\fro}\\
    =& \frac{1}{p\sqrt{n_3}}\left\|\overline{\mathcal{P}_{\mathbb{T}^{\perp}}(\mathcal{X})}\right\|_{\fro} \leq \frac{1}{p\sqrt{n_3}}\left\|\overline{\mathcal{P}_{\mathbb{T}^{\perp}}(\mathcal{X})}\right\|_{*} = \frac{\sqrt{n_3}}{p}\left\|\mathcal{P}_{\mathbb{T}^{\perp}}(\mathcal{X})\right\|_{\TNN}.
    \end{align*}
    As a result, we have $\left\|\mathcal{P}_{\mathbb{T}}(\mathcal{X})\right\|_{\fro}\leq \sqrt{2p}\left\|\mathcal{R}_{\Omega}\mathcal{P}_{\mathbb{T}}(\mathcal{X})\right\|_{\fro} \leq \sqrt{\frac{2n_3}{p}}\left\|\mathcal{P}_{\mathbb{T}^{\perp}}(\mathcal{X})\right\|_{\TNN}$.
\end{proof}

Next we present a detailed proof of \Cref{thm:samplingcmp4ccs}, our main theoretical result, demonstrating that our model ensures  tensor recovery in high-probability.  
\section{Proof of \Cref{thm:samplingcmp4ccs}}
In this section, we provide a detailed proof of our main theoretical result \Cref{thm:samplingcmp4ccs}. The proof is based on  our Two-Step Tensor Completion (TSTC) algorithm. For the ease of the reader, we state the TSTC algorithm in \Cref{alg:Twostep}. This algorithm focuses on subtensor completion before combining results with t-CUR. 
 \begin{algorithm}[!ht]
 \caption{Two-Step Tensor Completion (TSTC)}\label{alg:Twostep}
\begin{algorithmic}[1]
\State\textbf{Input}: 
$[\T]_{\Omega_{\mathcal{R}}\cup\Omega_{\mathcal{C}}}$: observed data; 
$\Omega_{\mathcal{R}}, \Omega_{\mathcal{C}}$: observation locations; 
$I,J$: lateral and horizontal indices;  
$r$: target rank;
  $\mathrm{TC}$: the chosen tensor completion solver. 
\State $\widetilde{\mathcal{R}} = \mathrm{TC}([\T]_{\Omega_{\mathcal{R}}},r) $ 
\State $\widetilde{\mathcal{C}} = \mathrm{TC}([\mathcal{C}]_{\Omega_{\mathcal{C}}},r) $ 
\State $\widetilde{\mathcal{U}} = [\widetilde{\mathcal{C}}]_{I,:,:} $
\State $\widetilde{\T} = \widetilde{\mathcal{C}}*{\widetilde{\mathcal{U}}}^\dagger *\widetilde{\mathcal{R}}$
\State\textbf{Output: }$\widetilde{\mathcal{T}}$: approximation of $\mathcal{T}$
\end{algorithmic}
\end{algorithm} 

Based on the idea of TSTC, it is crucial to  understand that how the tensor incoherence properties of the original low tubal-rank tensor transfer to subtensors. 
\subsection{Incoherence passes to subtensors }\label{incoherence_conditions}
Inspired by \citep[Theorem~3.5]{cai2021robust}, we explore how subtensors inherit the tensor incoherence conditions from the original tensor, differing from \citep{mohammadpour2023tensor} in tensor norm and the definition of the tensor incoherence condition. Our focus is on subtensor incoherence due to its impact on the required sampling rate for accurate low tubal-rank tensor recovery (\Cref{thm:sr4generalTCB}) and our emphasis on completing subtensors in tensor completion. We begin by examining the relationship between the tensor incoherence properties of subtensors and the original low tubal-rank tensor.
 
\begin{lemma}\label{lmm:c_incoherence1} 
Let $\mathcal{T} \in \mathbb{K}^{n_1 \times n_2 \times n_3}$ satisfy the tensor $\mu_0$-incoherence condition. Suppose that $\mathcal{T}$ has a compact t-SVD $\mathcal{T}=\mathcal{W}\ast\Sigma \ast\mathcal{V}^{\top}$ and a condition number $\kappa$.  Consider the subtensors $\mathcal{C}=[\mathcal{T}]_{:,J,:}$ and $\mathcal{R}=[\mathcal{T}]_{I,:,:}$, each maintaining the same tubal-rank as  $\mathcal{T}$. Their compact t-SVDs are represented as \[\mathcal{C}=\mathcal{W}_{\mathcal{C}}*\Sigma_{\mathcal{C}}*\mathcal{V}_{\mathcal{C}}^{\top} ~\text{and}~ \mathcal{R}=\mathcal{W}_{\mathcal{R}}*\Sigma_{\mathcal{R}}*\mathcal{V}_{\mathcal{R}}^{\top},\] then the following results hold:
\[\mu_{\mathcal{C}}\leq \kappa^2 \left\|[\mathcal{V}]_{J,:,:}^{\dagger}\right\|^2\frac{|J|}{n_2}\mu_0,~ \text{ and } ~\mu_{\mathcal{R}}\leq \kappa^2 \left\|[\mathcal{W}]_{I,:,:}^{\dagger}\right\|^2\frac{|I|}{n_1}\mu_0.\]
 \end{lemma}
\begin{proof}
 First, let's prove that \[\max_{i}\left\|[\widehat{\mathcal{W}_{\mathcal{C}}}^{\top}]_{:,:,k} \cdot \mathbf{e}_i\right\|_{\fro}\leq \sqrt{\frac{\mu_{0}r}{n_1}}.\] 
Notice that $\mathcal{C}=\mathcal{T}_{:, J,:}=\mathcal{W} * \Sigma *[\mathcal{V}]_{J,:,:}^{\top}$. Assume the compact t-SVD of $\Sigma *\left([\mathcal{V}]_{J,:,:}\right)^{\top}$ is   $\Sigma *\left([\mathcal{V}]_{J,:,:}\right)^{\top}=\mathcal{P} * \mathcal{S} * \mathcal{Q}^{\top}$. Thus, $\mathcal{P}\in\mathbb{K}^{r\times r\times n_3}$ is an orthonormal tensor, leading to $\mathcal{W}_{\mathcal{C}}=\mathcal{W} * \mathcal{P}$ based on the relationship that $\mathcal{C}=\mathcal{W} * \mathcal{P}*\mathcal{S}*\mathcal{Q}^{\top}$.\\
 $\mathcal{P}^{\top}*\mathcal{P} = \mathcal{I}$ implies that $[\widehat{\mathcal{P}}]_{:,:,k}^{\top}\cdot[\widehat{\mathcal{P}}]_{:,:,k} = \mathbb{I}_r$, where $\mathbb{I}_r$ is the $r\times r$ identity matrix for all $k\in[n_3]$. 
Therefore, we can establish that for $k\in[n_3]$,
\begin{equation}\label{eqn:leftincohereceofC}
\max\limits_{i}\| \widehat{\mathcal{W}_{\mathcal{C}}}]_{:,:,k}^{\top}\cdot \mathbf{e}_i\|_{\fro} =
\max\limits_{i}\| [\widehat{\mathcal{W}}]_{:,:,k}^{\top}\cdot \mathbf{e}_i\|_{\fro} \leq \sqrt{\frac{\mu_{0}r}{n_1}}.
\end{equation}

\noindent Next, let's prove $\max\limits_{i}\|[\widehat{\mathcal{V}_{\mathcal{C}}}^{\top}]_{:,:,k}\cdot \mathbf{e}_{i}\|_{\fro}\leq \kappa(\mathcal{T})\left\|([\mathcal{V}]_{J,:,:})^{\dagger}\right\| \sqrt{\frac{\mu_0 r}{n_2}}$.  The compact t-SVD of $\mathcal{C}$ implies $\mathcal{V}_{\mathcal{C}}^{\top} = \Sigma_{\mathcal{C}}^{\dagger}*\mathcal{W}_{\mathcal{C}}^{\top}*\mathcal{C}$. 
 Thus,   for each $k\in[n_3]$,
$[\widehat{\mathcal{V}}_{\mathcal{C}}^{\top}]_{:,:,k} = [\widehat{\Sigma}_{\mathcal{C}}^{\dagger}]_{:,:,k} \cdot [\widehat{\mathcal{W}}_{\mathcal{C}}^{\top}]_{:,:,k}\cdot [\widehat{\mathcal{C}}]_{:,:,k}$ holds and  $\|[\widehat{\mathcal{V}}_{\mathcal{C}}^\top]_{:,:,k}\cdot \mathbf{e}_i\|_{\fro}$ can be bounded by

\begin{align*}
\|[\widehat{\mathcal{V}}_{\mathcal{C}}^\top]_{:,:,k}\cdot \mathbf{e}_i\|_{\fro}=&\|[\widehat{\Sigma}_{\mathcal{C}}^{\dagger}]_{:,:,k} \cdot [\widehat{\mathcal{W}}_{\mathcal{C}}^{\top}]_{:,:,k}\cdot [\widehat{\mathcal{C}}]_{:,:,k}\cdot \mathbf{e}_i\|_{\fro}\\
\leq &\|[\widehat{\Sigma}_{\mathcal{C}}^{\dagger}]_{:,:,k}\|\|[\widehat{\mathcal{W}_{\mathcal{C}}}^\top]_{:,:,k}\cdot[\widehat{\mathcal{W}}]_{:,:,k}\cdot[\widehat{\Sigma}]_{:,:,k}\cdot [\widehat{\mathcal{V}}]_{J,:,k}^{\top}\cdot \mathbf{e}_i\|_{\fro}\\
\leq& \left\|\overline{\Sigma}_{\mathcal{C}}^{\dagger}\right\|\left\|\overline{\Sigma}\right\|\|[\widehat{\mathcal{V}}]_{J,:,k}^{\top}\cdot \mathbf{e}_i\|_{\fro}\\
\leq&  \|\mathcal{C}^\dagger \|\left\|\mathcal{T}\right\|\|[\widehat{\mathcal{V}}]_{:,:,k}^{\top}\cdot \mathbf{e}_i\|_{\fro}\\
=& \left\|([\mathcal{V}]_{J,:,:}^{\top})^{\dagger}*{\Sigma}^{\dagger}*\mathcal{W}^{\top}\right\|\left\|\mathcal{T}\right\|\|[\widehat{\mathcal{V}}]_{:,:,k}^{\top}\cdot \mathbf{e}_i\|_{\fro}\\
 \leq& \left\|([\mathcal{V}]_{J,:,:}^{\top})^{\dagger}\right\|\|\mathcal{T}^{\dagger}\|\left\|\mathcal{T}\right\|\|[\widehat{\mathcal{V}}]_{:,:,k}^{\top}\cdot \mathbf{e}_i\|_{\fro}\leq \kappa \left\|([\mathcal{V}]_{J,:,:}^{\top})^{\dagger}\right\| \sqrt{\frac{\mu_0 r}{n_2}}.
\end{align*} 
That is, \begin{equation}\label{eqn:rightincoherenceofC}
    \max\limits_{i}\|[\widehat{\mathcal{V}}_{\mathcal{C}}^{\top}]_{:,:,k}\cdot \mathbf{e}_{i}\|_{\fro}\leq \kappa \left\|[\mathcal{V}]_{J,:,:}^{\dagger}\right\|\sqrt{\frac{\mu_0|J|}{n_2}} \sqrt{\frac{r}{|J|}}.
\end{equation} 
 Combining \eqref{eqn:leftincohereceofC} and \eqref{eqn:rightincoherenceofC}, we can conclude that $\mu_{\mathcal{C}}\leq \kappa^2 \left\|[\mathcal{V}]_{J,:,:}^{\dagger}\right\|^2\frac{|J|}{n_2}\mu_0$.
 
 Applying above process  on $\mathcal{R}$,  we can get  $\mu_{\mathcal{R}}\leq \kappa^2 \left\|[\mathcal{W}]_{I,:,:}^{\dagger}\right\|^2\frac{|I|}{n_1}\mu_0$. 
\end{proof}
Following \Cref{lmm:c_incoherence1}, we explore the incoherence properties of uniformly sampled subtensors, with summarized results below.
 
\begin{lemma} \label{lmm:incoherence-unf}  Let  $\mathcal{T}\in\mathbb{K}^{n_1\times n_2 \times n_3}$ with  multi-rank $\vec{r}$,   and let   $\mathcal{T}=\mathcal{W}*\Sigma*\mathcal{V}^{\top}$ be its compact t-SVD. Additionally,  $\mathcal{T}$ satisfies the tensor $\mu_0$-incoherence condition,   and $\kappa$ denotes the condition number of $\mathcal{T}$. Suppose $I\subseteq[n_1]$ and $J\subseteq[n_2]$ are chosen uniformly at random with replacement. Then \[\rank_m(\mathcal{R})=\rank_m(\mathcal{C})=\rank_m(\mathcal{T}), 
\mu_{\mathcal{C}}\leq \frac{25}{4}\kappa^2 \mu_0 \text{ and } 
 \mu_{\mathcal{R}}\leq \frac{25}{4}\kappa^2 \mu_0\]
 hold with probability at least $1-\frac{1}{n_1^{\beta}}-\frac{1}{n_2^{\beta}}$ 
provided that  $|I|\geq 2\beta \mu_0 \|\vec{r}\|_{\infty} \log \left(n_1\|\vec{r}\|_{1}\right)$ and  $|J|\geq 2\beta \mu_0 \|\vec{r}\|_{\infty} \log \left(n_2\|\vec{r}\|_{1}\right)$.
\end{lemma}
\begin{proof}
According to   \Cref{lmm:key} by setting $\delta=0.815$ and $\beta_1=\beta_2=\beta$, we can easily get that 
\begin{align*}
    \mathbb{P}\left(\|[\mathcal{V}]_{J,:,:}^\dagger\|\leq \sqrt{\frac{25n_2}{4|J|}},\rank_m(\mathcal{C})=\rank_m(\mathcal{T})\right)\geq& 1-\frac{1}{n_2^\beta},\\
     \mathbb{P}\left(\|[\mathcal{W}]_{I,:,:}^\dagger\|\leq \sqrt{\frac{25n_1}{4|I|}},\rank_m(\mathcal{R})=\rank_m(\mathcal{T})\right)\geq& 1-\frac{1}{n_1^\beta}.
\end{align*}
 Therefore, 
 \begin{align}
    \mathbb{P}\left(\mu_{\mathcal{C}}\leq \frac{25}{4}\kappa^2 \mu_0,\rank_m(\mathcal{C})=\rank_m(\mathcal{T})\right)\geq& 1-\frac{1}{n_2^\beta}, \label{eqn:muC}\\
     \mathbb{P}\left(\mu_{\mathcal{R}}\leq \frac{25}{4}\kappa^2 \mu_0,\rank_m(\mathcal{R})=\rank_m(\mathcal{T})\right)\geq& 1-\frac{1}{n_1^\beta}\label{eqn:muR}.
\end{align}
Combining \eqref{eqn:muC} and \eqref{eqn:muR}, we can conclude that \[\rank_m(\mathcal{R})=\rank_m(\mathcal{C})=\rank_m(\mathcal{T}), 
\mu_{\mathcal{C}}\leq \frac{25}{4}\kappa^2 \mu_0 \text{ and } 
 \mu_{\mathcal{R}}\leq \frac{25}{4}\kappa^2 \mu_0\] 
 with probability at least $1-\frac{1}{n_1^{\beta}}-\frac{1}{n_2^{\beta}}$ 
provided that  \[|I|\geq 2\beta \mu_0 \|\vec{r}\|_{\infty} \log \left(n_1\|\vec{r}\|_{1}\right) \text{ and }   |J|\geq 2\beta \mu_0 \|\vec{r}\|_{\infty} \log \left(n_2\|\vec{r}\|_{1}\right).\]
\end{proof}

 \vspace{-20pt}
\subsection{Proof of \Cref{thm:samplingcmp4ccs}}
\begin{proof}
Note that  $I\subseteq[n_1]$ and $J\subseteq[n_2]$ are chosen  uniformly   with replacement. According to \Cref{lmm:incoherence-unf}, we thus have 
 \[\rank_m(\mathcal{R})=\rank_m(\mathcal{C})=\rank_m(\mathcal{T}), 
\mu_{\mathcal{C}}\leq \frac{25}{4}\kappa^2 \mu_0 \text{ and } 
 \mu_{\mathcal{R}}\leq \frac{25}{4}\kappa^2 \mu_0\]
 hold with probability at least 
\begin{align*}
   & 1-\frac{1}{n_1^{800\beta\kappa^2\log(n_1n_3+n_2n_3)}}-\frac{1}{n_2^{800\beta\kappa^2\log(n_1n_3+n_2n_3)}}\\
   =& 1-\frac{1}{(n_1n_3+n_2n_3)^{800\beta\kappa^2\log(n_1)}}-\frac{1}{(n_1n_3+n_2n_3)^{800\beta\kappa^2\log(n_2)}}
\end{align*} 
provided that  \[|I|\geq 3200\beta \mu_0 r\kappa^2\log^2(n_1n_3+n_2n_3)\geq 800\kappa^2\log(n_1n_3+n_2n_3)\beta \cdot(2  \mu_0 \|\vec{r}\|_{\infty} \log \left(n_1\|\vec{r}\|_{1}\right) ),\]
  \[
|J|\geq 3200\beta \mu_0 r\kappa^2\log^2(n_1n_3+n_2n_3)\geq 800\kappa^2\log(n_1n_3+n_2n_3)\beta\cdot(2 \mu_0 \|\vec{r}\|_{\infty} \log \left(n_2\|\vec{r}\|_{1}\right).
\]
 
Additionally, the following statements hold by   \Cref{thm:sr4generalTCB} and   the condition that $\mu_{\mathcal{C}}\leq \frac{25}{4}\kappa^2 \mu_0 \text{ and } 
 \mu_{\mathcal{R}}\leq \frac{25}{4}\kappa^2 \mu_0$:
\begin{enumerate}[label=\roman*),leftmargin=.5in]
    \item Given $\mathcal{C}\in\mathbb{K}^{n_1\times|J|\times n_3}$ with $\rank(\mathcal{C})=r$, 
$$p_{\mathcal{C}} \geq \frac{1600\beta(n_1+|J|)\mu_0 r\kappa^2 \log^2(n_1n_3+|J|n_3)}{n_1|J|}$$ for some $\beta>1$ ensures that 
  $\mathcal{C}$ is the unique minimizer to 
\begin{align*}
\min_{\mathcal{X}\in\mathbb{K}^{n_1\times |J|\times n_3}}~\|\mathcal{X}\|_{\TNN},~~~ 
\text { subject to }~ \mathcal{P}_{\Omega_{\mathcal{C}}}(\mathcal{X})=\mathcal{P}_{\Omega_{\mathcal{C}}}(\mathcal{C}).
\end{align*}
with probability at least $1-  \frac{3\log(n_1n_3+|J|n_3)}{(n_1n_3+|J|n_3)^{4\beta-2}}$. 

 \item Given $\mathcal{R}\in\mathbb{K}^{|I|\times n_2\times n_3}$ with $\rank(\mathcal{R})=r$, 
 \[p_{\mathcal{R}} \geq \frac{1600\beta(n_2+|I|)\mu_0r\kappa^2\log^2(n_2n_3+|I|n_3)}{n_2|I|}\] for some $\beta>1$ ensures that 
 $\mathcal{R}$ is the unique minimizer to 
\begin{align*}
\min_{\mathcal{X}\in\mathbb{K}^{|I|\times n_2\times n_3}}~~\|\mathcal{X}\|_{\TNN},~~\text { subject to }~ \mathcal{P}_{\Omega_{\mathcal{R}}}(\mathcal{X})=\mathcal{P}_{\Omega_{\mathcal{R}}}(\mathcal{R}).
\end{align*}
\end{enumerate}
Once $\mathcal{C}$ and $\mathcal{R}$ are uniquely recovered from $\Omega_\mathcal{C}$ and $\Omega_{\mathcal{R}}$, respectively. Then t-CUR decomposition can provide the reconstruction of $\mathcal{T}$ via $\mathcal{T}=\mathcal{C}*\mathcal{U}^\dagger*\mathcal{R}$ with the condition $\rank_m(\mathcal{R})=\rank_m(\mathcal{C})=\rank_m(\mathcal{T})$. 

Combining all the statements above, we can conclude that $\mathcal{T}$ can be uniquely recovered from $\Omega_{\mathcal{C}}\cup\Omega_{\mathcal{R}}$ with probability at least
\[
1-\frac{2}{(n_1n_3+n_2n_3)^{800\beta\kappa^2\log(n_2)}}-\frac{3\log(n_1n_3+|J|n_3)}{(n_1n_3+|J|n_3)^{4\beta-2}}-\frac{3\log(n_2n_3+|I|n_3)}{(n_2n_3+|I|n_3)^{4\beta-2}}.
\]
 \end{proof}
 \vspace{-25pt}
\section{More numerical experiments}\label{convergence_study}
In this section, we include further empirical data demonstrating the convergence behavior of the ITCURTC algorithm within the t-CCS model framework. In this experiment, we form a low tubal-rank tensor $\mathcal{T} = \mathcal{A} * \mathcal{B} \in \mathbb{R}^{n_1 \times n_2 \times n_3}$ using two Gaussian random tensors, where $\mathcal{A} \in \mathbb{R}^{n_1 \times r \times n_3}$ and $\mathcal{B} \in \mathbb{R}^{r \times n_2 \times n_3}$. Our objective is to examine the convergence behavior of the ITCURTC algorithm under different conditions. For the simulations, we set $n_1 = n_2 = 768$ and $n_3 = 256$, and generate partial observations using the t-CCS model by adjusting the rank $r$ and configuring the concentrated subtensors as $\mathcal{R} \in \mathbb{R}^{\delta n_1 \times n_2 \times n_3}$ and $\mathcal{C} \in \mathbb{R}^{n_1 \times \delta n_2 \times n_3}$, with $0 < \delta < 1$. For each  fixed $r$, we maintain a constant overall sampling rate $\alpha$.

Utilizing the observed data, the ITCURTC algorithm is then employed to approximate the original low tubal-rank tensor. The algorithm continues until the stopping criterion $\varepsilon_k \leq 10^{-6}$ is met, where $\varepsilon_k$ represents the relative error between the estimate at the $k$-th iteration and the actual tensor, defined as $\varepsilon_k = \frac{\|\mathcal{T} - \mathcal{\widehat{T}}_k\|_{\fro}}{\|\mathcal{T}\|_{\fro}}$.

For each specified set of parameters $(r, \delta, \alpha)$, we generate $10$ different tensor completion scenarios. The mean relative errors $\varepsilon_k$, along with the specific configurations, are reported in \Cref{figure4}. One can see that ITCURTC can achieve an almost linear convergence rate.

\begin{figure}[!ht]
    \centering
       \begin{minipage}[b]{0.2425\textwidth}
        \includegraphics[width=\textwidth]{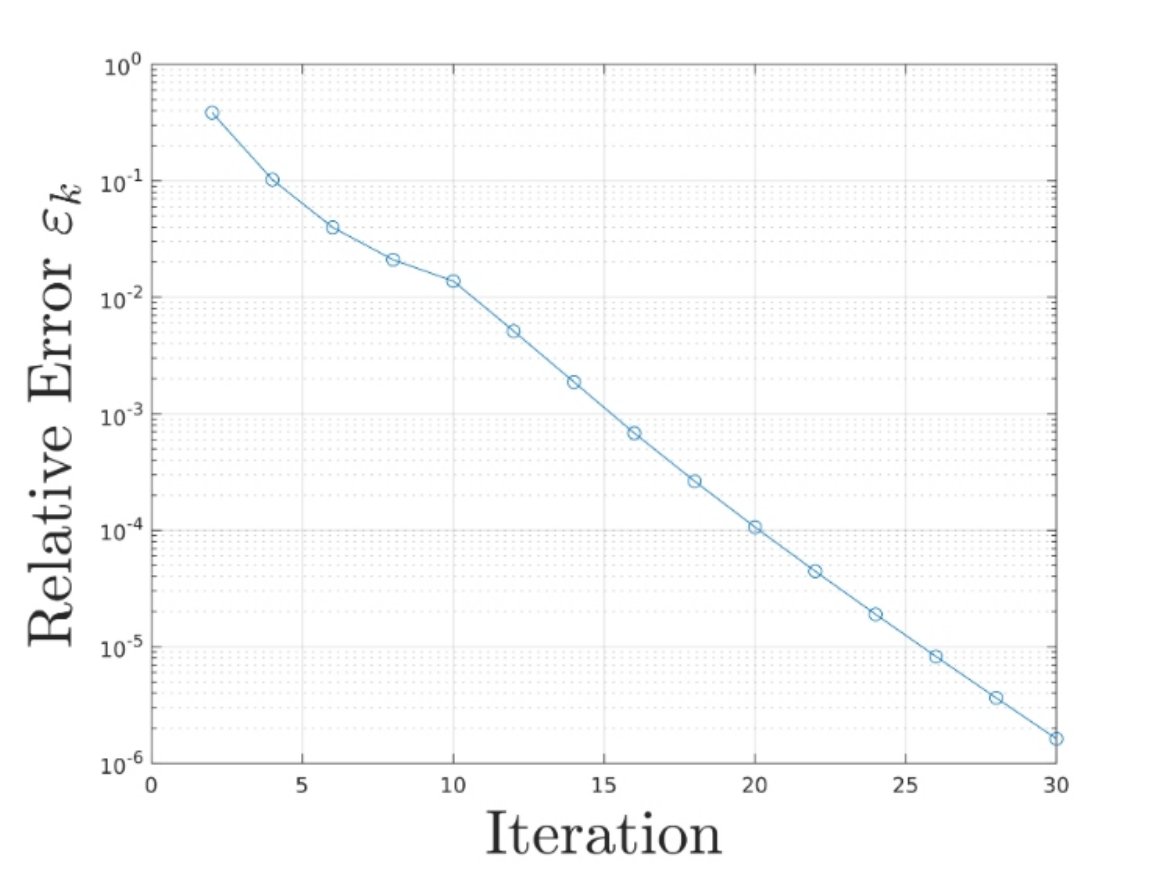}
    \end{minipage}
    \hfill 
    \begin{minipage}[b]{0.2425\textwidth}
        \includegraphics[width=\textwidth]{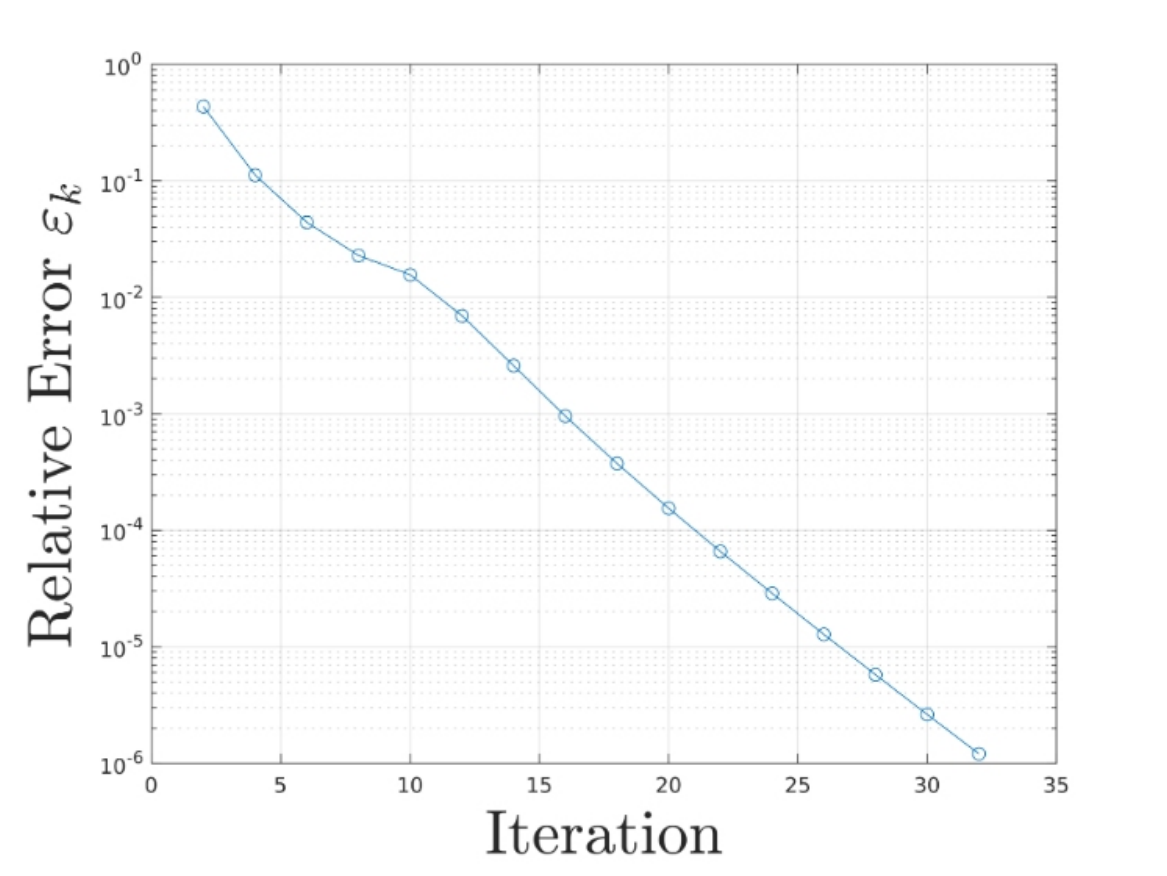}
    \end{minipage}
    \hfill 
    \begin{minipage}[b]{0.2425\textwidth}
        \includegraphics[width=\textwidth]{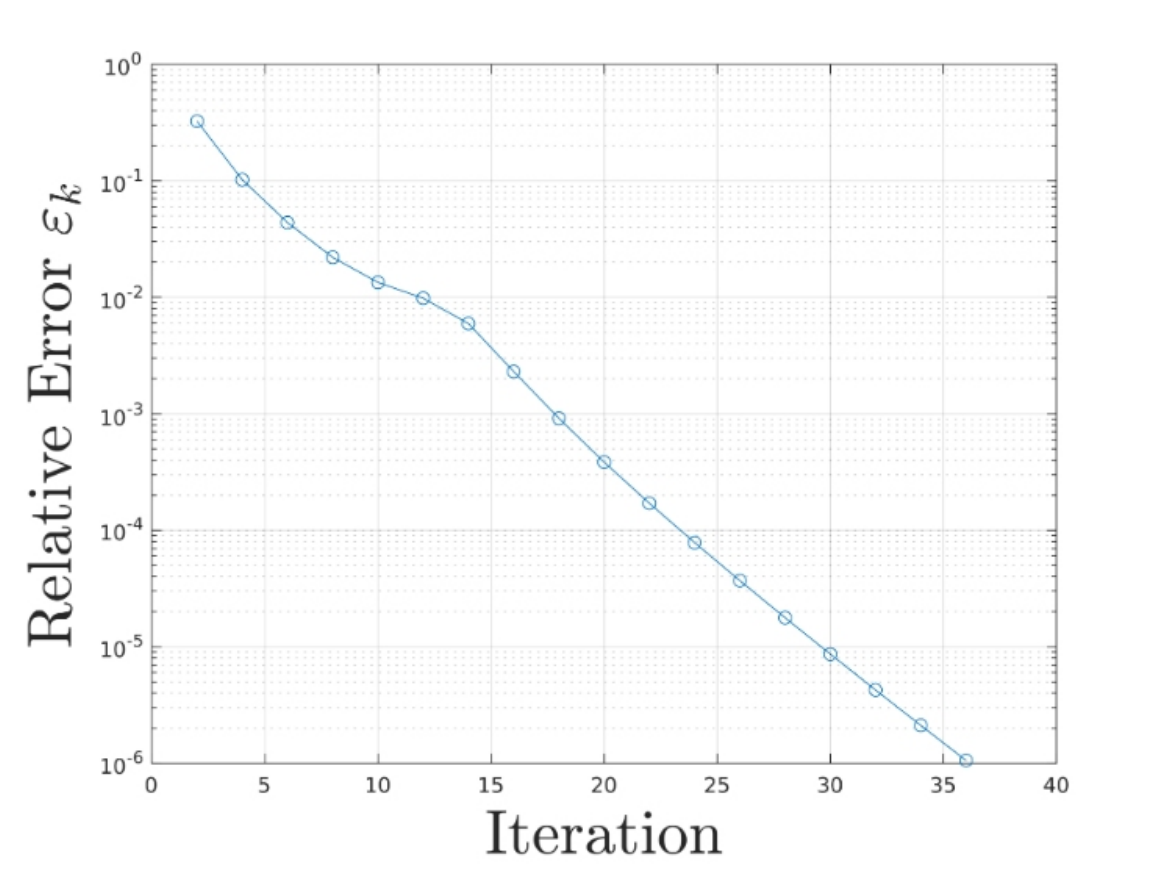}
    \end{minipage}
    \hfill
    \begin{minipage}[b]{0.2425\textwidth}
        \includegraphics[width=\textwidth]{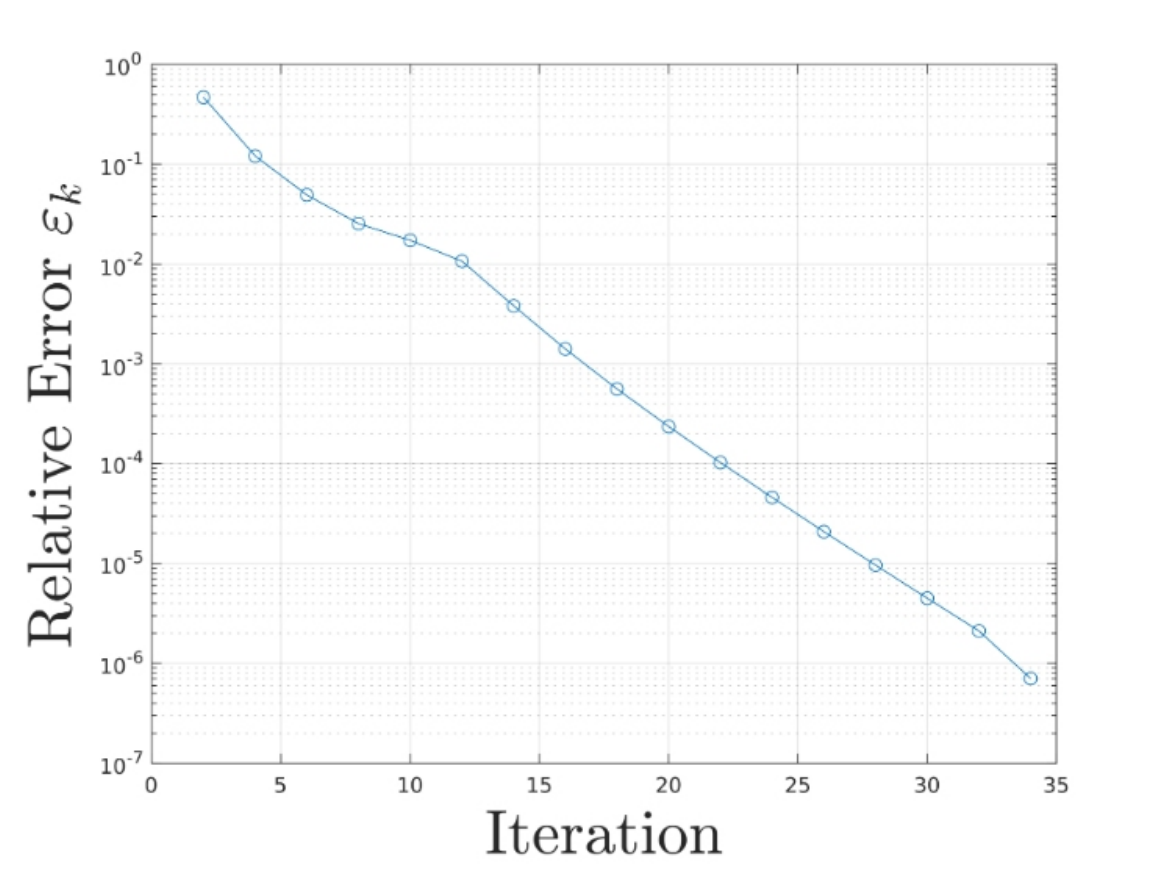}
    \end{minipage}
    \hfill 
    \begin{minipage}[b]{0.2425\textwidth}
        \includegraphics[width=\textwidth]{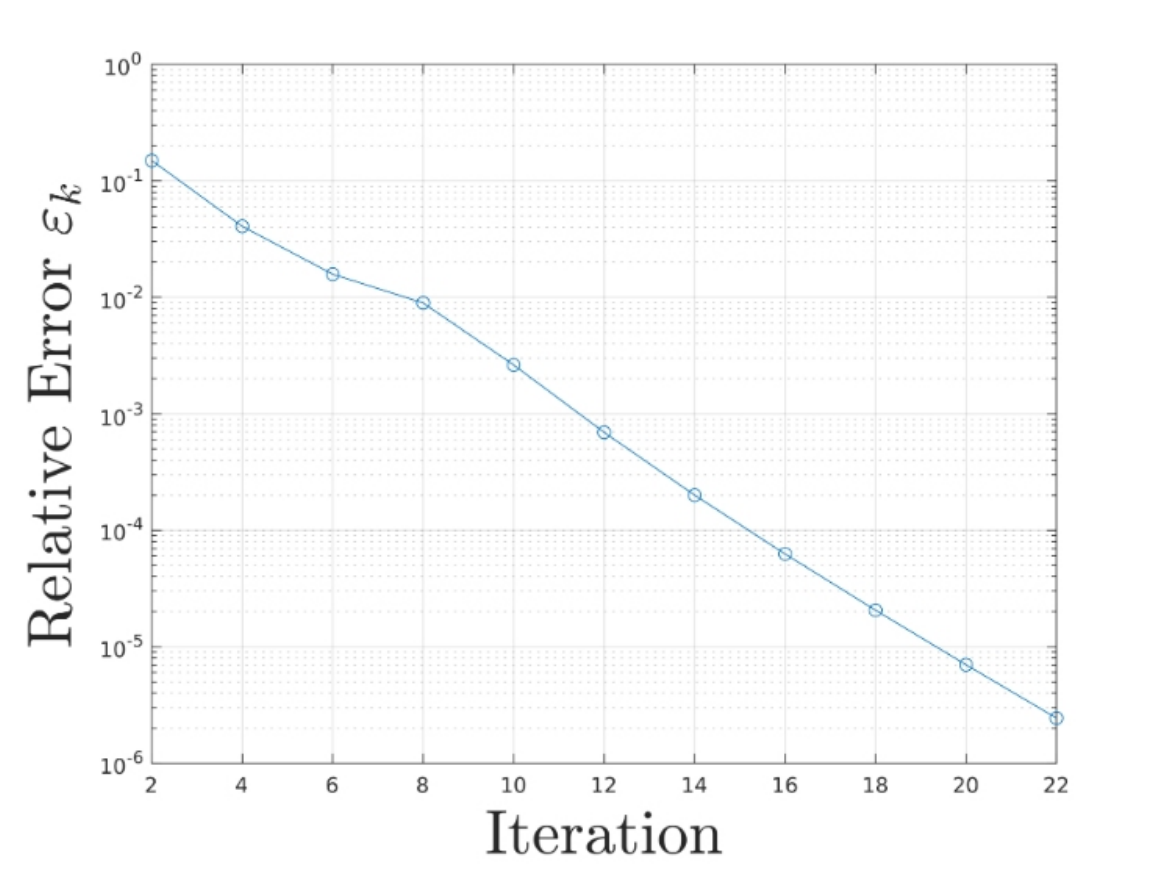}
        \subcaption*{$\delta=0.2$}
    \end{minipage}
    \hfill
    \begin{minipage}[b]{0.2425\textwidth}
        \includegraphics[width=\textwidth]{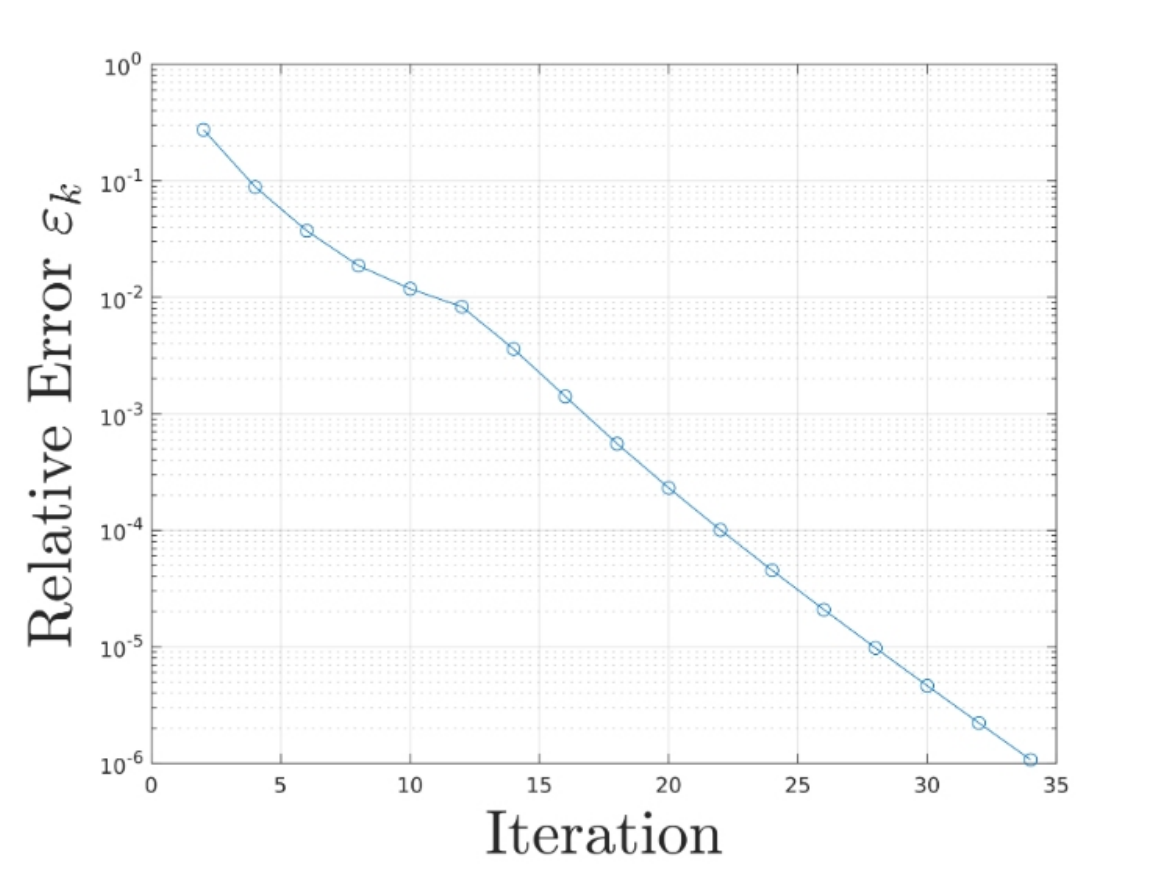}
        \subcaption*{$\delta=0.25$}
    \end{minipage}
    \hfill
    \begin{minipage}[b]{0.2425\textwidth}
        \includegraphics[width=\textwidth]{./figure/semilog/semilog_plot_delta_0.300000_p_0.250000_r_5_iter_50.pdf}
        \subcaption*{$\delta=0.3$}
    \end{minipage}
    \hfill
    \begin{minipage}[b]{0.2425\textwidth}
        \includegraphics[width=\textwidth]{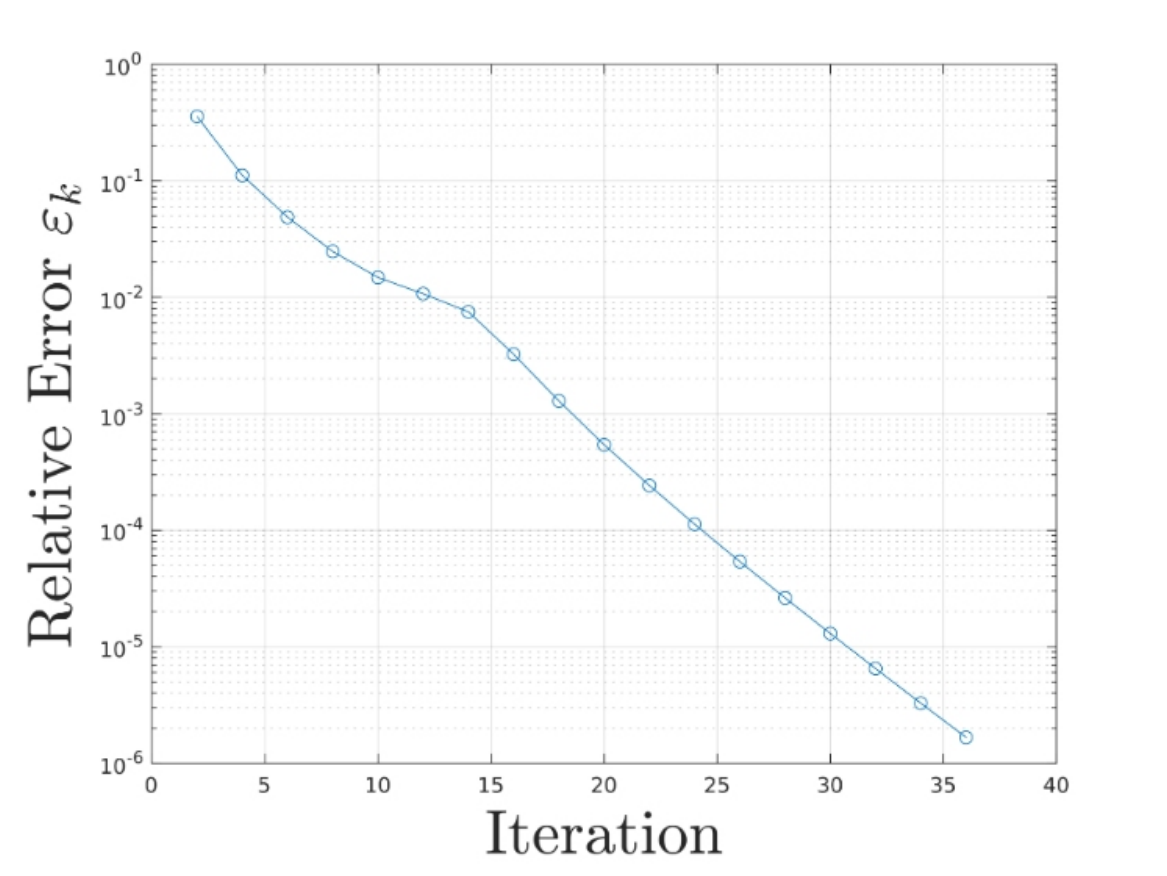}
        \subcaption*{$\delta=0.35$}
    \end{minipage}
    \caption{\footnotesize The averaged relative error of ITCURTC under the t-CCS model with respect to iterations over 10 independent trials. \textbf{1st row}:  the overall sampling rate $\alpha=0.15$ and tubal-rank $r=2$. \textbf{2nd row}:  $\alpha=0.25$ and $r=5$.}
    \label{figure4}
\end{figure}

\bibliographystyle{plain}
\bibliography{ref}
\end{document}